\documentclass{article}

\usepackage{microtype}
\usepackage{graphicx}
\usepackage{subcaption}
\usepackage{booktabs} 
\usepackage{amsthm}

\usepackage{hyperref}

\newtheorem{theorem}{Theorem}
\newtheorem{lemma}{Lemma}

\usepackage{amsmath,amsfonts,bm}

\newcommand{\D}{\mathcal{D}}
\newcommand{\Proba}{\mathbb{P}}

\newcommand{\R}{\mathbb{R}}

\newcommand{\E}{\mathbb{E}}

\newcommand{\x}{\mathbf{x}}

\newcommand{\calX}{\mathcal{X}}

\newcommand{\calY}{\mathcal{Y}}

\newcommand{\calS}{\mathcal{S}}
\newcommand{\calT}{\mathcal{T}}

\newcommand{\calH}{\mathcal{H}}
\newcommand{\calR}{\mathcal{R}}
\newcommand{\calZ}{\mathcal{Z}}
\newcommand{\calN}{\mathcal{N}}
\newcommand{\RN}[1]{%
  \textup{\uppercase\expandafter{\romannumeral#1}}%
}

\newcommand{\y}{\mathbf{y}}

\newcommand{\blambda}{\boldsymbol{\lambda}}

\usepackage[accepted]{icml2021}

\icmltitlerunning{Aggregating From Multiple Target-Shifted Sources}

\begin{document}


\twocolumn[
\icmltitle{Aggregating From Multiple Target-Shifted Sources}



\icmlsetsymbol{equal}{*}

\begin{icmlauthorlist}
\icmlauthor{Changjian Shui}{to}
\icmlauthor{Zijian Li}{ed}
\icmlauthor{Jiaqi Li}{goo}
\icmlauthor{Christian Gagné}{to,mila}
\icmlauthor{Charles X. Ling}{goo}
\icmlauthor{Boyu Wang}{goo,vi}
\end{icmlauthorlist}

\icmlaffiliation{to}{Université Laval}
\icmlaffiliation{goo}{Western University}
\icmlaffiliation{ed}{Guangdong University of Technology}
\icmlaffiliation{mila}{Canada CIFAR AI Chair, Mila}
\icmlaffiliation{vi}{Vector Institute}
\icmlcorrespondingauthor{Boyu Wang}{bwang@csd.uwo.ca}
\icmlcorrespondingauthor{Christian Gagné}{christian.gagne@gel.ulaval.ca}

\icmlkeywords{Domain Adaptation}

\vskip 0.3in
]



\printAffiliationsAndNotice{}  

\begin{abstract}
Multi-source domain adaptation aims at leveraging the knowledge from multiple tasks for predicting a related target domain. A crucial aspect is to properly combine different sources based on their relations. In this paper, we analyzed the problem for aggregating source domains with different label distributions, where most recent source selection approaches fail. Our proposed algorithm differs from previous approaches in two key ways: the model aggregates multiple sources mainly through the similarity of semantic conditional distribution rather than marginal distribution; the model proposes a \emph{unified} framework to select relevant sources for three popular scenarios, i.e., domain adaptation with limited label on target domain, unsupervised domain adaptation and label partial unsupervised domain adaption. We evaluate the proposed method through extensive experiments. The empirical results significantly outperform the baselines. 
\end{abstract}

\section{Introduction}
Domain Adaptation (DA) \citep{pan2009survey} is based on the motivation that learning a new task is easier after having learned a similar task. By learning the inductive bias from a related source domain $\calS$ and then leveraging the shared knowledge upon learning the target domain $\calT$, the prediction performance can be significantly improved. Based on this, DA arises in tremendous deep learning applications such as computer vision \citep{zhang2019recent,hoffman2018cycada}, natural language processing \citep{ruder2019transfer,houlsby2019parameter} and biomedical engineering \citep{raghu2019transfusion,wang2020common}.

In various real-world applications, we want to transfer knowledge from \emph{multiple sources} $(\calS_1,\dots,\calS_T)$ to build a model for the target domain, which requires an effective selection and leveraging the \emph{most useful} sources. Clearly, solely combining all the sources and applying one-to-one single DA algorithm can lead to undesired results, as it can include irrelevant or even untrusted data from certain sources, which can severely influence the performance \citep{zhao2020multi}.

\begin{figure}
  \begin{center}
    \includegraphics[width=0.4\textwidth]{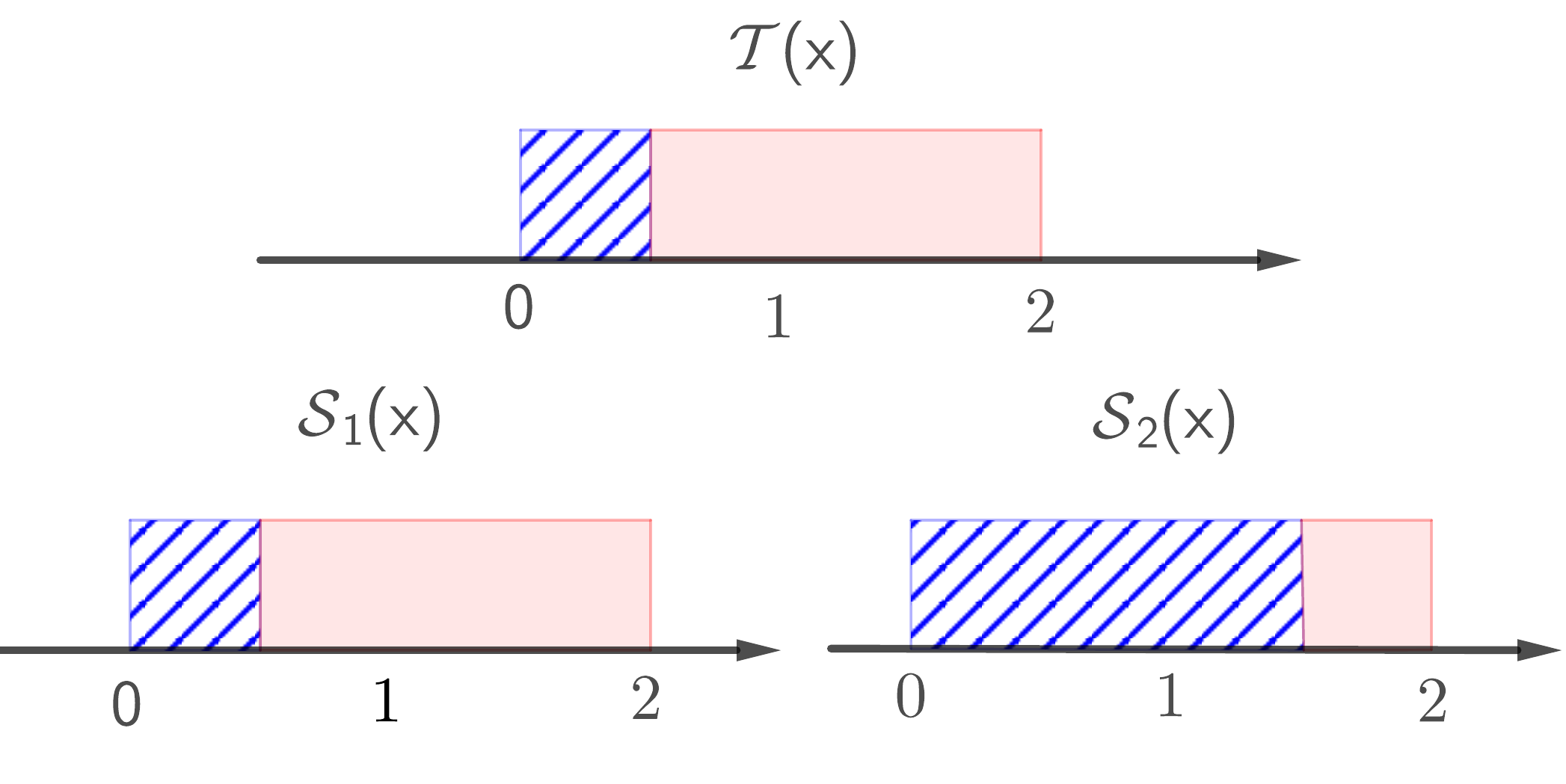}
  \end{center}
  \caption{Limitation of merely considering marginal distribution $\Proba(x)$ in the source selection. In a binary classification, we have $\calS_1(x)=\calS_2(x)=\calT(x)$, however adopting $\calS_2$ is worse than $\calS_1$ for predicting target $\calT$ due to different decision boundaries.}
  \label{fig:limit_marginal}
\end{figure}

To select related sources, most existing works \citep{zhao2018adversarial,peng2019moment,li2018domain,shui2019principled,wang2019multitask,wen2019domain} used the marginal distribution similarity ($\calS_t(x),\calT(x)$) to search the similar tasks. However, this can be problematic if their label distributions are different. As illustrated in Fig.~\ref{fig:limit_marginal}, in a binary classification, the source-target marginal distributions are identical  ($\calS_1(x)=\calS_2(x)=\calT(x)$), however, using $\calS_2$ for helping predict target domain $\calT$ will lead to a negative transfer since their decision boundaries are rather different. This is not only theoretically interesting but also practically demanding. For example, in medical diagnostics, the disease distribution between the countries can be drastically different \citep{liu2004predictive,geiss2014prevalence}. Thus applying existing approaches for leveraging related medical information from other data abundant countries to the destination country will be problematic. 

In this work, we aim to address multi-source deep DA under different label distributions with $\calS_t(y)\neq\calT(y), \calS_t(x|y)\neq\calT(x|y)$, which is more realistic and challenging. In this case, if label information on $\calT$ is absent (unsupervised DA), it is known as a underspecified problem and unsolvable \emph{in the general case} \citep{ben2010impossibility,johansson2019support}. For example, in Figure \ref{fig:limit_marginal}, it is impossible to know the preferable source if there is no label information on the target domain. Therefore, a natural extension is to assume limited label on target domain, which is commonly encountered in practice and a stimulating topic in recent research \citep{mohri2012new,wang2019transfer,saito2019semi,konstantinov2019robust,mansour2020theory}. Based on this, we propose a novel DA theory with limited label on $\calT$ (Theorem \ref{theory:w_o_rep}, \ref{theory:w_rep}), which motivates a novel source selection strategy by mainly considering the similarity of semantic conditional distribution $\Proba(x|y)$ and source re-weighted prediction loss. 

Moreover, in the \emph{specific case}, the proposed source aggregation strategy can be further extended to the unsupervised scenarios. Concretely, in our algorithm, we assume the problem satisfies the Generalized Label Shifted (GLS) condition \citep{combes2020domain}, which is related to the cluster assumption and feasible in many practical applications, as shown in Sec.~\ref{sec_man:gls}.
Based on GLS, we simply add a label distribution ratio estimator, to assist the algorithm in selecting related sources in two popular multi-source scenarios:  unsupervised DA and unsupervised label partial DA  \citep{cao2018partial} with $\text{supp}(\calT(y))\subseteq\text{supp}(\calS_t(y))$ (i.e., inherently label distribution shifted.) 

Compared with previous work, the proposed method has the following benefits:

\textbf{Better Source Aggregation Strategy}~We overcome the limitation of previous selection approaches when label distributions are different by significant improvements. Notably, the proposed approach is shown to simultaneously learn meaningful task relations and label distribution ratio.   

\textbf{Unified Method}~We provide a unified perspective to understand the source selection approach in different scenarios, in which previous approaches regarded them as separate problems. We show their relations in Fig.~\ref{fig:iterative_algo}.

\section{Related Work}
Below we list the most related work and delegate additional related work in the Appendix. 

\textbf{Multi-Source DA} has been investigated in previous literature with different aspects to aggregate source datasets. In the popular unsupervised DA, \citet{zhao2018adversarial, li2018extracting,peng2019moment, wen2019domain, hoffman2018algorithms} adopted the marginal distribution $d(\calS_t(x),\calT(x))$ of $\calH$-divergence \citep{ben2007analysis}, discrepancy \citep{mansour2009domain} and Wasserstein distance \citep{arjovsky2017wasserstein} to estimate domain relations. These works provided theoretical insights through upper bounding the target risk by the source risk, domain discrepancy of $\Proba(x)$ and an un-observable term $\eta$ -- the optimal risk on all the domains. However, as the counterexample indicates, relying on $\Proba(x)$ does not necessarily select the most related source. Therefore, \citet{konstantinov2019robust,wang2019transfer,mansour2020theory} alternatively considered the divergence between two domains with limited target label by using $\calY$-discrepancy, which is commonly faced in practice and less focused in theory. However, we empirically show it is still difficult to handle target-shifted sources.  

\textbf{Target-Shifted DA} \citep{zhang2013domain} is a common phenomenon in DA with $\calS(y)\neq\calT(y)$. Several theoretical analysis has been proposed under label shift assumption with $\calS_t(x|y)=\calT(x|y)$, e.g.\ \citet{azizzadenesheli2018regularized, garg2020unified}. \citet{pmlr-v89-redko19a} proposed optimal transport strategy for the multiple unsupervised DA by assuming $\calS_t(x|y)=\calT(x|y)$. However, this assumption is restrictive for many real-world cases, e.g., in digits dataset, the conditional distribution is clearly different between MNIST and SVHN. In addition, the representation learning based approach is \emph{not} considered in their framework. Therefore, \citet{wu2019domain,combes2020domain} analyzed DA under different assumptions in the \emph{embedding space} $\calZ$ for one-to-one unsupervised deep DA problem but did not provide guidelines of \emph{leveraging different sources} to ensure a reliable transfer, which is our core contribution. 
Moreover, the aforementioned works focus on one specific scenario, without considering its flexibility for other scenarios such as \emph{partial multi-source unsupervised DA}, where the label space in the target domain is a subset of the source domain (i.e., for some classes $\mathcal{S}_t(y)\neq 0$; $\mathcal{T}(y)= 0$) and class distributions are \emph{inherently} shifted.

\section{Problem Setup and Theoretical Insights}

Let $\calX$ denote the input space and $\calY$ the output space. We consider the predictor $h$ as a scoring function \citep{hoffman2018algorithms} with $h: \calX\times\calY \to R$ and predicted loss as $\ell: \R\to\R_{+}$ is positive, $L$-Lipschitz and upper bound by $L_{\max}$. We also assume that $h$ is $K$-Lipschitz w.r.t. the feature $x$ (given the same label), i.e. for $\forall y$, $\|h(x_1,y)-h(x_2,y)\|_2\leq K \|x_1-x_2\|_2$. We denote the expected risk w.r.t distribution $\D$: $R_{\D}(h) = \E_{(x,y)\sim\D}\,\ell(h(x,y))$ and its empirical counterpart (w.r.t. a given dataset $\hat{\D}$) $\hat{R}_{\D}(h) = \sum_{(x,y)\in\hat{\D}}\,\ell(h(x,y))$. 

In this work, we adopt the commonly used Wasserstein distance as the metric to measure domains' similarity, which is theoretically tighter than the previously adopted TV distance \cite{gong2016domain} and Jensen-Shnannon divergence. Besides, based on previous work, a common strategy to adjust the imbalanced label portions is to introduce \emph{label-distribution ratio} weighted loss with
$R^{\alpha}_{\calS}(h) = \E_{(x,y)\sim\calS}\,\alpha(y)\ell(h(x,y))$ with $\alpha(y) = \calT(y)/\calS(y)$.
We also denote $\hat{\alpha}(y)$ as its empirical counterpart, estimated from the data. 

Besides, in order to measure the task relations, we define $\blambda$ ($\blambda[t]\geq 0, \sum_{t=1}^T \blambda[t] =1$) as the \emph{task relation coefficient} vector by assigning higher weight to the more related task. 
Then we prove Theorem 1, which proposes theoretical insights of combining source domains through properly estimating $\blambda$.

\begin{theorem}\label{theory:w_o_rep}
Let $\{\hat{\calS}_t = \{(x_i,y_i)\}_{i=1}^{N_{\calS_t}}\}_{t=1}^T$ and $\hat{\calT} = \{(x_i,y_i)\}_{i=1}^{N_{\calT}}$, respectively be $T$ source and target i.i.d. samples. For $\forall h \in\calH$ with $\calH$ the hypothesis family and $\forall \blambda$, with high probability $\geq 1-4\delta$, the target risk can be upper bounded by:
\begin{small}
\begin{equation*}
\begin{split}
     & R_{\calT}(h) \leq \underbrace{\sum_{t}\blambda[t] \hat{R}^{\hat{\alpha}_t}_{\calS_t}(h)}_{(\RN{1})} 
     + \underbrace{L_{\max}d^{\sup}_{\infty} \sqrt{\sum_{t=1}^{T} \frac{\blambda[t]^2}{\beta_t}}\sqrt{\frac{\log(1/\delta)}{2N}}}_{(\RN{2})}\\
    & \qquad + \underbrace{LK \sum_{t}\blambda[t] \E_{y\sim\hat{\calT}(y)} W_1(\hat{\calT}(x|Y=y)\|\hat{\calS_t}(x|Y=y))}_{(\RN{3})}\\
    & \qquad + \underbrace{L_{\max} \sup_{t} \|\alpha_t-\hat{\alpha}_t\|_2}_{(\RN{4})} + \underbrace{\text{Comp}(N_{\calS_1}, \dots, N_{\calS_T}, N_{\calT},\delta)}_{(\RN{5})},
\end{split}
\end{equation*}
\end{small}\\[-2.5em]

where $N=\sum_{t=1}^T N_{\calS_t}$ and $\beta_t = N_{\calS_t}/N$ and $d_{\infty}^{\sup} = \max_{t\in[1,T], y\in \calY}\alpha_{t}(y)$ the maximum true label distribution ratio value. $W_1(\cdot\|\cdot)$ is the Wasserstein-1 distance with $L_2$-distance as the cost function.
$\text{Comp}(N_{\calS_1}, \dots, N_{\calS_T}, N_{\calT}, \delta)$ is a function that decreases with larger $N_{\calS_1},\dots, N_{\calT}$, given a fixed $\delta$ and hypothesis family $\calH$. (See Appendix for details)
\end{theorem}
\paragraph{Discussions}  (1) In (\RN{1}) and (\RN{3}), the relation coefficient $\blambda$ is decided by $\hat{\alpha}_t$-weighted loss $\hat{R}^{\hat{\alpha}_t}_{\calS_t}(h)$ and conditional Wasserstein distance $\E_{y\sim\hat{\calT}(y)} W_1(\hat{\calT}(x|Y=y)\|\hat{\calS_t}(x|Y=y))$. Intuitively,  a higher $\blambda[t]$ is assigned to the source $t$ with a \emph{smaller weighted prediction loss} and \emph{a smaller weighted semantic conditional Wasserstein distance}. In other words, the source selection depends on the similarity of the conditional distribution $\Proba(x|y)$ rather than $\Proba(x)$.

(2) If each source has equal samples ($\beta_t=1/T$), then term (\RN{2}) will become $\|\blambda\|_2$,  \emph{a regularization term for the encouragement of uniformly leveraging all sources}. Term  (\RN{2}) is meaningful in the selection, because if several sources are simultaneously similar to the target, then the algorithm tends to select \emph{a set of} related domains rather than only one most related domain (without regularization). 

(3) Considering (\RN{1},\RN{2},\RN{3}), we derive a novel source selection approach through the trade-off between assigning a higher $\blambda[t]$ to the source $t$ that has a smaller weighted prediction loss and similar semantic distribution with smaller conditional Wasserstein distance, and assigning balanced $\blambda[t]$ for avoiding concentrating on one source.

(4) $\|\hat{\alpha}_t-\alpha_t\|_2$ (\RN{4}) indicates the gap between ground-truth and empirical label ratio. Therefore, if we can estimate a good label distribution ratio $\hat{\alpha}_t$, these terms can be small.  $\text{Comp}(N_{\calS_1}, \dots, N_{\calS_T}, N_{\calT}, \delta)$ (\RN{5}) is a function that reflects the convergence behavior, which decreases with larger observation numbers. If we fix $\calH, \delta$, $N$ and $N_{\calT}$, this term can be viewed as a constant.

\paragraph{Analysis in the Representation Learning} Apart from Theorem \ref{theory:w_o_rep}, we further drive theoretical analysis in the \emph{representation learning}, which motivates practical guidelines in the deep learning regime. We define a stochastic embedding $g$ and we denote its conditional distribution w.r.t.\ latent variable $Z$ (induced by $g$) as $\calS(z|Y=y) = \int_{x} g(z|x) \calS(x|Y=y) dx$. Then we have:
\begin{theorem}\label{theory:w_rep}

We assume the settings of loss, the hypothesis are the same with Theorem \ref{theory:w_o_rep}. We further denote the stochastic feature learning function $g:\calX\to\calZ$, and the hypothesis $h:\calZ\times\calY \to \R$. Then $\forall \blambda$, the target risk is upper bounded by:\\[-0.5em]
\begin{equation*}
\begin{split}
    & R_{\calT}(h, g) \leq \sum_{t}\blambda[t] R^{\alpha_t}_{\calS_t}(h, g) \\
    & \quad + LK \sum_{t}\blambda[t] \E_{y\sim\calT(y)} W_1(\calS_t(z|Y=y)\| \calT(z|Y=y)),
\end{split}
\end{equation*}

where $R_{\calT}(h,g)=\E_{(x,y)\sim\calT(x,y)}\E_{z\sim g(z|x)} \ell(h(z,y))$ is the expected risk w.r.t.\ the function $g,h$.
\end{theorem}
Theorem 2 motivates the practice of deep learning, which requires to learn an embedding function $g$ that minimizes the weighted conditional Wasserstein distance and learn $(g,h)$ that minimizes the weighted source risk $R_{\calS_t}^{\alpha_t}$.

\section{Practical Algorithm in Deep Learning}\label{sec:framework_main}
From the aforementioned theoretical results, we derive novel source aggregation approaches and training strategies, which can be summarized as follows.

\textbf{Source Selection Rule}~Balance the trade-off between assigning a higher $\blambda[t]$ to the source $t$ that has a smaller weighted prediction loss and semantic conditional Wasserstein distance, and assigning balanced $\blambda[t]$. 

\textbf{Training Rules}~(1) Learning an embedding function $g$ that minimizes the weighted conditional Wasserstein distance, learning classifier $h$ that minimizes the $\hat{\alpha}_t$-weighted source risk; (2) Properly estimate the label distribution ratio $\hat{\alpha}_t$.

Based on these ideas, we proposed Wasserstein Aggregation Domain Network (WADN) to automatically learn the network parameters and select related sources, where the high-level protocol is illustrated in Fig.~\ref{fig:iterative_algo}. 

\begin{figure}[t]
  \centering
    \includegraphics[width=0.45\textwidth]{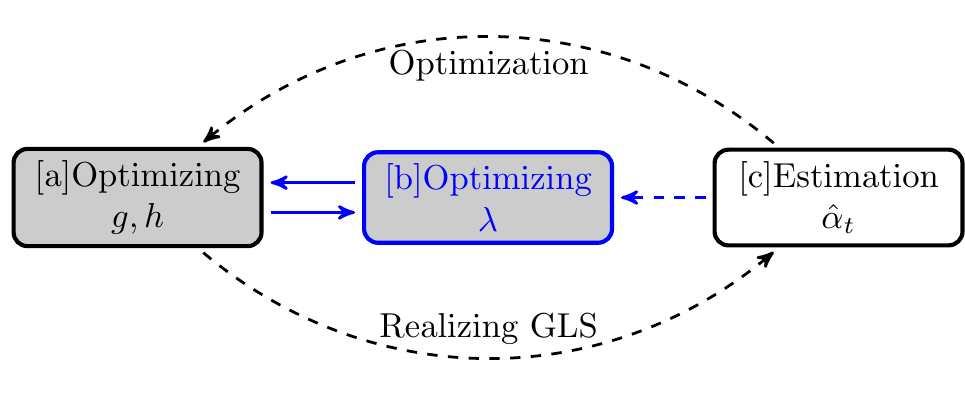}
  \caption{Illustration of proposed algorithm (WADN) and relation with other scenarios. WADN consists three components: \textbf{[a]} learning embedding function $g$ and classifier $h$; \textbf{[b]} source aggregation through properly estimating $\blambda$; \textbf{[c]} label distribution ratio ($\hat{\alpha}_t$) estimator. 
    (1) If target labels are available, then WADN only requires \textbf{[a,b]} without gradually estimating $\hat{\alpha}_t$ (dashed arrows). (2)  In the unsupervised scenarios, if we only have one source, WADN only contains \textbf{[a,c]} and recovers the single DA problem with label proportion shift, which can be solved under specific assumptions such as GLS \citep{li2019target,combes2020domain} or \citep{wu2019domain}. (3) If there are multiple sources in the unsupervised DA, WADN gradually selects the related sources through interacting with other algorithmic components. (shown in blue).}
  \label{fig:iterative_algo}
\end{figure}

\subsection{Training Rules}\label{sec_man:representation}
Based on Theorem 2, given a fixed label ratio $\hat{\alpha}_t$ and fixed $\blambda$, the goal is to find a representation function $g:\calX\to\calZ$ and a hypothesis function $h:\calZ\times\calY\to\R$ such that:
\begin{equation*}
\begin{split}
    \min_{g,h} & \sum_{t}\blambda[t] \hat{R}^{\hat{\alpha}_t}_{\calS_t}(h,g) \\
    & + C_0 \sum_{t}\blambda[t] \E_{y\sim\hat{\calT}(y)} W_1(\hat{\calS}_t(z|Y=y)\| \hat{\calT}(z|Y=y))
\end{split}
\end{equation*}

\paragraph{Explicit Conditional Loss} One can \emph{explicitly} solve the conditional optimal transport problem with $g$ and $h$ for a given $Y=y$. However, due to the high computational complexity in solving $T \times |\calY|$ optimal transport problems, the original form is practically intractable. To address this, we can approximate the conditional distribution on latent space $Z$ as Gaussian distribution with identical Covariance matrix such that $\hat{\calS}_t(z|Y=y) \approx \calN(\mathbf{C}^y_t,\mathbf{\Sigma})$ and $\hat{\calT}(z|Y=y) \approx \calN(\mathbf{C}^y,\mathbf{\Sigma})$. Then we have $W_1(\hat{\calS}_t(z|Y=y)\| \hat{\calT}(z|Y=y)) \leq \|\mathbf{C}^y_t-\mathbf{C}^y\|_2$. Intuitively, the approximation term is equivalent to the well known \emph{feature mean matching}~\citep{sugiyama2012machine}, which computes the feature centroid of each class (on the latent space $Z$) and aligns them by minimizing their $L_2$ distance. 

\paragraph{Implicit Conditional Loss} Apart from approximation, we can derive a dual term for facilitating the computation, which is equivalent to the re-weighted Wasserstein adversarial loss by the label-distribution ratio.  
\begin{lemma}\label{lemma1}
The weighted conditional Wasserstein distance can be implicitly expressed as:
\begin{equation*}
\begin{split}
    & \sum_{t}\blambda[t] \E_{y\sim \calT(y)} W_1(\calS_t(z|Y=y)\| \calT(z|Y=y)) \\
    & =  \max_{d_1,\cdots,d_T} \sum_{t}\blambda[t] [\E_{z\sim\calS_t(z)} \bar{\alpha}_t(z) d_t(z) - \E_{z\sim\calT(z)} d_t(z)],
\end{split}
\end{equation*}
where $\bar{\alpha}_{t}(z) = \mathbf{1}_{\{(z,y)\sim\calS_t\}}\alpha_t(Y=y) $, 
and $d_1,\dots,d_T:\calZ\to R_{+}$ are the $1$-Lipschitz domain discriminators~\citep{ganin2016domain}.
\end{lemma}
Lemma \ref{lemma1} reveals that one can train $T$ domain discriminators with weighted Wasserstein adversarial loss. When the source target distributions are identical, this loss recovers the conventional Wasserstein adversarial loss \citep{arjovsky2017wasserstein}. In practice, we adopt a hybrid approach by linearly combining the explicit and implicit matching, in which empirical results show its effectiveness.
\paragraph{Estimation $\hat{\alpha}$} When the target labels are available, $\hat{\alpha}_t$ can be directly estimated from the data with $\hat{\alpha}_t(y) = \hat{\calT}(y)/\hat{\calS}(y)$ and $\hat{\alpha}_t\to \alpha_t$ can be proved from asymptotic statistics. As for the unsupervised scenarios, we will discuss in Sec. \ref{sec:uda_alpha}.

\subsection{Estimation Relation Coefficient $\blambda$} \label{sec:slove_lambda}
Inspired by Theorem 1, given a \emph{fixed} $\hat{\alpha}_t$ and $(g,h)$, we estimate $\blambda$ through optimizing the derived upper bound.
\begin{equation*}
    \begin{split}
    \min_{\blambda} \quad & \sum_{t}\blambda[t] \hat{R}^{\hat{\alpha}_t}_{\calS_t}(h,g) + C_1 \sqrt{\sum_{t=1}^{T} \frac{\blambda^2[t]}{\beta_t}}\\
    & + C_0 \sum_{t}\blambda[t] \E_{y\sim\hat{\calT}(y)} W_1(\hat{\calT}(z|Y=y)\|\hat{\calS}(z|Y=y))\\
    \text{s.t} &\quad \forall t, \blambda[t] \geq 0, \sum_{t=1}^{T}\blambda[t] = 1
    \end{split}
\end{equation*}
In practice, $\hat{R}^{\hat{\alpha}_t}_{\calS_t}(h,g)$ is the weighted empirical prediction loss and
$\E_{y\sim\hat{\calT}(y)} W_1(\hat{\calT}(z|Y=y)\|\hat{\calS}(z|Y=y))$ is approximated by the dynamic form of critic function from Lemma \ref{lemma1}. Then, solving $\blambda$ can be viewed as a standard convex optimization problem with linear constraints, which can be effectively resolved through standard convex optimizer.

\section{Extension to Unsupervised Scenarios}\label{sec_man:gls}
In this section, we extend WADN to the unsupervised multi-source DA, which is known as unsolvable if semantic conditional distribution ($\calS_t(x|y)\neq \calT(x|y)$) and label distribution ($\calS_t(y)\neq \calT(y)$) are simultaneously different and no specific conditions are considered \citep{ben2010impossibility,johansson2019support}. 

In algorithm WADN, this challenging turns to properly estimate conditional Wasserstein distance and label distribution ratio $\hat{\alpha}_t(y)$ to help estimate $\blambda$. According to Lemma \ref{lemma1}, estimating the conditional Wasserstein distance can be viewed as $\hat{\alpha}_t$-weighted adversarial loss, thus if we can correctly estimate label distribution ratio such that $\hat{\alpha}_t\to\alpha_t$, then we can properly compute the conditional Wasserstein-distance through the adversarial term.

Therefore, the problem turns to properly estimate the label distribution ratio. To this end, we assume the problem satisfies Generalized Label Shift (GLS) condition  \citep{combes2020domain}, which has been theoretically justified and empirically evaluated in the single source unsupervised DA. The GLS condition states that \emph{in unsupervised DA, there exists an optimal embedding function $g^{\star}\in\mathcal{G}$ that can ultimately achieve $\calS_t(z|y)=\calT(z|y)$ on the latent space.} \citep{combes2020domain} further pointed out that the clustering assumption \citep{chapelle2005semi} on $\calZ$ is one sufficient condition to reach GLS, which is feasible for many practical applications. 

Based on the achievability condition of GLS, the techniques of \citep{lipton2018detecting,garg2020unified} can be adopted to gradually estimate $\hat{\alpha}_t$ during learning the embedding function. Following this spirit, we add an  distribution ratio estimator for $\{\hat{\alpha}_t\}_{t=1}^{T}$, shown in Sec.~\ref{sec:uda_alpha}.

\subsection{Estimation $\hat{\alpha}_t$}\label{sec:uda_alpha}
\paragraph{Unsupervised DA} We denote $\bar{\calS}_t(y)$, $\bar{\calT}(y)$ as the predicted $t$-source/target label distribution through the hypothesis $h$, and also define $C_{\hat{\calS}_t}[y,k] = \hat{\calS}_t[\text{argmax}_{y^{\prime}} h(z,y^{\prime})=y, Y=k]$ is the $t$-source \emph{prediction confusion matrix}. According to the GLS condition, we have $\bar{\calT}(y)=\bar{\calT}_{\hat{\alpha}_t}(y)$, with $\bar{\calT}_{\hat{\alpha}_t}(Y=y) = \sum_{k=1}^{\calY} C_{\hat{\calS}_t}[y,k]\hat{\alpha}_t(k)$ the constructed target prediction distribution from the $t$-source information. (See Appendix for justification). Then we can estimate $\hat{\alpha}_t$
through matching these two distributions by minimizing $D_{\text{KL}}(\bar{\calT}(y)\|\bar{\calT}_{\hat{\alpha}_t}(y))$, which is equivalent to solve the following convex optimization:
\begin{equation}
\begin{split}
   & \min_{\hat{\alpha}_t} \quad  -\sum_{y=1}^{|\calY|} \bar{\calT}(y) \log(\sum_{k=1}^{|\calY|} C_{\hat{\calS}_t}[y,k]\hat{\alpha}_t(k)) \\ 
   & \text{s.t} \quad \forall y\in\calY, \hat{\alpha}_t(y)\geq 0, \quad \sum_{y=1}^{|\calY|} \hat{\alpha}_t(y)\hat{\calS}_t(y) = 1
\end{split}
    \label{eq:solving_alpha}
\end{equation}

\paragraph{Unsupervised Partial DA} If we have $\text{supp}(\calT(y))\subseteq \text{supp}(\calS_t(y))$, ${\alpha}_t$ will be sparse due to the non-overlapped classes. Thus, we impose such prior knowledge by adding a regularizer $\|\hat{\alpha}_t\|_1$ to the objective of Eq.~(\ref{eq:solving_alpha}) to induce the sparsity in $\hat{\alpha}_t$. 

In training the neural network, the non-overlapped classes will be automatically assigned with a small or zero $\hat{\alpha}_t$, then $(g,h)$ will be less affected by the classes with small $\hat{\alpha}_t$. 

\begin{center}
\begin{algorithm}[t]
		\caption{WADN (unsupervised scenario, one epoch)}
		\begin{algorithmic}[1] 
        \ENSURE  Label ratio $\hat{\alpha}_t$ and task relation $\blambda$. Feature Learner $g$, Classifier $h$, statistic critic function $d_1,\dots,d_T$, class centroid for source $\mathbf{C}_t^y$ and target $\mathbf{C}^y$.($t=1,\dots,T$)
        \STATE \(\triangleright\) DNN Parameter Training Stage (fixed $\alpha_t$ and $\blambda$) \(\triangleleft\)
        \FOR{mini-batch of samples $(\x_{\calS_1},\y_{\calS_1})\sim\hat{\calS}_1$, $\dots$, $(\x_{\calS_T},\y_{\calS_T})\sim\hat{\calS}_T$, $(\x_{\calT})\sim\hat{\calT}$ }
        \STATE Target predicted-label $\bar{\y}_{\calT} = \text{argmax}_{y} h(g(\x_{\calT}),y)$
        \STATE Compute unnormalized source confusion matrix on current \emph{batch} $C_{\hat{\calS}_t}[y,k]$.
        \STATE Compute feature centroid for source $C_t^y$  and target $C^y$ on current \emph{batch}; Use moving average to update source and target class centroid $\mathbf{C}_t^y$ and $\mathbf{C}^y$.
        \STATE Updating $g,h,d_1,\dots,d_T$, by optimizing:
\begin{small}
\begin{multline*}
  \min_{g,h} \max_{d_1,\dots,d_T}  \underbrace{\sum\mathop{}_{\mkern-3mu t}\blambda[t]\hat{R}^{\hat{\alpha}_t}_{\calS_t}(h, g)}_{\text{Classification Loss}} \\
  + \epsilon C_0  \underbrace{\sum\mathop{}_{\mkern-3mu t}\blambda[t] \E_{y\sim\bar{\calT}(y)}\|\mathbf{C}_t^y - \mathbf{C}^y\|_2}_{\text{Explicit Conditional Loss}} 
  \\
  + (1-\epsilon) C_0 \underbrace{\sum\mathop{}_{\mkern-3mu t}\blambda[t][\E_{z\sim\hat{\calS}_t(z)} \bar{\alpha}_t(z) d(z) - \E_{z\sim\hat{\calT}(z)} d(z)]}_{\text{Implicit Conditional Loss}}
\end{multline*}
\end{small}
        \ENDFOR
        \STATE \(\triangleright\) Estimation $\hat{\alpha}_t$ and $\blambda$ \(\triangleleft\)
		\STATE Compute normalized source confusion matrix; Solve $\{\hat{\alpha}_t\}_{t=1}^T$ w.r.t. current training epoch through Sec.\ref{sec:uda_alpha} ; Update global $\hat{\alpha}_t$ through moving average.
		\STATE Solve $\blambda$ through Sec.\ref{sec:slove_lambda} w.r.t. current training epoch; Update global $\blambda$ through moving average.
        \end{algorithmic}
        \label{WMDARN_algo_main}
\end{algorithm}
\end{center}

\begin{table*}[t]
\centering
\caption{Unsupervised DA: Accuracy $(\%)$ on \emph{Source-Shifted} Amazon Review (Left) and Digits (Right). }
\label{tab:uda_amazon_digits}
\vskip 0.1in
\resizebox{0.49\textwidth}{!}{
\begin{tabular}{c|cccc|c}
\toprule
Target & Books & DVD &  Electronics & Kitchen  & Average  \\ \midrule
Source &   68.15$_{\pm 1.37}$   &  69.51$_{\pm 0.74}$ &      82.09$_{\pm 0.88}$  &  75.30$_{\pm 1.29}$ &  73.81 \\ \midrule
DANN   &  65.59$_{\pm 1.35}$    & 67.23$_{\pm 0.71}$ &  80.49$_{\pm 1.11}$     & 74.71$_{\pm 1.53}$  &   72.00 \\ \midrule
MDAN   &    68.77$_{\pm 2.31}$  & 67.81$_{\pm 2.46}$ &  80.96$_{\pm 0.77}$   & 75.67$_{\pm 1.96}$  & 73.30 \\ \midrule
MDMN   & 70.56$_{\pm 1.05}$  &69.64$_{\pm 0.73}$     &  82.71$_{\pm 0.71}$    & 77.05$_{\pm 0.78}$  &  74.99\\ \midrule
M$^3$SDA  & 69.09$_{\pm 1.26}$ & 68.67$_{\pm 1.37}$ & 81.34$_{\pm 0.66}$ & 76.10$_{\pm 1.47}$  & 73.79 \\ \midrule
DARN   &71.21$_{\pm 1.16}$  & 68.68$_{\pm 1.12}$ &81.51$_{\pm 0.81}$      & 77.71$_{\pm 1.09}$   & 74.78  \\ \midrule
WADN   &    \textbf{73.72}$_{\pm 0.63}$  & \textbf{79.64}$_{\pm 0.34}$ & \textbf{84.64}$_{\pm 0.48}$    & \textbf{83.73}$_{\pm 0.50}$   & \textbf{80.43}  \\ \bottomrule
\end{tabular}}
\hfill
\resizebox{0.49\textwidth}{!}{
\begin{tabular}{c|cccc|c}
\toprule
Target & MNIST & SVHN  &  SYNTH  & USPS  & Average  \\ \midrule
Source &  84.93$_{\pm 1.50}$   &  67.14$_{\pm 1.40}$   &   78.11$_{\pm 1.31}$   & 86.02$_{\pm 1.12}$ & 79.05 \\ \midrule
DANN   &  86.99$_{\pm 1.53}$   &  69.56$_{\pm 2.26}$   &   78.73$_{\pm 1.30}$   & 86.81$_{\pm 1.74}$ & 80.52 \\ \midrule
MDAN   &  87.86$_{\pm 2.24}$   & 69.13$_{\pm 1.56}$     &    79.77$_{\pm 1.69}$  & 86.50$_{\pm 1.59}$ & 80.81 \\ \midrule
MDMN   &  87.31$_{\pm 1.88}$     &  69.84$_{\pm 1.59}$    & 80.27$_{\pm 0.88}$     & 86.61$_{\pm 1.41}$  & 81.00 \\ \midrule
M$^3$SDA  &  87.22$_{\pm 1.70}$   &  68.89$_{\pm 1.93}$   &   80.01$_{\pm 1.77}$   & 86.39$_{\pm 1.68}$ & 80.87 \\ \midrule
DARN   &   86.98$_{\pm 1.29}$       &   68.59$_{\pm 1.79}$   &  80.68$_{\pm 0.61}$  &86.85$_{\pm 1.78}$  &  80.78\\ \midrule
WADN   &   \textbf{89.07}$_{\pm 0.72}$      &   \textbf{71.66}$_{\pm 0.77}$     &  \textbf{82.06}$_{\pm 0.89}$   & \textbf{90.07}$_{\pm 1.10}$  & \textbf{83.22} \\ \bottomrule
\end{tabular}}
\end{table*}

\begin{figure}[b]
    \centering
    \includegraphics[scale=0.40]{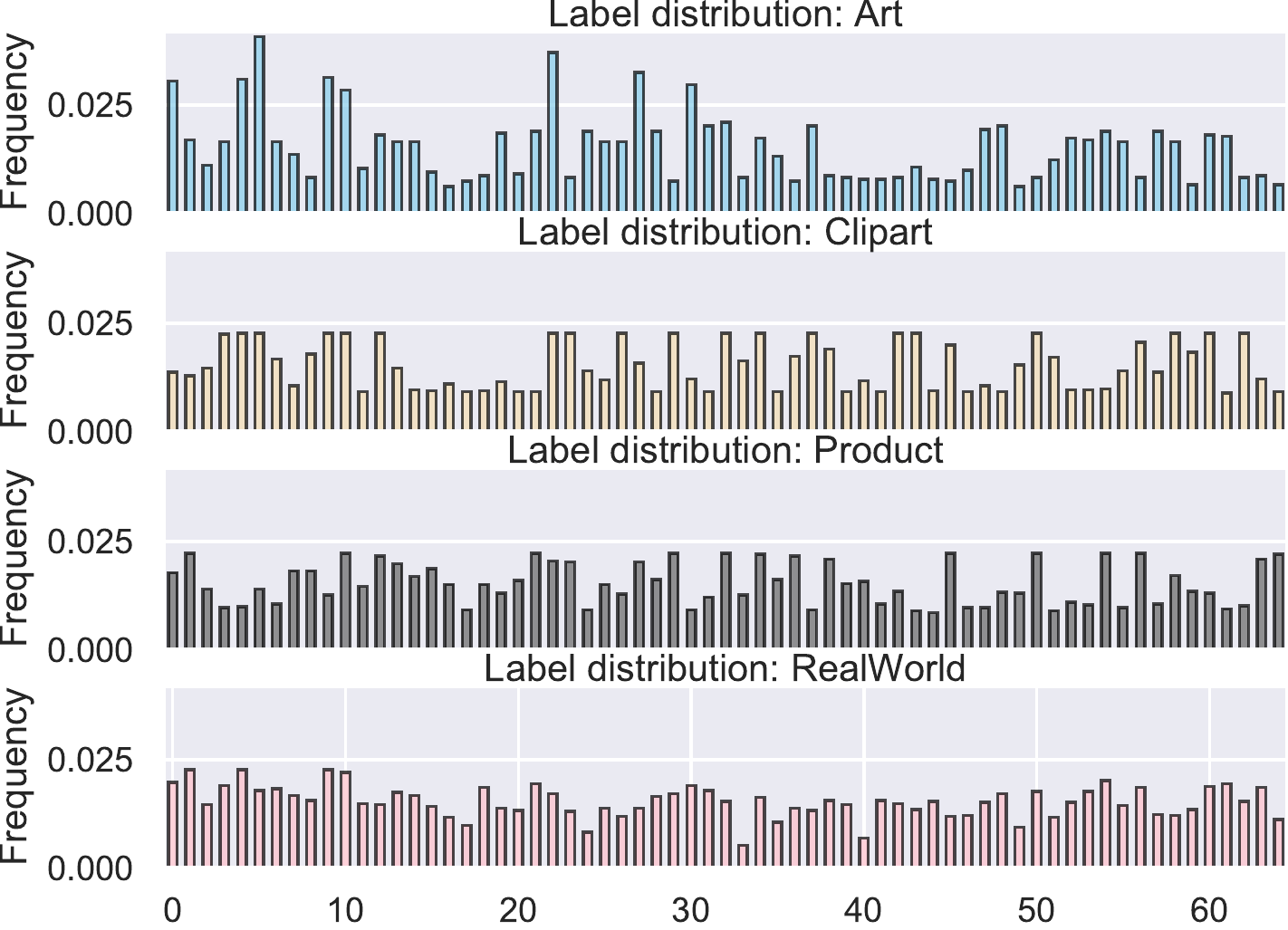}
    \caption{Label distribution on Office-Home Dataset}
    \label{fig:office-home-label}
\end{figure}

\begin{table}[!t]
\centering
\caption{Unsupervised DA: Accuracy $(\%)$ on Office-Home}
\label{tab:uda_office_home}
\vskip 0.1in
\resizebox{0.5\textwidth}{!}{
\begin{tabular}{c|cccc|c}
\toprule
Target & Art & Clipart &  Product & Real-World  & Average  \\ \midrule
Source & 49.25$_{\pm 0.60}$   & 46.89$_{\pm 0.61}$    & 66.54$_{\pm 1.72}$   & 73.64$_{\pm 0.91}$   & 59.08    \\ \midrule
DANN   & 50.32$_{\pm 0.32}$   & 50.11$_{\pm 1.16}$    & 68.18$_{\pm 1.27}$   & 73.71$_{\pm 1.63}$   & 60.58    \\ \midrule
MDAN   & 67.93$_{\pm 0.36}$   & 66.61$_{\pm 1.32}$    & 79.24$_{\pm 1.52}$   & 81.82$_{\pm 0.65}$   & 73.90    \\ \midrule
MDMN   & 68.38$_{\pm 0.58}$   & 67.42$_{\pm 0.53}$    & 82.49$_{\pm 0.56}$   & 83.32$_{\pm 1.93}$   & 75.28    \\ \midrule
M$^3$SDA  & 63.77$_{\pm 1.07}$   & 62.30$_{\pm 0.44}$     & 75.85$_{\pm 1.24}$     & 79.92$_{\pm 0.60}$  &  70.46 \\ \midrule
DARN   & 69.89$_{\pm 0.42}$     & 68.61$_{\pm 0.50}$     & 83.37$_{\pm 0.62}$     & 84.29$_{\pm 0.46}$  &  76.54 \\ \midrule
WADN   & \textbf{73.78}$_{\pm 0.43}$     & \textbf{70.18}$_{\pm 0.54}$     & \textbf{86.32}$_{\pm 0.38}$     & \textbf{87.28}$_{\pm 0.87}$  &  \textbf{79.39} \\ \bottomrule
\end{tabular}}
\end{table}

\subsection{Algorithm implementation and discussion}
We give an algorithmic description of Fig.~\ref{fig:iterative_algo}, shown in Algorithm \ref{WMDARN_algo_main}.
The high-level protocol is to \emph{iteratively} optimizes the neural-network parameters to gradually realize GLS condition with $g \to g^{\star}$ and dynamically update $\blambda$, $\hat{\alpha}_t$ to better estimate conditional distance and aggregate the sources. The GLS assumes the achievability of existing an optimal $g^{\star}$. Our iterative algorithm can achieve a stationary solution but due to the highly non-convexity of deep network, converging to the global optimal does not necessarily guarantee.

Concretely, we update the $\hat{\alpha}_t$ and $\blambda$ on the fly through a moving averaging strategy. Within one training epoch over the mini-batches, we fix the $\hat{\alpha}_t$ and $\blambda$ and optimize the network parameters $g,h$. Then at each training epoch, we re-estimate the $\hat{\alpha}_t$ and $\blambda$ by using the proposed estimator. When computing the explicit conditional loss, we empirically adopt the target pseudo-label. The implicit and explicit trade-off coefficient is set as $\epsilon=0.5$. As for optimization $\blambda$ and $\alpha_t$, it is a standard convex optimization problem and we use package CVXPY.

As for WADN with limited target label, we do not require label distribution ratio component and directly compute $\hat{\alpha}_t$.

\section{Experiments}\label{sec:experments}
In this section, we compare the proposed approaches with several baselines on the popular tasks. For all the scenarios, the following multi-source DA baselines are evaluated: (\RN{1}) \textbf{Source} method applied only labelled source data to train the model. (\RN{2}) \textbf{DANN} \citep{ganin2016domain}. We follow the protocol of \cite{wen2019domain} to merge all the source dataset as a global source domain.  (\RN{3}) \textbf{MDAN} \citep{zhao2018adversarial}; (\RN{4}) \textbf{MDMN} \citep{li2018extracting}; (\RN{5}) \textbf{M$^3$SDA} \citep{peng2019moment} adopted maximizing classifier discrepancy \citep{saito2018maximum} and 
(\RN{6}) \textbf{DARN} \citep{wen2019domain}. For the multi-source with limited target label and partial unsupervised multi-source DA, we additionally add specific baselines. All the baselines are re-implemented in the same network structure for fair comparisons. The detailed network structures, hyper-parameter settings, training details are delegated in Appendix.

We evaluate the performance on three different datasets: (1) \textbf{Amazon Review.} \citep{blitzer2007biographies} It contains four domains 
(Books, DVD, Electronics, and Kitchen) with positive and negative product reviews. We follow the common data pre-processing strategies as \citep{chen2012marginalized} to form a $5000$-dimensional bag-of-words feature. Note that the label distribution in the original dataset is uniform. \emph{To show the benefits of the proposed approach, we create a label distribution drifted task by randomly dropping $50\%$ negative reviews of all the sources while keeping the target identical.} (2) \textbf{Digits}. It consists four digits recognition datasets including MNIST, USPS \citep{hull1994database}, SVHN \citep{netzer2011reading} and Synth \citep{ganin2016domain}. \emph{We also create a label distribution drift for the sources by randomly dropping $50\%$ samples on digits 5-9 and keep target identical.}  (3) \textbf{Office-Home Dataset} \citep{venkateswara2017deep}. It contains 65 classes for four different domains: Art, Clipart, Product and Real-World. We used the ResNet50 \citep{he2016deep} pretrained from the ImageNet in PyTorch as the base network for feature learning and put a MLP for the classification. The label distributions in these four domains are different and we did not manually create a label drift, shown in Fig.~\ref{fig:office-home-label}.

\begin{table*}[t]
\centering
\caption{Multi-Source DA with Limited Target Label: Accuracy $(\%)$ on \emph{Source-Shifted} Amazon Review (Left) and Digits (Right).}
\label{tab:transfer_amazon}
\vskip 0.1in 
\resizebox{0.49\textwidth}{!}{
\begin{tabular}{c|cccc|c}
\toprule
Target & Books & DVD &  Electronics & Kitchen  & Average  \\ \midrule
Source + Tar &    72.59$_{\pm 1.89}$ &  73.02$_{\pm 1.84}$    & 81.59$_{\pm 1.58}$    & 77.03$_{\pm 1.73}$  & 76.06  \\ \midrule
DANN   &    67.35$_{\pm 2.28}$       &   66.33$_{\pm 2.42}$      & 78.03$_{\pm 1.72}$     & 74.31$_{\pm 1.71}$  & 71.50  \\ \midrule
MDAN   &    68.70$_{\pm 2.99}$       &    69.30$_{\pm 2.21}$       & 78.78$_{\pm 2.21}$ & 74.07$_{\pm 1.89}$  & 72.71  \\ \midrule
MDMN   &    69.19$_{\pm 2.09}$       &     68.71$_{\pm 2.39}$        & 81.88$_{\pm 1.46}$  & 78.51$_{\pm 1.91}$  & 74.57  \\ \midrule
M$^3$SDA  &  69.28$_{\pm 1.78}$ &    67.40$_{\pm 0.46}$     & 76.28$_{\pm 0.81}$    & 76.50$_{\pm 1.19}$  & 72.36  \\ \midrule
DARN   & 68.57$_{\pm 1.35}$     &    68.77$_{\pm 1.81}$     & 80.19$_{\pm 1.66}$    & 77.51$_{\pm 1.20}$  & 73.76 \\ \midrule
RLUS   &    71.83$_{\pm 1.71}$      &    69.64$_{\pm 2.39}$     & 81.98$_{\pm 1.04}$    & 78.69$_{\pm 1.15}$  & 75.54 \\ \midrule
MME   &    69.66$_{\pm 0.58}$  &   71.36$_{\pm 0.96}$  &   78.88$_{\pm 1.51}$   &  76.64$_{\pm 1.73}$  &  74.14 
  \\ \midrule  
WADN   & \textbf{74.83}$_{\pm 0.84}$     &    \textbf{75.05}$_{\pm 0.62}$    & \textbf{84.23}$_{\pm 0.58}$    & \textbf{81.53}$_{\pm 0.90}$  & \textbf{78.91} \\ \bottomrule
\end{tabular}}
\hfill
\resizebox{0.49\textwidth}{!}{
\begin{tabular}{c|cccc|c}
\toprule
Target & MNIST & SVHN  &  SYNTH  & USPS  & Average  \\ \midrule
Source + Tar &   79.63$_{\pm 1.74}$      & 56.48$_{\pm 1.90}$       & 69.64$_{\pm 1.38}$     & 86.29$_{\pm 1.56}$  &  73.01 \\ \midrule
DANN   & 86.77$_{\pm 1.30}$              & 69.13$_{\pm 1.09}$       & 78.82$_{\pm 1.35}$     & 86.54$_{\pm 1.03}$  &  80.32 \\ \midrule
MDAN   & 86.93$_{\pm 1.05}$              & 68.25$_{\pm 1.53}$       & 79.80$_{\pm 1.17}$     & 86.23$_{\pm 1.41}$  &  80.30 \\ \midrule
MDMN   & 77.59$_{\pm 1.36}$              & 69.62$_{\pm 1.26}$       & 78.93$_{\pm 1.64}$     & 87.26$_{\pm 1.13}$  &  78.35 \\ \midrule
M$^3$SDA  & 85.88$_{\pm 2.06}$           & 68.84$_{\pm 1.05}$       & 76.29$_{\pm 0.95}$     & 87.15$_{\pm 1.10}$  &  79.54 \\ \midrule
DARN   & 86.58$_{\pm 1.46}$              & 68.86$_{\pm 1.30}$       & 80.47$_{\pm 0.67}$     & 86.80$_{\pm 0.89}$  &  80.68 \\ \midrule
RLUS   &  87.61$_{\pm 1.08}$             & \textbf{70.50}$_{\pm 0.94}$       & 79.52$_{\pm 1.30}$     & 86.70$_{\pm 1.13}$  &  81.08  \\ \midrule
MME   & 87.24$_{\pm 0.95}$  &  65.20$_{\pm1.35}$  &  80.31$_{\pm 0.60}$  &  87.88$_{\pm 0.76}$  &  80.16 
 \\ \midrule
WADN   & \textbf{88.32}$_{\pm 1.17}$     & \textbf{70.64}$_{\pm 1.02}$       & \textbf{81.53}$_{\pm 1.11}$     & \textbf{90.53}$_{\pm 0.71}$  &  \textbf{82.75}  \\ \bottomrule
\end{tabular}}
\end{table*}

\begin{figure*}[t]
\centering
\begin{subfigure}{0.32\textwidth}
\centering
     \includegraphics[width=0.55\textwidth]{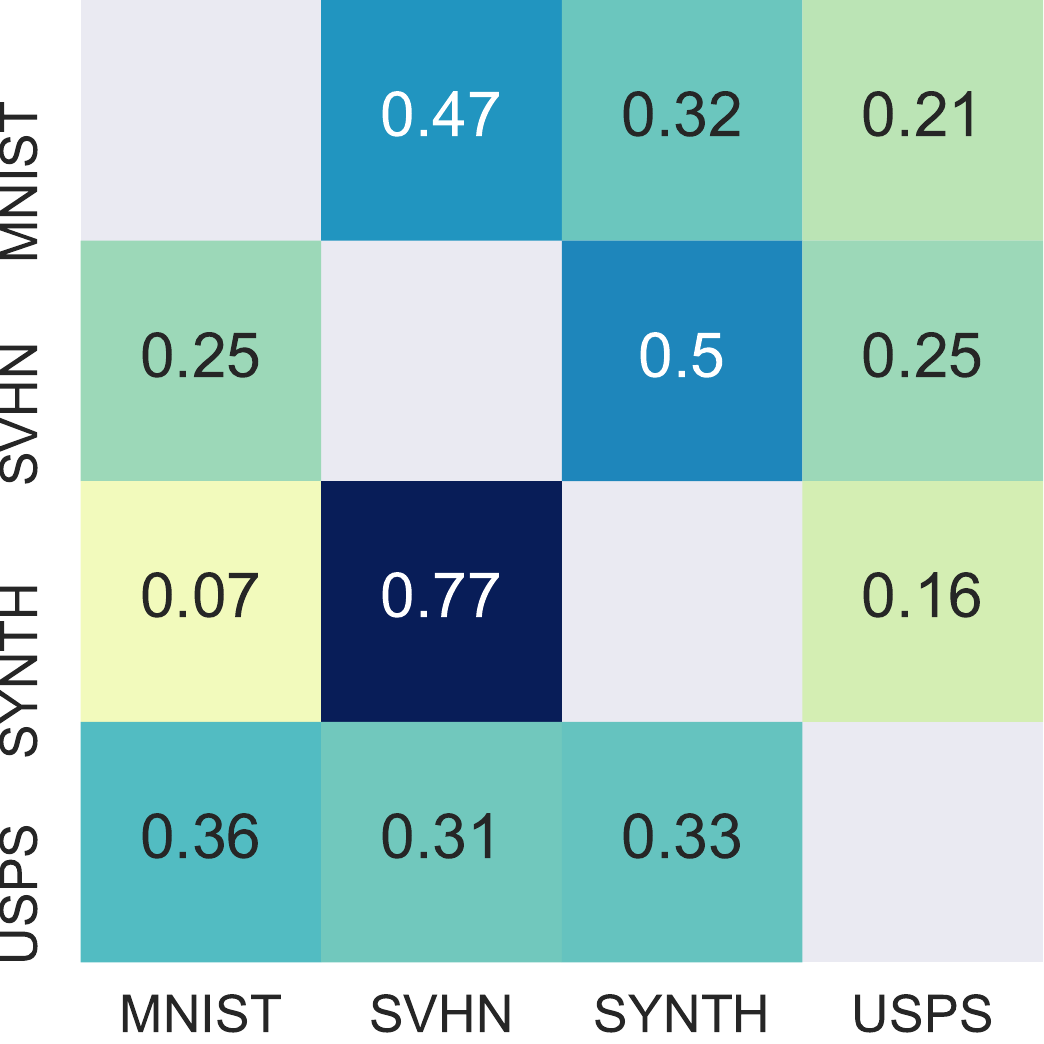}
     \caption{Visualization of $\blambda$}
 \end{subfigure}
\begin{subfigure}{0.32\textwidth}
     \includegraphics[width=1.0\textwidth]{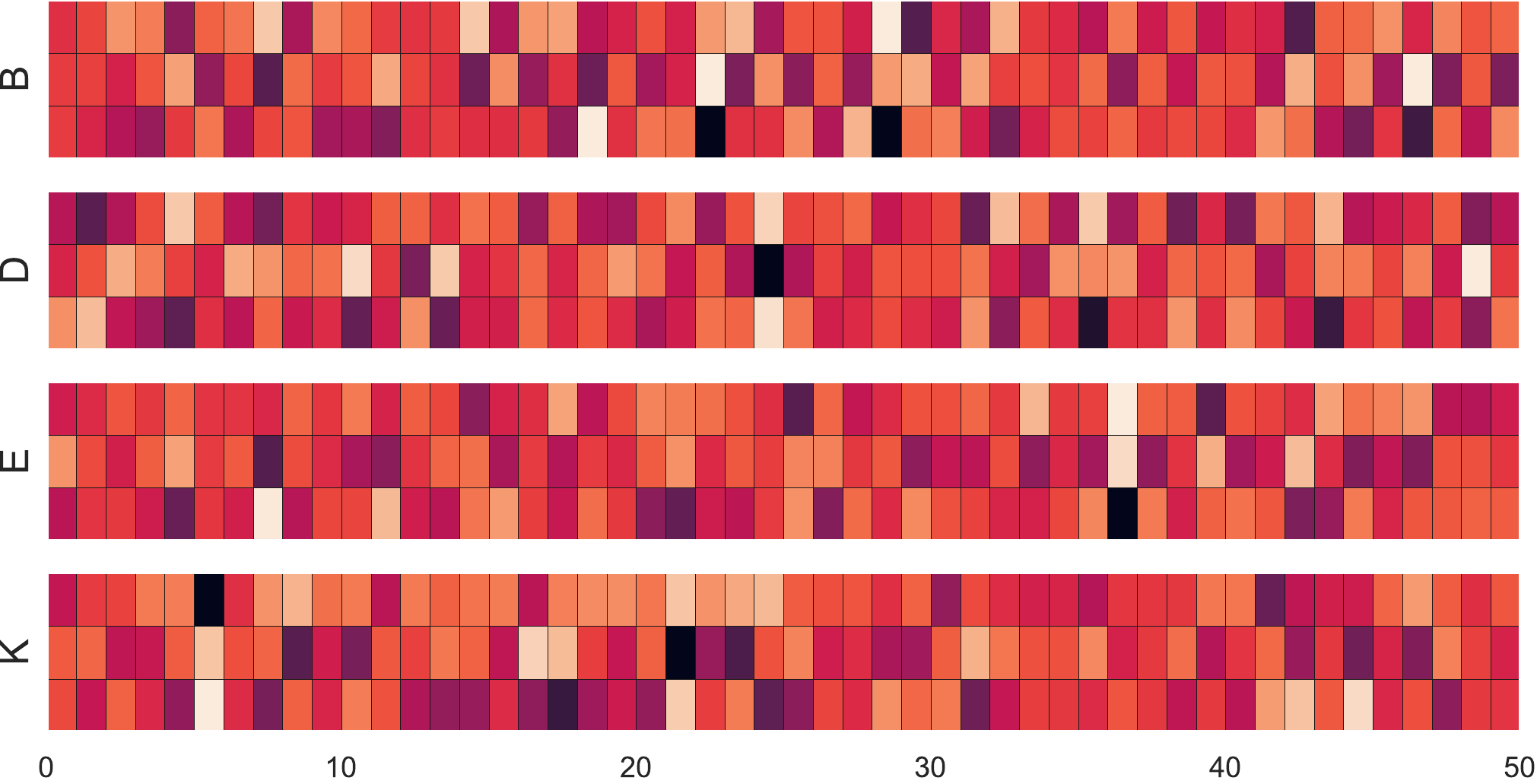}
     \caption{DARN \citep{wen2019domain}}
  \end{subfigure}
  \quad
  \begin{subfigure}{0.32\textwidth}
     \includegraphics[width=1.0\textwidth]{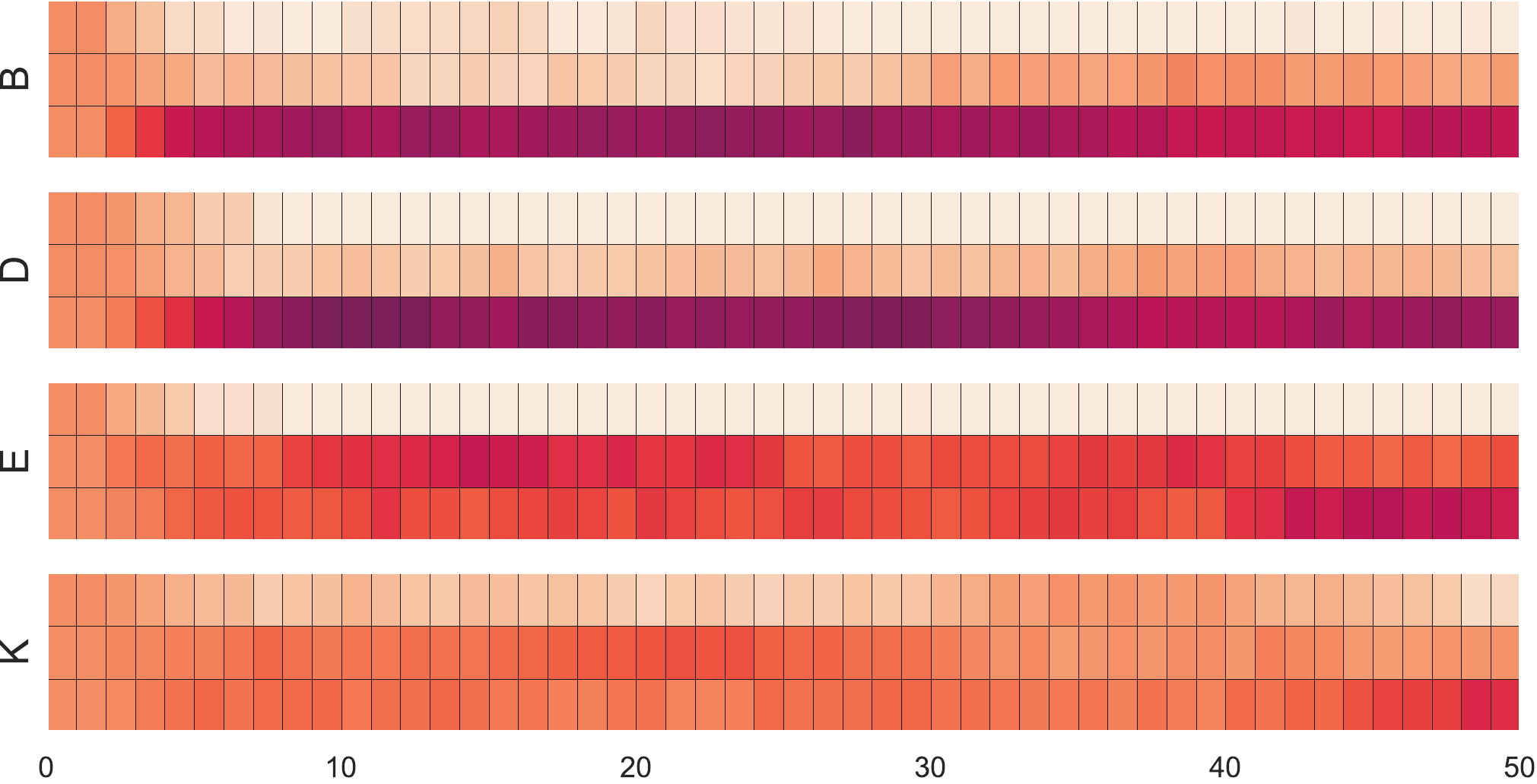}
     \caption{WADN}
  \end{subfigure}
  \caption{Understanding Aggregation Principles in Unsupervised DA. (a) Visualization of $\blambda$ on digits datset, each row corresponds to a target domain, which indicates a \emph{non-uniform} and \emph{non-symmetric} task relations.
  (b,c) The evolution of $\blambda$ with three sources of Amazon dataset (B=Books, D=DVD, E=Electronics, K=Kitchen) during the training epoch. We compare with a recent principle approach DARN, which uses $\Proba(x)$ to measure the similarity and dynamically update the $\blambda$. The results verifies the limitation of DARN under changing label distributions with relative unstable results. } 
  \label{fig:analysis_lambda}
\end{figure*}

\subsection{Unsupervised Multi-Source DA}
In the unsupervised multi-source DA, we evaluate the proposed approach on all three datasets. We use a similar hyper-parameter selection strategy as in DANN \citep{ganin2016domain}. All reported results are averaged from five runs. The detailed experimental settings are illustrated in Appendix. The empirical results are illustrated in Tab. \ref{tab:uda_amazon_digits} and \ref{tab:uda_office_home}. Since we did not change the target label distribution throughout the whole experiment, we still report the target accuracy as the metric. We report the means and standard deviations for each approach. The best approaches based on a two-sided Wilcoxon signed-rank test (significance level $p=0.05$) are shown in bold.


The empirical results reveal a significantly better performance ($\approx 2\%-6\%$) on different benchmarks. For understanding the aggregation principles of WADN, we visualize the task relations in digits (Fig.~\ref{fig:analysis_lambda}(a)) with demonstrating a \emph{non-uniform} $\blambda$, which highlights the importance of properly choosing the most related source rather than simply merging all the data. For example, when the target domain is SVHN, WADN mainly leverages the information from SYNTH, since they are more semantically similar, and MNIST does not help too much for SVHN, which is also observed by \cite{ganin2016domain}. Besides, Fig.~\ref{fig:analysis_lambda}(b) visualizes the evolution of $\blambda$ between WADN and recent principled approach DARN \citep{wen2019domain}, which utilized the $\Proba(x)$ information and dynamic updating to find the similar domains. Compared with WADN, $\blambda$ in DARN is \emph{unstable} during updating under drifted label distribution. 

Besides, we conduct the ablation study through evaluating the performance under different levels of source label shift in Amazon Review dataset (Fig.~\ref{fig:paper_aba_study}(a)). The results show strong practical benefits for WADN in the larger label shift. The additional analysis and results can be found in Appendix.

\subsection{Multi-Source DA with Limited Target Labels}
We adopt Amazon Review and Digits in the multi-source DA with limited target samples, which have been widely used. In the experiments, we still use shifted sources. We randomly sample only $10\%$ labeled samples (w.r.t. target dataset in unsupervised DA) as training set and the rest $90\%$ samples as the unseen target test set. We adopt the same hyper-parameters and training strategies with unsupervised DA. We specifically add two recent baselines RLUS \citep{konstantinov2019robust} and MME \citep{saito2019semi}, which also considered DA with the labeled target domain.

\begin{figure*}[t]
  \centering
  \begin{subfigure}{0.32\textwidth}
  \centering
     \includegraphics[scale=0.32]{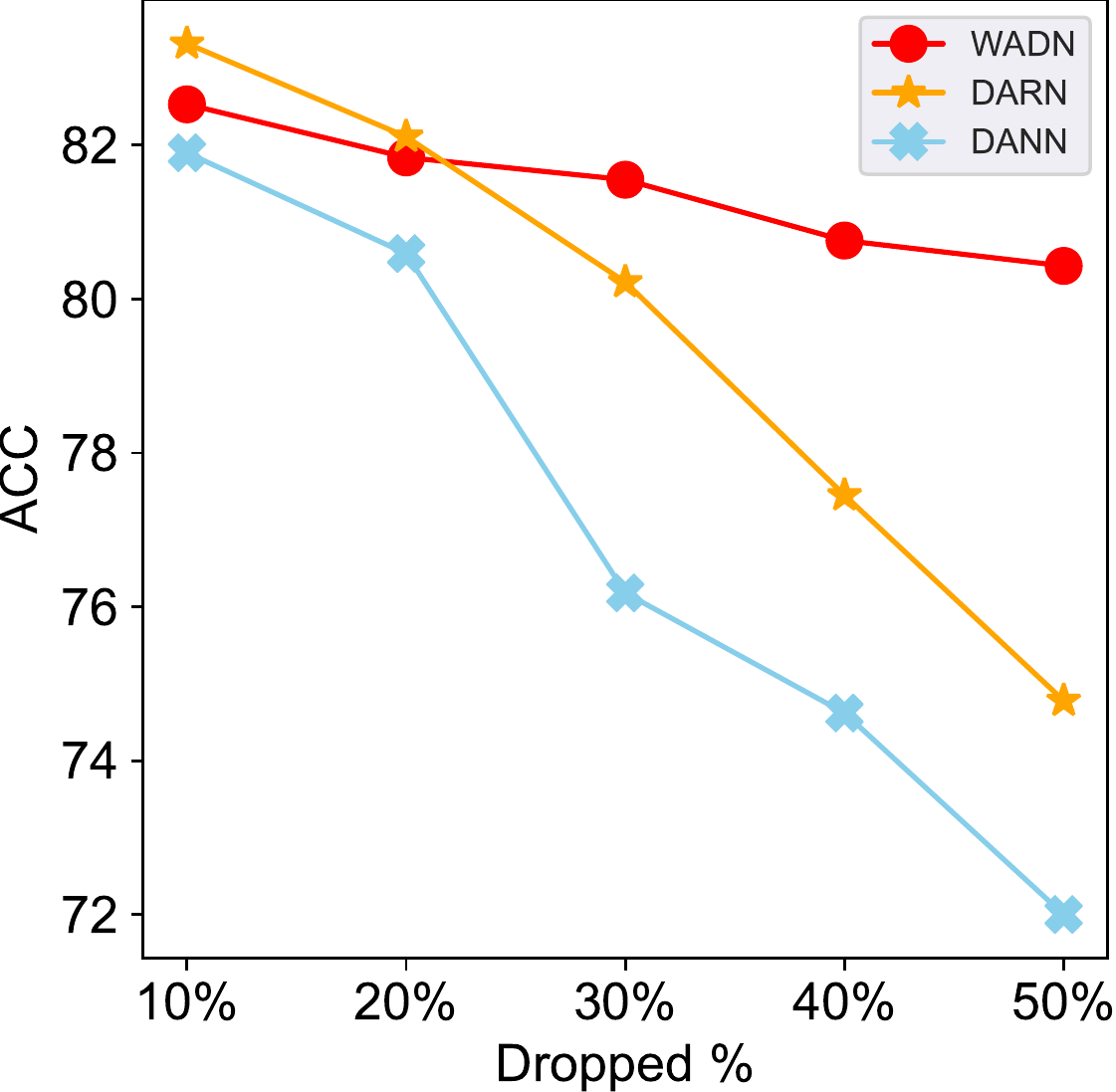}
     \caption{}
  \end{subfigure}
  \begin{subfigure}{0.32\textwidth}
  \centering
     \includegraphics[scale=0.32]{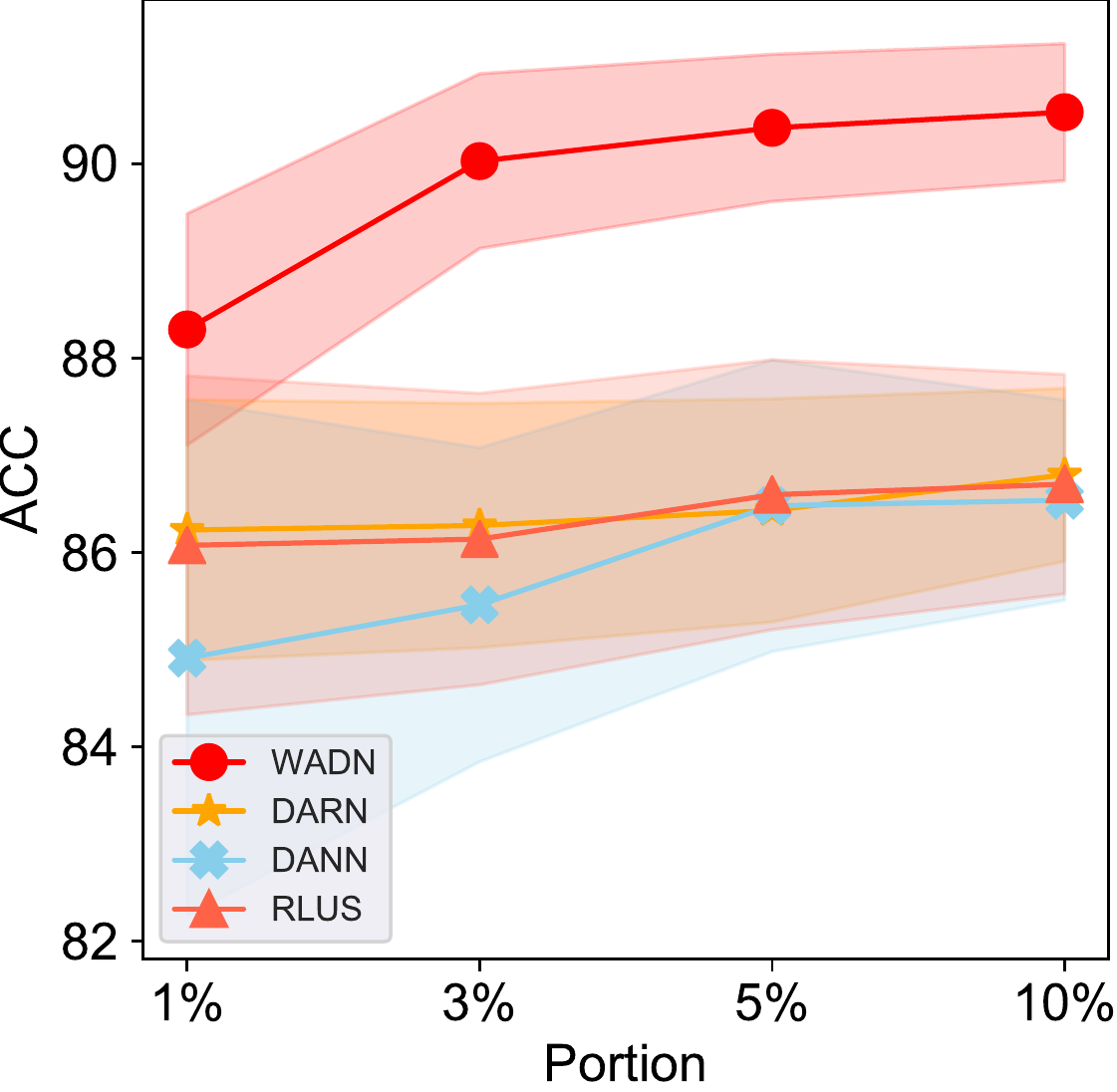}
     \caption{}
  \end{subfigure}
  \begin{subfigure}{0.32\textwidth}
  \centering
     \includegraphics[scale=0.32]{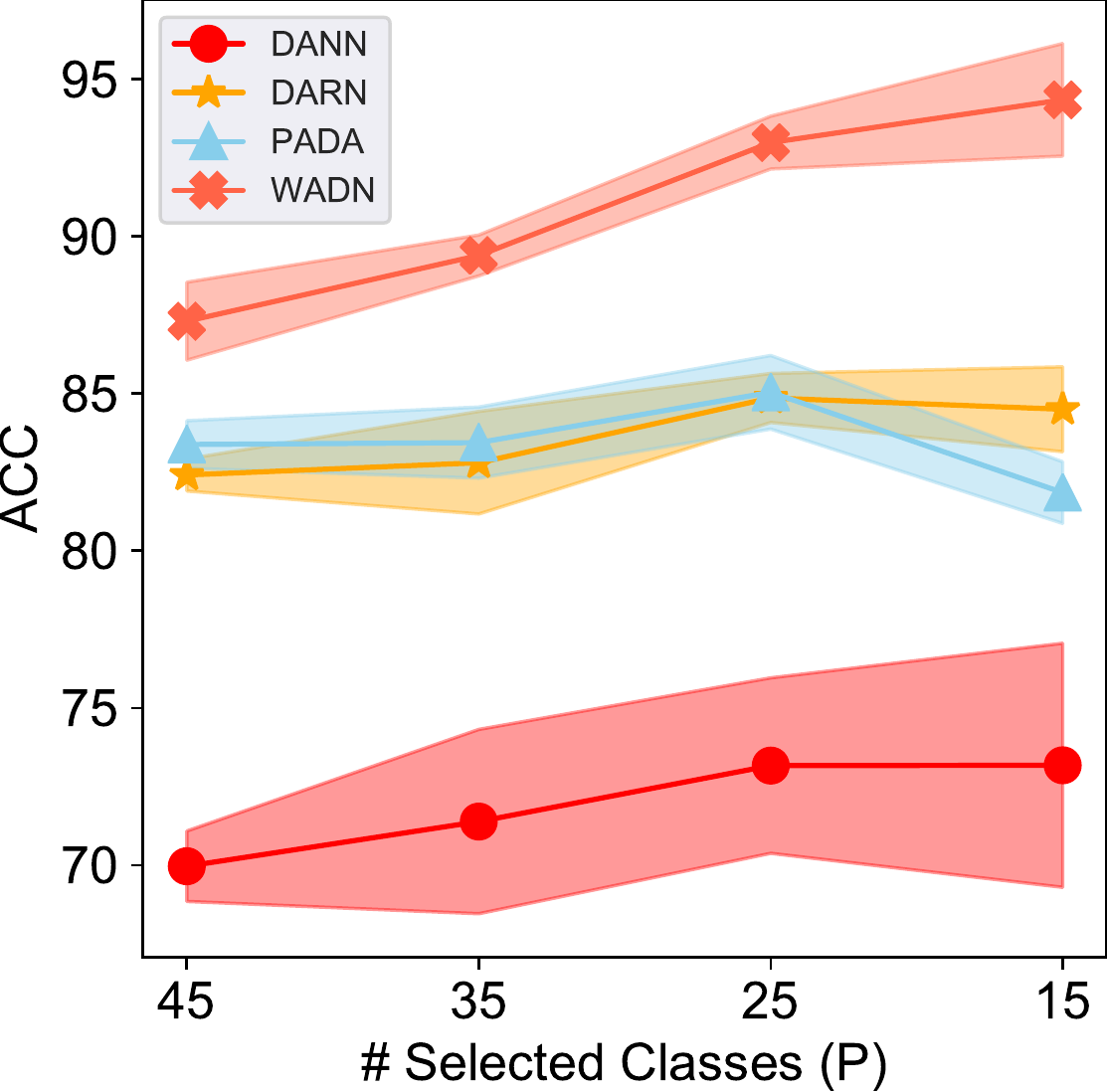}
     \caption{}
  \end{subfigure}
  \caption{Ablation study on different scenarios. (a) Unsupervised DA with Amazon Review dataset. Accuracy under different levels of label shifted sources (higher dropping rate means larger label drift). The results are reported on the average of all the domains, see the results for each domain in Appendix. (b) Multi-Source DA with limited target label in digits task with target USPS. The performance (mean $\pm$ std) of WADN is consistently better under different target samples (smaller portion indicates fewer target samples). (C) Partial Multi-source DA in office-home dataset with target domain Product. Performance (mean $\pm$ std) of different number of selected classes on the target, where WADN shows a consistent better performance under different selected sub-classes.}
  \label{fig:paper_aba_study}
\end{figure*}

The results are reported in Tab.~\ref{tab:transfer_amazon}, which also indicates strong empirical improvement.
Interestingly, on the Amazon review dataset, the previous aggregation approach RLUS is unable to select the related source when label distribution varies. 
To show the effectiveness of WADN, we test various portions of labelled samples ($1\%\sim 10\%$) on the target. The results in Fig.~\ref{fig:paper_aba_study}(b) on USPS dataset show consistently better than the baseline, even in the few target samples scenarios such as $1-3\%$.

\begin{table}[t]
\centering
\caption{Unsupervised Multi-Source Partial DA: Accuracy $(\%)$ on Office-Home (\#Source: 65, \#Target: 35)}
\label{tab:partial_office_home}
\vskip 0.1in
\resizebox{0.49\textwidth}{!}{
\begin{tabular}{c|cccc|c}
\toprule
Target & Art & Clipart &  Product & Real-World  & Average  \\ \midrule
Source &  50.56$_{\pm 1.42}$  & 49.79$_{\pm 1.14}$      & 68.10$_{\pm 1.33}$    & 78.24$_{\pm 0.76}$   & 61.67  \\ \midrule
DANN   &  53.86$_{\pm 2.23}$  & 52.71$_{\pm 2.20}$      & 71.25$_{\pm 2.44}$    & 76.92$_{\pm 1.21}$    & 63.69  \\ \midrule
MDAN   &  67.56$_{\pm 1.39}$  & 65.38$_{\pm 1.30}$      & 81.49$_{\pm 1.92}$    & 83.44$_{\pm 1.01}$   & 74.47 \\ \midrule
MDMN   &  68.13$_{\pm 1.08}$  & 65.27$_{\pm 1.93}$      & 81.33$_{\pm 1.29}$    & 84.00$_{\pm 0.64}$   & 74.68 \\ \midrule
M$^3$SDA  & 65.10$_{\pm 1.97}$   &  61.80$_{\pm 1.99}$  & 76.19$_{\pm 2.44}$    & 79.14$_{\pm 1.51}$  &  70.56 \\ \midrule
DARN   &  71.53$_{\pm 0.63}$ & 69.31$_{\pm 1.08}$       & 82.87$_{\pm 1.56}$    & 84.76$_{\pm 0.57}$   & 77.12  \\ \midrule
PADA   & 74.37$_{\pm 0.84}$      & 69.64$_{\pm 0.80}$   & 83.45$_{\pm 1.13}$    & 85.64$_{\pm 0.39}$  &  78.28 \\\midrule
WADN   & \textbf{80.06}$_{\pm 0.93}$    &\textbf{75.90}$_{\pm 1.06}$    & \textbf{89.55}$_{\pm 0.72}$  & \textbf{90.40}$_{\pm 0.39}$   & \textbf{83.98}   \\ \bottomrule
\end{tabular}}
\end{table}

\subsection{Partial Unsupervised Multi-Source DA}
In this scenario, we adopt the Office-Home dataset to  evaluate our approach, as it contains large (65) classes. We do not change the source domains and we randomly choose 35 classes from the target. We evaluate all the baselines on the same selected classes and repeat 5 times. All reported results are averaged from 3 different sub-class selections (15 runs in total), shown in Tab.~\ref{tab:partial_office_home}.  We additionally compare PADA \citep{cao2018partial} approach by merging all sources and use one-to-one partial DA algorithm. We adopt the same hyper-parameters and training strategies in unsupervised DA scenario.

The reported results are also significantly better than the current multi-source DA or one-to-one partial DA approach, which again emphasizes the benefits of WADN: properly selecting the related sources by using semantic information.

\begin{figure}[ht]
  \centering
     \includegraphics[scale=0.40]{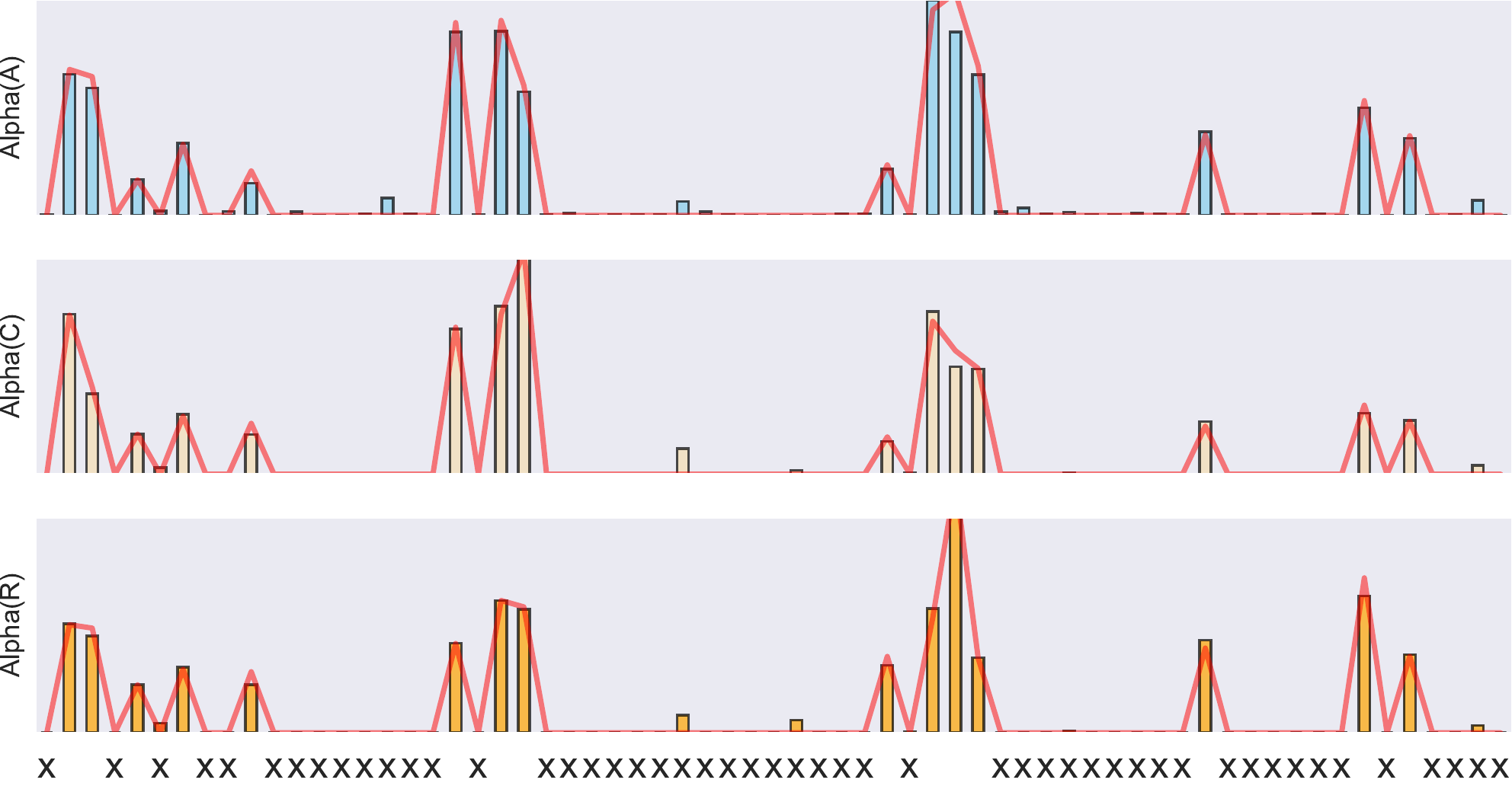}
  \caption{Analysis on Partial DA of target Product. We select 15 classes and visualize estimated $\hat{\alpha}_t$ (the bar plot). The "X" along the x-axis represents the index of \emph{dropped} 50 classes. The red curves are the true label distribution ratio. See Appendix for additional results and analysis.}
  \label{fig:office_home_analysis}
\end{figure}

Besides, we change the number of selected classes (Fig \ref{fig:paper_aba_study}(c)), the proposed WADN still indicates consistent better results by a large margin, which indicates the importance of considering $\hat{\alpha}_t$ and $\blambda$. In contrast, DANN shows unstable results on average in less selected classes. Beside, WADN shows a good estimation of the label distribution ratio (Fig \ref{fig:office_home_analysis}) and has correctly detected the non-overlapping classes, which verifies the effectiveness of the label-distribution estimator and indicates its good explainability.

\section{Conclusion}
In this paper, we proposed a novel algorithm WADN for multi-source domain adaptation problem under different label proportions. WADN differs from previous approaches in two key prospects: a better source aggregation approach when label distributions change; a unified empirical framework for three popular DA scenarios. We evaluated the proposed method by extensive experiments and showed its strong empirical results. 

\section*{Acknowledgments}
C. Shui and C. Gagné acknowledge support from NSERC-Canada and CIFAR.  B. Wang is supported by NSERC Discovery Grants Program.

\bibliography{icml}
\bibliographystyle{icml2021}

\newpage
\appendix
\onecolumn
\section{Additional Related Work}\label{sec:addtional_related_work}

\paragraph{Additional Multi-source DA Theory} has been investigated in the previous literature. In the unsupervised DA, \citep{ben2010theory,zhao2018adversarial, peng2019moment} adopted $\calH$-divergence of marginal distribution $\D(x)$ to estimate the domain relations.\citep{li2018extracting} also applied Wasserstein distance of $\D(x)$ to estimate pair-wise domain distance. \cite{mansour2009multiple,wen2019domain} used the Discrepancy distance to derive a tighter theoretical bound. The motivated practice from the aforementioned method used the feature information to learn the task relations, with the general following forms:
\begin{equation*}
    R_{\calT}(h) \leq \sum_{t}\blambda[t] R_{\calS}(h) + \sum_{t}\blambda[t] d(\calS_t(x),\calT(x)) + \beta
\end{equation*}
However, as we stated in the paper, $d(\calS_t(x),\calT(x))$ is not a proper to measure the task's relations. Besides, \cite{hoffman2018algorithms} used R\'enyi divergence that requires $\text{supp}(\calT(x))\subseteq\text{supp}(\calS(x))$, which generally does not hold in the complicated real-world applications. \cite{konstantinov2019robust,mansour2020theory} adopted $\calY$-discrepancy \citep{mohri2012new} to measure the joint distribution similarity. However, $\calY$ discrepancy is practically difficult to estimate from the data and we empirically show it is difficult to handle the target-shifted sources.

\paragraph{Multi-source DA Practice} has been proposed from various prospective. The key idea is to estimate the importance of different sources and then select the most related ones, to mitigate the influence of negative transfer. In the multi-source unsupervised DA, \citep{sankaranarayanan2018generate,balaji2019normalized,pei2018multi,zhao2019multi, zhu2019aligning, zhao2020multi, zhao2019multi, stojanov2019data, li2019target, wang2019tmda, lin2020multi} proposed different practical strategies in the classification, regression and semantic segmentation problems. In the presence of available labels on the target domain, \cite{hoffman2012discovering,tan2013multi,wei2017source,yao2010boosting,konstantinov2019robust} used generalized linear model to learn the target. \cite{christodoulidis2016multisource,li2019multi,chen-etal-2019-multi} focused on deep learning approaches and \cite{lee2019learning} proposed an ad-hoc strategy to combine to sources in the few-shot target domains. In contrast, these ideas are generally \emph{data-driven approaches} and do not propose a principled practice to understand the source combination and understand task relations.

\paragraph{Label-Partial Unsupervised DA} 
Label-Partial can be viewed as a special case of the target-shifted DA. \footnote{Since $\text{supp}(\calT(y))\subseteq \text{supp}(\calS_t(y))$ then we naturally have $\calT(y)\neq\calS_t(y)$. } Most existing works focus on one-to-one partial DA \citep{zhang2018importance,chen2020selective,bucci2019tackling,cao2019learning} by adopting the re-weighting training approach without a principled understanding. In our paper, we first analyzed this common practice and adopt the label distribution ratio as its weights, which provides a principled approach to detect the non-overlapped classes in the representation learning.  

\subsection{Other scenarios related to Multi-Source DA}
\paragraph{Domain Generalization} The domain generalization (DG) resembles multi-source transfer but aims at different goals. A common setting in DG is to learn multiple source but directly predict on the unseen target domain. The conventional DG approaches generally learn a distribution invariant features \citep{balaji2018metareg, saenko2010adapting, Motiian_2017_ICCV, ilse2019diva} or conditional distribution invariant features \citep{li2018deep,akuzawa2019adversarial}.  However, our theoretical results reveal that in the presence of label shift (i.e $\alpha_t(y)\neq 1$) and outlier tasks then learning conditional or marginal invariant features can not guarantee a small target risk. Our theoretical result enables a formal understanding about the inherent difficulty in DG problems.


\paragraph{Multi-Task Learning} The goal of multi-task learning \citep{zhang2017survey} aims to improve the prediction performance of \textbf{all} the tasks. In our paper, we aim at controlling the prediction risk of a specified target domain. We also notice some practical techniques are common such as the shared parameter \citep{zhang2012convex}, shared representation \citep{ruder2017overview}, etc.           

\section{Additional Figures}
We additionally visualize the label distributions in our experiments.
\begin{figure}[h]
  \centering
  \begin{subfigure}{0.3\textwidth}
  \centering
     \includegraphics[scale=0.3]{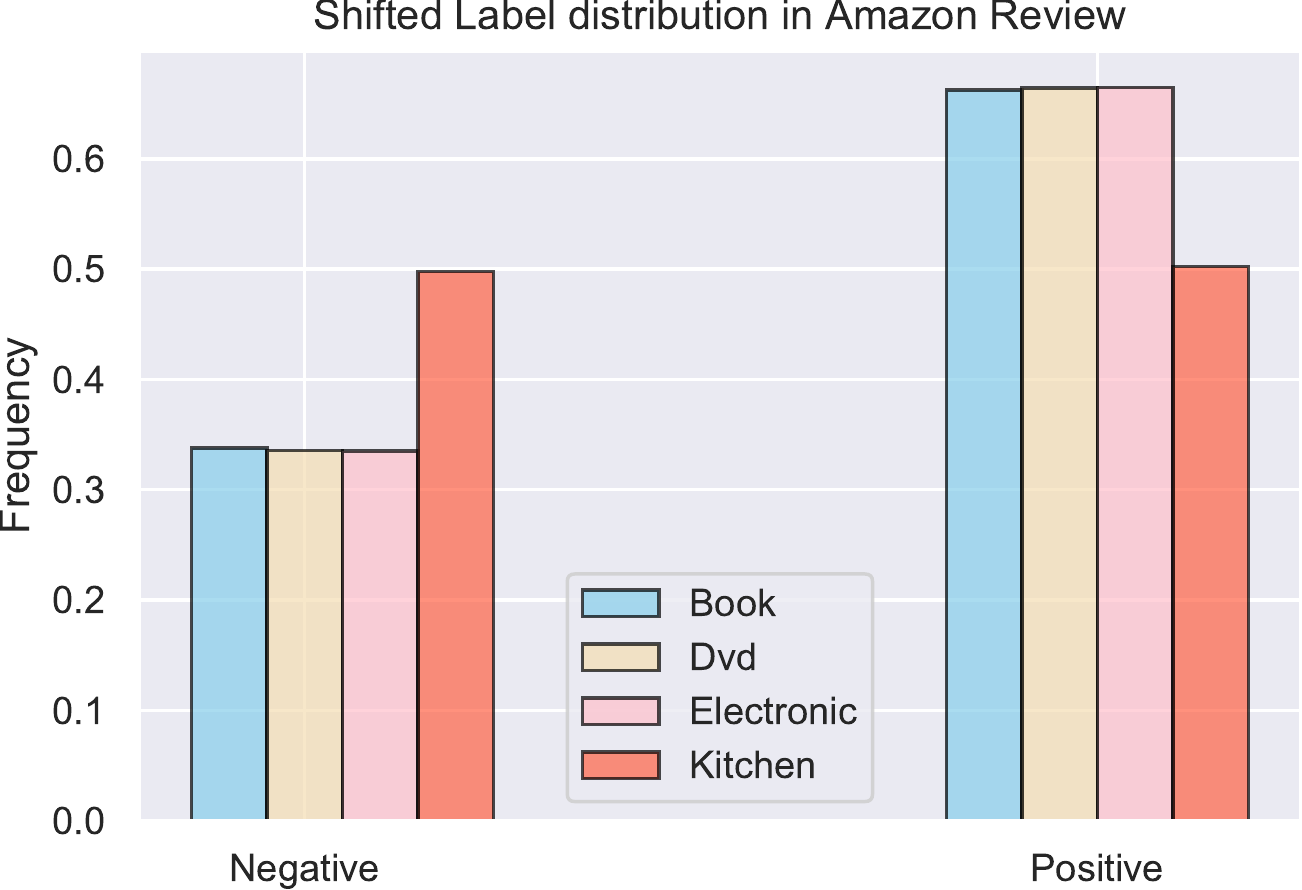}
     \caption{Amazon}
  \end{subfigure}
  \begin{subfigure}{0.3\textwidth}
  \centering
     \includegraphics[scale=0.3]{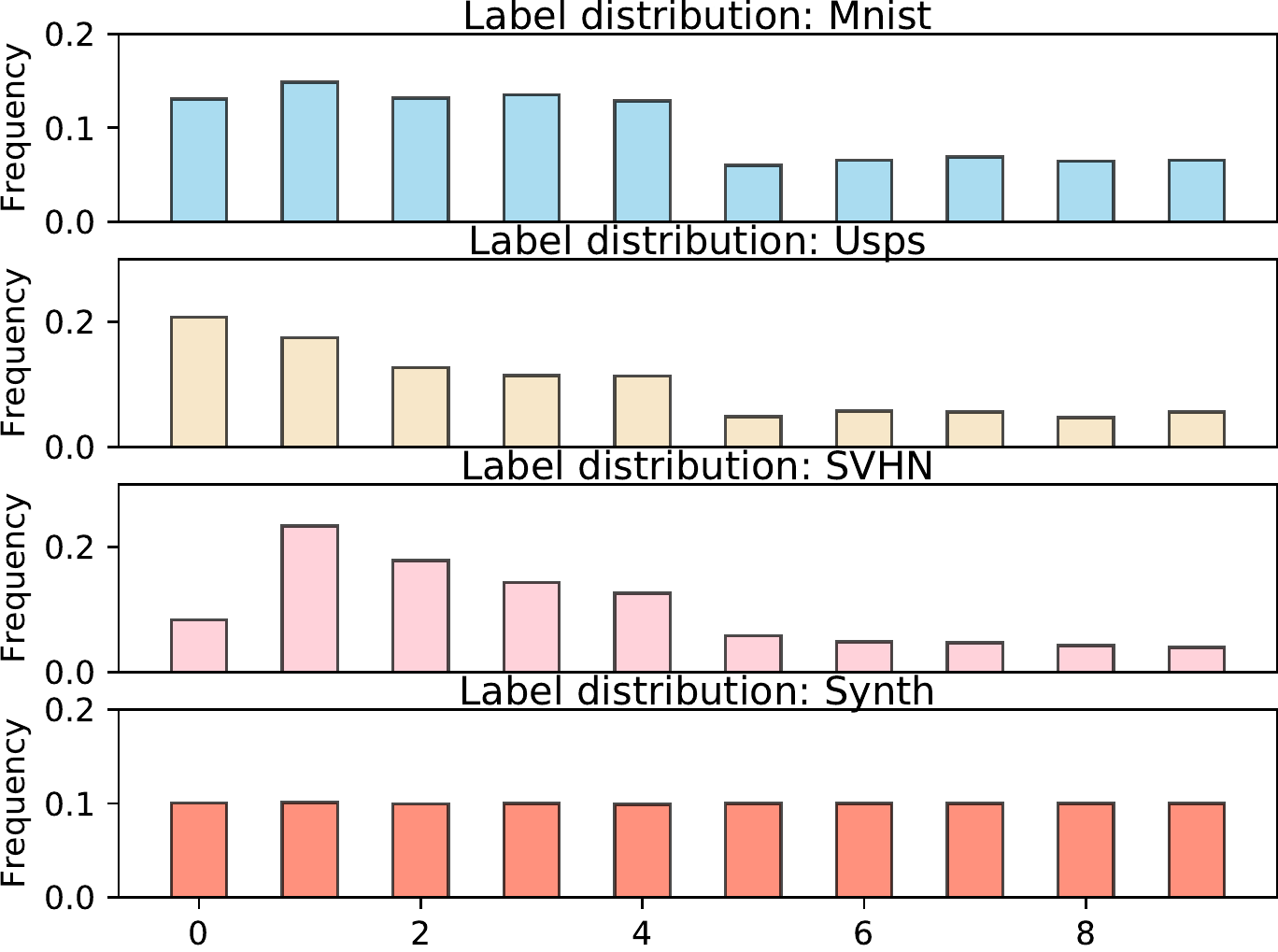}
     \caption{Digits}
  \end{subfigure}
  \begin{subfigure}{0.3\textwidth}
  \centering
     \includegraphics[scale=0.3]{figure/office-home-label.pdf}
     \caption{Office-Home}
  \end{subfigure}
  \caption{Label distribution visualization. (a) One example in Amazon Review dataset with sources: Book, Dvd, Electronic and target: Kitchen. We randomly drop $50\%$ of the negative reviews in all the sources while keeping target label distribution unchanged. (b) One example in Digits dataset with Sources: MNIST, USPS, SVHN and Target Synth. We randomly drop $50\%$ data on digits 5-9 in all sources while keeping target label distribution unchanged. (c) Office-Home dataset. The original label distribution is non-uniform. See Appendix \ref{sec: Details} for details.}
  \label{fig:drifted_label_distribution}
\end{figure}

\section{Notation Tables}
\begin{table}[htbp]
\caption{Table of Notations}
\vskip 0.1in
\begin{center}
\begin{tabular}{r c p{10cm} }
\toprule
$R_{\D}(h) = \E_{(x,y)\sim\D} \ell(h(x,y))$ &  & Expected Risk on distribution $\D$ w.r.t. hypothesis $h$ \\
$\hat{R}_{\D}(h) = \frac{1}{N} \sum_{i=1}^{N} \ell(h(x_i,y_i))$ &  & Empirical Risk on observed data $\{(x_i,y_i)\}_{i=1}^{N}$ that are i.i.d. sampled from $\D$. \\
$\alpha$ and $\hat{\alpha}_t$ &  & True and empirical label distribution ratio $\alpha(y)=\calT(y)/\calS(y)$\\        
$\hat{R}^{\alpha}_{\calS}(h) = \frac{1}{N} \sum_{i=1}^{N} \alpha(y_i)\ell(h(x_i,y_i))$ &  & Empirical Weighted Risk on observed data $\{(x_i,y_i)\}_{i=1}^{N}$. \\
$\calS(z|y)= \int_x g(z|x)S(x|Y=y) dx$ &  &  Conditional distribution w.r.t. latent variable $Z$ that induced by feature learning function $g$.\\
$W_1(\calS_t(z|y)\|\calT(z|y))$ &  & Conditional Wasserstein distance on the latent space $Z$ \\

\bottomrule
\end{tabular}
\end{center}
\label{tab:TableOfNotation}
\end{table}

\section{Proof of Theorem 1}\label{sec:appendix_proof1}

\paragraph{Proof idea} Theorem 1 consists three steps in the proof:

\begin{lemma}
If the prediction loss is assumed as $L$-Lipschitz and the hypothesis is $K$-Lipschitz w.r.t. the feature $x$ (given the same label), i.e. for $\forall Y=y$, $\|h(x_1,y)-h(x_2,y)\|_2\leq K \|x_1-x_2\|_2$. Then the target risk can be upper bounded by:
\begin{equation}
    R_{\calT}(h) \leq \sum_{t}\blambda[t] R^{\alpha_t}_{\calS}(h) + LK \sum_{t}\blambda[t] \E_{y\sim\calT(y)} W_1(\calT(x|Y=y)\|\calS(x|Y=y))
\end{equation}
\end{lemma}

\begin{proof}
The target risk can be expressed as:
\begin{equation*}
    R_{\calT}(h(x,y)) = \E_{(x,y)\sim\calT} \ell(h(x,y)) = \E_{y\sim\calT(y)} \E_{x\sim\calT(x|y)} \ell(h(x,y))
\end{equation*}
By denoting $\alpha(y) = \frac{\calT(y)}{\calS(y)}$, then we have:
\begin{equation*}
        \E_{y\sim\calT(y)} \E_{y\sim\calT(x|y)} \ell(h(x,y)) =  \E_{y\sim\calS(y)} \alpha(y) \E_{x\sim\calT(x|y)} \ell(h(x,y))
\end{equation*}
Then we aim to upper bound $\E_{x\sim\calT(x|y)} \ell(h(x,y))$. For any fixed $y$,
\begin{equation*}
     \E_{x\sim\calT(x|y)} \ell(h(x,y))-\E_{x\sim\calS(x|y)} \ell(h(x,y)) \leq |\int_{x\in\calX} \ell(h(x,y)) d(\calT(x|y)-\calS(x|y))|
\end{equation*}
Then according to the Kantorovich-Rubinstein duality, for \textbf{any} distribution coupling $\gamma\in\Pi(\calT(x|y),\calS(x|y))$, then we have:
\begin{equation*}
    \begin{split}
       & =\inf_{\gamma} |\int_{\calX\times\calX} \ell(h(x_p,y)) - \ell(h(x_q,y))d\gamma(x_p,x_q)| \\
       & \leq \inf_{\gamma} \int_{\calX\times\calX} |\ell(h(x_p,y)) - \ell(h(x_q,y))| d\gamma(x_p,x_q) \\
       & \leq L \inf_{\gamma} \int_{\calX\times\calX} |h(x_p,y)) - h(x_q,y)| d\gamma(x_p,x_q) \\
       & \leq LK \inf_{\gamma} \int_{\calX\times\calX} \|x_p-x_q\|_2 d\gamma(x_p,x_q) \\
       & = LK W_1(\calT(x|Y=y)\|\calS(x|Y=y))
    \end{split}
\end{equation*}
The first inequality is obvious; and the second inequality comes from the assumption that $\ell$ is $L$-Lipschitz; the third inequality comes from the hypothesis is $K$-Lipschitz w.r.t. the feature $x$ (given the same label), i.e. for $\forall Y=y$, $\|h(x_1,y)-h(x_2,y)\|_2\leq K \|x_1-x_2\|_2$.

Then we have:
\begin{equation*}
\begin{split}
       R_{\calT}(h) & \leq \E_{y\sim\calS(y)} \alpha(y) [\E_{x\sim\calS(x|y)} \ell(h(x,y)) + LK W_1(\calT(x|y)\|\calS(x|y))]\\
       & = \E_{(x,y)\sim\calS} \alpha(y) \ell(h(x,y)) + LK \E_{y\sim\calT(y)} W_1(\calT(x|Y=y)\|\calS(x|Y=y)) \\
       & = R^{\alpha}_{\calS}(h) + LK \E_{y\sim\calT(y)} W_1(\calT(x|Y=y)\|\calS(x|Y=y))
\end{split}
\end{equation*}

Supposing each source $\calS_{t}$ we assign the weight $\blambda[t]$ and label distribution ratio $\alpha_t(y) = \frac{\calT(y)}{\calS_t(y)}$, then by combining this $T$ source target pair, we have:
\begin{equation*}
    R_{\calT}(h) \leq \sum_{t}\blambda[t] R^{\alpha_t}_{\calS_t}(h) + LK\sum_{t}\blambda[t] \E_{y\sim\calT(y)} W_1(\calT(x|Y=y)\|\calS_t(x|Y=y))
\end{equation*}
\end{proof}

Then we will prove Theorem 1 from this result, we will derive the non-asymptotic bound, estimated from the finite sample observations. Supposing the empirical label ratio value is $\hat{\alpha}_t$, then for any simplex $\blambda$ we can prove the high-probability bound.

\subsection{Bounding the empirical and expected prediction risk}
\begin{proof}
We first bound the first term, which can be upper bounded as:
\begin{small}
\begin{equation*}
    \sup_{h} |\sum_{t}\blambda[t] R^{\alpha_t}_{\calS_t}(h) - \sum_{t}\blambda[t] \hat{R}^{\hat{\alpha}_t}_{\calS_t}(h)|\leq \underbrace{\sup_{h}|\sum_{t}\blambda[t] R^{\alpha_t}_{\calS_t}(h) - \sum_{t}\blambda[t] \hat{R}^{\alpha_t}_{\calS_t}(h)|}_{(\RN{1})} + \underbrace{\sup_{h} |\sum_{t}\blambda[t] \hat{R}^{\alpha_t}_{\calS_t}(h) - \sum_{t}\blambda[t] \hat{R}^{\hat{\alpha}_t}_{\calS_t}(h)|}_{(\RN{2})}
\end{equation*}
\end{small}

\paragraph{Bounding term $(\RN{1})$} According to the McDiarmid inequality, each item changes at most $|\frac{2\blambda[t]\alpha_t(y)\ell}{N_{\calS_t}}|$. Then we have:
\begin{equation*}
    P\left((\RN{1}) -  \E (\RN{1})\geq t\right)\leq \exp(\frac{-2t^2}{\sum_{t=1}^T \frac{4}{\beta_t N} \blambda^2[t]\alpha_t(y)^2\ell^2})=\delta 
\end{equation*}
By substituting $\delta$, at high probability $1-\delta$ we have:
\begin{equation*}
    (\RN{1}) \leq  \E (\RN{1}) + L_{\max}d^{\sup}_{\infty} \sqrt{\sum_{t=1}^{T} \frac{\blambda[t]^2}{\beta_t}} 
    \sqrt{\frac{\log(1/\delta)}{2N}}
\end{equation*}
Where $L_{\max} = \sup_{h\in\calH} \ell(h)$ and $N=\sum_{t=1}^T N_{\calS_t}$ the total source observations and $\beta_t = \frac{N_{\calS_t}}{N}$ the frequency ratio of each source. And $d_{\infty}^{\sup} = \max_{t=1,\dots,T} d_{\infty}(\calT(y)\|\calS(y)) = \max_{t=1,\dots,T} \max_{y\in[1,\calY]}\alpha_{t}(y)$, the maximum true label shift value (constant).

Bounding $\E \sup (\RN{1})$, the expectation term can be upper bounded as the form of Rademacher Complexity:
\begin{equation*}
\begin{split}
     \E (\RN{1}) & \leq  2 \E_{\sigma} \E_{\hat{\calS}_1^T} \sup_{h}  \sum_{t=1}^T \blambda[t] \sum_{(x_t,y_t)\in\hat{\calS_t}} \frac{1}{TN}\left( \alpha_t(y)\ell(h(x_t,y_t)\right) \\
     & \leq 2 \sum_{t}\blambda[t]  \E_{\sigma}\E_{\hat{\calS}_1^T} \sup_{h}  \sum_{(x_t,y_t)\in\hat{\calS_t}} \frac{1}{TN}\left( \alpha_t(y)\ell(h(x_t,y_t)\right) \\
     & \leq 2 \sup_{t} \E_{\sigma}\E_{\hat{\calS}_t} \sup_{h}  \sum_{(x_t,y_t)\in\hat{\calS_t}} \frac{1}{TN}\left[ \alpha_t(y)\ell(h(x_t,y_t))\right]\\
     & = \sup_{t} 2\calR_t(\ell,\calH) = 2\bar{R}(\ell,\calH)
\end{split}
\end{equation*}
Where $\bar{R}(\ell,\calH) = \sup_{t} \calR_t(\ell,\calH) = \sup_{t}\sup_{h\sim\calH} \E_{\hat{\calS}_t,\sigma}  \sum_{(x_t,y_t)\in\hat{\calS_t}} \frac{1}{TN}\left[ \alpha_t(y)\ell(h(x_t,y_t))\right]$, represents the Rademacher complexity w.r.t. the prediction loss $\ell$, hypothesis $h$ and \emph{true} label distribution ratio $\alpha_t$.

Therefore with high probability $1-\delta$, we have:
\begin{equation*}
    \sup_{h}|\sum_{t}\blambda[t] R^{\alpha_t}_{\calS}(h) - \sum_{t}\blambda[t] \hat{R}^{\alpha_t}_{\calS}(h)|\leq \bar{\calR}(\ell,h) + L_{\max}d^{\sup}_{\infty} \sqrt{\sum_{t=1}^{T} \frac{\blambda[t]^2}{\beta_t}} 
    \sqrt{\frac{\log(1/\delta)}{2N}}
\end{equation*}

\paragraph{Bounding Term $(\RN{2})$} For all the hypothesis $h$, we have:
\begin{equation*}
\begin{split}
    |\sum_{t}\blambda[t] \hat{R}^{\alpha_t}_{\calS_t}(h) - \sum_{t}\blambda[t] \hat{R}^{\hat{\alpha}_t}_{\calS_t}(h)|& =|\sum_{t}\blambda[t] \frac{1}{N_{\calS_t}} \sum_{i}^{N_{\calS_t}} (\alpha(y(i))-\hat{\alpha}(y(i)))\ell(h)| \\
    & =\sum_{t}\blambda[t] \frac{1}{N_{\calS_t}}|\sum_{y}^{|\calY|} (\alpha(Y=y)-\hat{\alpha}(Y=y))\bar{\ell}(Y=y)|
\end{split}
\end{equation*}
Where $\bar{\ell}(Y=y) = \sum_{i}^{N_{\calS_t}}\ell(h(x_i,y_i=y))$, represents the cumulative error, conditioned on a given label $Y=y$. According to the Holder inequality, we have:
\begin{equation*}
\begin{split}
     \sum_{t}\blambda[t] \frac{1}{N_{\calS_t}}|\sum_{y}^{|\calY|} (\alpha_t(Y=y)-\hat{\alpha}_t(Y=y))\bar{\ell}(Y=y)| & \leq \sum_{t}\blambda[t] \frac{1}{N_{\calS_t}} \|\alpha_t-\hat{\alpha}_t\|_2 \|\bar{\ell}(Y=y)\|_2 \\
     & \leq L_{\max} \sum_{t}\blambda[t] \|\alpha_t-\hat{\alpha}_t\|_2 \\
     & \leq L_{\max} \sup_{t} \|\alpha_t-\hat{\alpha}_t\|_2
\end{split}
\end{equation*}

Therefore, $\forall h\in\calH$, with high probability $1-\delta$ we have:
\begin{equation*}
    \sum_{t}\blambda[t] R^{\alpha_t}_{\calS}(h) \leq  \sum_{t}\blambda[t] \hat{R}^{\hat{\alpha}_t}_{\calS}(h) +
    2\bar{\calR}(\ell,h) + L_{\max}d^{\sup}_{\infty} \sqrt{\sum_{t=1}^{T} \frac{\blambda[t]^2}{\beta_t}} 
    \sqrt{\frac{\log(1/\delta)}{2N}} + L_{\max} \sup_{t} \|\alpha_t-\hat{\alpha}_t\|_2
\end{equation*}

\subsection{Bounding empirical Wasserstein Distance}
Then we need to derive the sample complexity of the empirical and true distributions, which can be decomposed as the following two parts. For any $t$, we have:
\begin{equation*}
\begin{split}
     & \E_{y\sim\calT(y)} W_1(\calT(x|Y=y)\|\calS_t(x|Y=y)) - \E_{y\sim\hat{\calT}(y)} W_1(\hat{\calT}(x|Y=y)\|\hat{\calS_t}(x|Y=y)) \\
     & \underbrace{\leq \E_{y\sim\calT(y)} W_1(\calT(x|Y=y)\|\calS_t(x|Y=y)) - \E_{y\sim\calT(y)} W_1(\hat{\calT}(x|Y=y)\|\hat{\calS_t}(x|Y=y))}_{(\RN{1})} \\
     & + \underbrace{\E_{y\sim\calT(y)} W_1(\hat{\calT}(x|Y=y)\|\hat{\calS_t}(x|Y=y)) - \E_{y\sim\hat{\calT}(y)} W_1(\hat{\calT}(x|Y=y)\|\hat{\calS_t}(x|Y=y))}_{(\RN{2})}
\end{split}
\end{equation*}
\paragraph{Bounding $(\RN{1})$} 
We have:
\begin{equation*}
\begin{split}  
& \E_{y\sim\calT(y)} W_1(\calT(x|Y=y)\|\calS_t(x|Y=y)) - \E_{y\sim\calT(y)} W_1(\hat{\calT}(x|Y=y)\|\hat{\calS_t}(x|Y=y)) \\
& = \sum_{y} \calT(y) \left(W_1(\calT(x|Y=y)\|\calS_t(x|Y=y)) - W_1(\hat{\calT}(x|Y=y)\|\hat{\calS_t}(x|Y=y)\right) \\
& \leq |\sum_{y}\calT(y)| \sup_{y} \left(W_1(\calT(x|Y=y)\|\calS_t(x|Y=y)) - W_1(\hat{\calT}(x|Y=y)\|\hat{\calS_t}(x|Y=y)\right) \\
& = \sup_{y} \left(W_1(\calT(x|Y=y)\|\calS_t(x|Y=y)) - W_1(\hat{\calT}(x|Y=y)\|\hat{\calS_t}(x|Y=y)\right)\\
& \leq \sup_{y}~[W_1(\calS_t(x|Y=y)\|\hat{\calS_t}(x|Y=y)) 
 + W_1(\hat{\calS_t}(x|Y=y)\|\hat{\calT}(x|Y=y)) \\
& +  W_1(\hat{\calT}(x|Y=y)\|\calT(x|Y=y))- W_1(\hat{\calT}(x|Y=y)\|\hat{\calS_t}(x|Y=y))]\\
& = \sup_{y} W_1(\calS_t(x|Y=y)\|\hat{\calS_t}(x|Y=y)) + W_1(\hat{\calT}(x|Y=y)\|\calT(x|Y=y))
\end{split}
\end{equation*}
The first inequality holds because of the Holder inequality. As for the second inequality, we use the triangle inequality of Wasserstein distance. $W_1(P\|Q)\leq W_1(P\|P_1) + W_1(P_1\|P_2) + W_1(P_2\|Q)$.

According to the convergence behavior of Wasserstein distance \citep{weed2019sharp}, with high probability $\geq 1-2\delta$ we have:
\begin{equation*}
      W_1(\calS_t(x|Y=y)\|\hat{\calS_t}(x|Y=y)) + W_1(\hat{\calT}(x|Y=y)\|\calT(x|Y=y)) 
     \leq \kappa(\delta, N^{y}_{\calS_t}, N^{y}_{\calT})
\end{equation*}
Where $k(\delta,N^{y}_{\calS_t},N^{y}_{\calT}) = C_{t,y}(N^{y}_{\calS_t})^{-s_{t,y}} + C_{y}(N^{y}_{\calT})^{-s_{y}} +  \sqrt{\frac{1}{2}\log(\frac{2}{\delta})}(\sqrt{\frac{1}{N^{y}_{\calS_t}}} + \sqrt{\frac{1}{N_t^y}})$, where $N^{y}_{\calS_t}$ is the number of $Y=y$ in source $t$ and $N^{y}_{\calT}$ is the number of $Y=y$ in target distribution. $C_{t,y}$, $C_y$ $s_{t,y}>2$, $s_y>2$ are positive constant in the concentration inequality.  This indicates the convergence behavior between empirical and true Wasserstein distance. 

If we adopt the union bound (over all the labels) by setting $\delta \gets \delta/|\calY|$, then with high probability $\geq 1-2\delta$, we have:
\begin{equation*}
    \sup_{y} W_1(\calS(x|Y=y)\|\hat{\calS}(x|Y=y)) + W_1(\hat{\calT}(x|Y=y)\|\calT(x|Y=y)) 
     \leq \kappa(\delta, N^{y}_{\calS_t}, N^{y}_{\calT})
\end{equation*}
where $\kappa(\delta,N^{y}_{\calS_t},N^{y}_{\calT}) = C_{t,y}(N^{y}_{\calS_t})^{-s_{t,y}} + C_{y}(N^{y}_{\calT})^{-s_{y}} +  \sqrt{\frac{1}{2}\log(\frac{2|\calY|}{\delta})}(\sqrt{\frac{1}{N^{y}_{\calS_t}}} + \sqrt{\frac{1}{N^{y}_{\calT}}})$

Again by adopting the union bound (over all the tasks) by setting $\delta \gets \delta/T$, with high probability $\geq 1-2\delta$, we have:
\begin{small}
\begin{equation*}
\sum_{t} \blambda[t] \E_{y\sim\calT(y)} W_1(\calT(x|Y=y)\|\calS(x|Y=y)) - \sum_{t}\blambda[t] \E_{y\sim\calT(y)} W_1(\hat{\calT}(x|Y=y)\|\hat{\calS}(x|Y=y)) \leq \sup_{t} \kappa(\delta, N^{y}_{\calS_t}, N^{y}_{\calT})
\end{equation*}
\end{small}
Where $\kappa(\delta, N^{y}_{\calS_t}, N^{y}_{\calT}) = C_{t,y}(N^{y}_{\calS_t})^{-s_{t,y}} + C_{y}(N^{y}_{\calT})^{-s_{y}} + \sqrt{\frac{1}{2}\log(\frac{2T|\calY|}{\delta})}(\sqrt{\frac{1}{N^{y}_{\calS_t}}} + \sqrt{\frac{1}{N^{y}_{\calT}}})$.

\paragraph{Bounding $(\RN{2})$} We can bound the second term:
\begin{equation*}
\begin{split}
    &  \E_{y\sim\calT(y)} W_1(\hat{\calT}(x|Y=y)\|\hat{\calS_t}(x|Y=y)) - \E_{y\sim\hat{\calT}(y)} W_1(\hat{\calT}(x|Y=y)\|\hat{\calS_t}(x|Y=y)) \\
    & \leq \sup_{y} W_1(\hat{\calT}(x|Y=y)\|\hat{\calS_t}(x|Y=y)) |\sum_{y}\calT(y)-\hat{\calT}(y)| \\
    & \leq C^t_\text{max} |\sum_{y}\calT(y)-\hat{\calT}(y)| 
\end{split}
\end{equation*}
Where $C^t_\text{max} = \sup_{y} W_1(\hat{\calT}(x|Y=y)\|\hat{\calS}(x|Y=y))$ is a positive and bounded constant. 
Then we need to bound  $|\sum_{y}\calT(y)-\hat{\calT}(y)|$, by adopting MicDiarmid's inequality, we have at high probability $1-\delta$:
\begin{align*}
    |\sum_{y}\calT(y)-\hat{\calT}(y)| & \leq \E_{\hat{\calT}}|\sum_{y}\calT(y)-\hat{\calT}(y)| + \sqrt{\frac{\log(1/\delta)}{2N_{\calT}}} \\
    & = 2 \E_{\sigma} \E_{\hat{\calT}} \sum_{y}\sigma\hat{\calT}(y) + \sqrt{\frac{\log(1/\delta)}{2N_{\calT}}}
\end{align*}
Then we bound $\E_{\sigma} \E_{\hat{\calT}} \sum_{y}\sigma\hat{\calT}(y)$. We use the properties of Rademacher complexity [Lemma 26.11, \citep{shalev2014understanding}] and notice that $\hat{\calT}(y)$ is a probability simplex, then we have:
\begin{align*}
    \E_{\sigma} \E_{\hat{\calT}} \sum_{y}\sigma\hat{\calT}(y) \leq \sqrt{\frac{2\log(2|\calY|)}{N_{\calT}}}
\end{align*}
Then we have $|\sum_{y}\calT(y)-\hat{\calT}(y)| \leq \sqrt{\frac{2\log(2|\calY|)}{N_{\calT}}} + \sqrt{\frac{\log(1/\delta)}{2N_{\calT}}} $

Then using the union bound and denoting $\delta\gets \delta/T$, with high probability $\geq 1-\delta$ and for any simplex $\blambda$, we have:
\begin{equation*}
\begin{split}
\sum_{t}\blambda[t] \E_{y\sim\calT(y)} W_1(\hat{\calT}(x|Y=y)\|\hat{\calS_t}(x|Y=y)) & \leq \sum_{t}\blambda[t] \E_{y\sim\hat{\calT}(y)} W_1(\hat{\calT}(x|Y=y)\|\hat{\calS_t}(x|Y=y)) \\
    &  C_{\max}(\sqrt{\frac{2\log(2|\calY|)}{N_{\calT}}} + \sqrt{\frac{\log(T/\delta)}{2N_{\calT}}})
\end{split}
\end{equation*}
where $C_{\max} = \sup_{t} C^t_{\max}$.

Combining together, we can derive the PAC-Learning bound, which is estimated from the finite samples (with high probability $1-4\delta$):
\begin{equation*}
\begin{split}
     R_{\calT}(h) & \leq \sum_{t}\blambda_t \hat{R}^{\hat{\alpha}_t}_{\calS_t}(h) + 
    LH \sum_{t}\blambda_t \E_{y\sim\hat{\calT}(y)} W_1(\hat{\calT}(x|Y=y)\|\hat{\calS}(x|Y=y))
    + L_{\max}d^{\sup}_{\infty} \sqrt{\sum_{t=1}^{T} \frac{\blambda_t^2}{\beta_t}}\sqrt{\frac{\log(1/\delta)}{2N}} \\
    & + 2\bar{\calR}(\ell,h) 
     + L_{\max} \sup_{t} \|\alpha_t-\hat{\alpha}_t\|_2 + \sup_{t} \kappa(\delta, N^{y}_{\calS_t}, N^{y}_{\calT}) + C_{\max}(\sqrt{\frac{2\log(2|\calY|)}{N_{\calT}}} + \sqrt{\frac{\log(T/\delta)}{2N_{\calT}}}) 
\end{split}
\end{equation*}
Then we denote $\text{Comp}(N_{\calS_1},\dots,N_{\calT}, \delta) = 2\bar{\calR}(\ell,h) + \sup_{t} \kappa(\delta, N^{y}_{\calS_t}, N^{y}_{\calT}) + C_{\max}(\sqrt{\frac{2\log(2|\calY|)}{N_{\calT}}} + \sqrt{\frac{\log(T/\delta)}{2N_{\calT}}})$ as the convergence rate function that decreases with larger $N_{\calS_1},\dots, N_{\calT}$. 
Bedsides, $\bar{\calR}(\ell,h)=\sup_{t} \calR_t(\ell,\calH)$ is the re-weighted Rademacher complexity. Given a fixed hypothesis with finite VC dimension \footnote{If the hypothesis is the neural network, the Rademacher complexity can still be bounded analogously through recent theoretical results in deep neural-network}, it can be proved  $\bar{\calR}(\ell,h) = \min_{N_{\calS_1},\dots,N_{\calS_T}} \mathcal{O}(\sqrt{\frac{1}{N_{\calS_t}}})$ i.e \citep{shalev2014understanding}. 
\end{proof}

\section{Proof of Theorem 2}

We first recall the stochastic feature representation $g$ such that $g:\calX \to \calZ$ and \emph{scoring hypothesis} h $h:\calZ\times\calY \to \R$ and the prediction loss $\ell$ with $\ell: \R\to\R$. \footnote{Note this definition is different from the conventional binary classification with binary output, and it is more suitable in the multi-classification scenario and cross entropy loss \citep{hoffman2018algorithms}. For example, if we define $l=-\log(\cdot)$ and $h(z,y)\in (0,1)$ as a scalar score output. Then $\ell(h(z,y))$ can be viewed as the cross-entropy loss for the neural-network.}

\begin{proof}
The marginal distribution and conditional distribution w.r.t. latent variable $Z$ that are induced by $g$, which can be reformulated as:
\begin{equation*}
    \calS(z) = \int_{x} g(z|x)\calS(x) dx \quad\quad  \calS(z|y) = \int_{x} g(z|x)  \calS(x|Y=y) dx
\end{equation*}

In the multi-class classification problem, we additionally define the following distributions:
\begin{align*}
     & \mu^{k}(z) = \calS(Y=k,z)  = \calS(Y=k)  \calS(z|Y=k) \\
     & \pi^{k}(z) = \calT(Y=k,z)  = \calT(Y=k)  \calT(z|Y=k)
\end{align*}
Based on \citep{nguyen2009surrogate} and $g(z|x)$ is a stochastic representation learning function, the loss conditioned a fixed point $(x,y)$ w.r.t. $h$ and $g$ is $\E_{z\sim g(z|x)} \ell(h(z,y))$. Then taking the expectation over the $\calS(x,y)$ we have:
\footnote{An alternative understanding is based on the Markov chain. In this case it is a DAG with 
$Y \xleftarrow{\calS(y|x)} X \xrightarrow{g} Z$, $X\xrightarrow{\calS(y|x)} Y \xrightarrow{h} S \xleftarrow{h} Z \xleftarrow{g} X$. (S is the output of the scoring function). Then the expected loss over the all random variable can be equivalently written as $\int \Proba(x,y,z,s)~\ell(s) ~d(x,y,z,s) = \int \Proba(x) \Proba(y|x) \Proba(z|x) \Proba(s|z,y) \ell(s) = \int \Proba(x,y) \Proba(z|x) \Proba(s|z,y) \ell(s) d(x,y)d(z)d(s)$. Since the scoring $S$ is determined by $h(x,y)$, then $\Proba(s|y,z)=1$. According to the definition we have $\Proba(z|x)=g(z|x)$, $\Proba(x,y)=\calS(x,y)$, then the loss can be finally expressed as $\E_{\calS(x,y)}\E_{g(z|x)}\ell(h(z,y))$} 
\begin{equation*}
\begin{split}
    R_{\calS}(h,g) & =  \E_{(x,y)\sim\calS(x,y)}\E_{z\sim g(z|x)} \ell(h(z,y)) \\
    & = \sum_{k=1}^{|\calY|}\calS(y=k) \int_{x}\calS(x|Y=k)\int_{z} g(z|x) \ell(h(z,y=k)) dzdx \\
    & = \sum_{k=1}^{|\calY|} \calS(y=k) \int_{z} [\int_{x}\calS(x|Y=k) g(z|x)dx] \ell(h(z,y=k))dz \\
    & = \sum_{k=1}^{|\calY|} \calS(y=k) \int_{z} \calS(z|Y=k)\ell(h(z,y=k)) dz \\
    & = \sum_{k=1}^{|\calY|} \int_{z} \calS(z,Y=k)\ell(h(z,y=k)) dz\\
    & = \sum_{k=1}^{|\calY|} \int_{z} \mu^{k}(z) \ell(h(z,y=k)) dz
\end{split}
\end{equation*}
Intuitively, the expected loss w.r.t. the joint distribution $\calS$ can be decomposed as the expected loss on the label distribution $\calS(y)$ (weighted by the labels) and conditional distribution $\calS(\cdot|y)$ (real valued conditional loss). 

Then the expected risk on the $\calS$ and $\calT$ can be expressed as:
\begin{align*}
     & R_{\calS}(h,g) = \sum_{k=1}^{|\calY|} \int_{z} \ell(h(z,y=k)) \mu^{k}(z) dz  \\
    &  R_{\calT}(h,g) = \sum_{k=1}^{|\calY|} \int_{z} \ell(h(z,y=k)) \pi^{k}(z) dz 
\end{align*}

By denoting $\alpha(y) = \frac{\calT(y)}{\calS(y)}$, we have the $\alpha$-weighted loss:
\begin{equation*}
\begin{split}
     R^{\alpha}_{\calS}(h, g) =  & \calT(Y=1) \int_{z} \ell(h(z,y=1)) \calS(z|Y=1) + \calT(Y=2)\int_{z}\ell(h(z,y=2))\calS(z|Y=2) \\
     & + \dots + \calT(Y=k)\int_{z}\ell(h(z,y=k))\calS(z|Y=k) dz
\end{split}
\end{equation*}
Then we have:
\begin{equation*}
\begin{split}
    R_{\calT}(h, g) - R^{\alpha}_{\calS}(h,g) & \leq \sum_{k}\calT(Y=k) \int_{z}\ell(h(z,y=k)) d|\calS(z|Y=k) - \calT(z|Y=k)| \\
\end{split}
\end{equation*}
Under the same assumption, we have the loss function $\ell(h(z,Y=k))$ is KL-Lipschitz w.r.t. the cost $\|\cdot\|_2$ (given a fixed $k$). Therefore by adopting the same proof strategy (Kantorovich-Rubinstein duality) in Lemma 2, we have
\begin{equation*}
    \begin{split}
    & \leq KL \calT(Y=1)W_1(\calS(z|Y=1)\|\calT(z|Y=1)) 
     + \dots + KL \calT(Y=k) W_1(\calS(z|Y=k)\|\calT(z|Y=k)) \\
        & = KL \E_{y\sim\calT(y)} W_1(\calS(z|Y=y)\|\calT(z|Y=y))
    \end{split}
\end{equation*}
Therefore, we have:
\begin{equation*}
    R_{\calT}(h, g) \leq  R^{\alpha}_{\calS}(h, g) +  LK \E_{y\sim\calT(y)}  W_1(\calS(z|Y=y)\|\calT(z|Y=y))
\end{equation*}

Based on the aforementioned result, we have $\forall t=1,\dots,T$ and denote $\calS = \calS_t$ and $\alpha(y) = \alpha_t(y) = \calT(y)/\calS_t(y)$:
\begin{equation*}
    \blambda[t]  R_{\calT}(h, g) \leq  \blambda[t] R^{\alpha_t}_{\calS_t}(h, g) +  LK  \blambda[t] \E_{y\sim\calT(y)}  W_1(\calS_t(z|Y=y)\|\calT(z|Y=y))
\end{equation*}
Summing over $t=1,\dots,T$, we have:
\begin{equation*}
     R_{\calT}(h, g) \leq  \sum_{t=1}^T \blambda[t] R^{\alpha_t}_{\calS_t}(h, g) +  LK \sum_{t=1}^T \blambda[t] \E_{y\sim\calT(y)}  W_1(\calS_t(z|Y=y)\|\calT(z|Y=y))
\end{equation*}
\end{proof}

\section{Approximation $W_1$ distance}\label{sec:approx_w1}
According to Jensen inequality, we have 
\begin{equation*}
   W_1(\hat{\calS}_t(z|Y=y)\| \hat{\calT}(z|Y=y)) \leq \sqrt{[W_2(\hat{\calS}_t(z|Y=y)\| \hat{\calT}(z|Y=y))]^2} 
\end{equation*}
Supposing $\hat{\calS}_t(z|Y=y) \approx \calN(\mathbf{C}^y_t,\mathbf{\Sigma})$ and $\hat{\calT}(z|Y=y) \approx \calN(\mathbf{C}^y,\mathbf{\Sigma})$, then we have:
\begin{equation*}
    [W_2(\hat{\calS}_t(z|Y=y)\| \hat{\calT}(z|Y=y)]^2  = \|\mathbf{C}^y_t-\mathbf{C}^y\|_2^2 + \text{Trace}(2\mathbf{\Sigma} - 2 (\mathbf{\Sigma}\mathbf{\Sigma})^{1/2}) = \|\mathbf{C}^y_t-\mathbf{C}^y\|_2^2
\end{equation*}
We would like to point out that assuming the identical covariance matrix is more computationally efficient during the matching. This is advantageous and reasonable in the deep learning regime: we adopted the mini-batch (ranging from 20-128) for the neural network parameter optimization, in each mini-batch the samples of each class are \textbf{small}, then we compute the empirical covariance/variance matrix will be surely \textbf{biased} to the ground truth variance  
and induce a much higher complexity to optimize. By the contrary, the empirical mean is \textbf{unbiased} and computationally efficient, we can simply use the moving the moving average to efficiently update the estimated mean value (with a unbiased estimator). The empirical results verify the effectiveness of this idea.

\section{Proof of Lemma 1}
For each source $\calS_t$, by introducing the duality of Wasserstein-1 distance, for $y\in\calY$, we have:
\begin{equation*}
\begin{split}
     W_1(\calS_t(z|y)\|\calT(z|y)) & = \sup_{\|d\|_{L}\leq 1} \E_{z\sim\calS_t(z|y)} d(z) - \E_{z\sim\calT(z|y)} d(z)\\
     & =  \sup_{\|d\|_{L}\leq 1} \sum_{z} \calS_t(z|y) d(z) - \sum_{z} \calT(z|y) d(z)\\
     & = \frac{1}{\calT(y)} \sup_{\|d\|_{L}\leq 1}  \frac{\calT(y)}{\calS_t(y)} \sum_{z} \calS_t(z,y)d(z) - \sum_{z} \calT(z,y)d(z)
\end{split}
\end{equation*}
Then by defining $\bar{\alpha}_{t}(z) = \mathbf{1}_{\{(z,y)\sim\calS_t\}} \frac{\calT(Y=y)}{\calS_t(Y=y)} =  \mathbf{1}_{\{(z,y)\sim\calS_t\}}\alpha_t(Y=y)$, we can see for each pair observation $(z,y)$ sampled from the same distribution, then $\bar{\alpha}_{t}(Z=z) = \alpha_{t}(Y=y)$. Then we have:
\begin{equation*}
    \begin{split}
        \sum_{y} \calT(y) W_1(\calS_t(z|y)\|\calT(z|y)) & = \sum_{y} \sup_{\|d\|_{L}\leq 1} \{\sum_{z} \alpha_t(y) \calS_t(z,y)d(z) - \sum_{z} \calT(z,y)d(z)\} \\
        & = \sup_{\|d\|_{L}\leq 1} \sum_{z} \bar{\alpha}_t(z) \calS_t(z) d(z) - \sum_{z} \calT(z) d(z) \\
        & = \sup_{\|d\|_{L}\leq 1} \E_{z\sim\calS_t(z)} \bar{\alpha}_t(z) d(z) - \E_{z\sim\calT(z)} d(z)
    \end{split}
\end{equation*}

We propose a simple example to understand $\bar{\alpha}_t$: supposing three samples in $\calS_t = \{(z_1,Y=1), (z_2,Y=1) , (z_3,Y=0)\}$ then $\bar{\alpha}_t(z_1)=\bar{\alpha}_t(z_2)=\alpha_t(1)$ and $\bar{\alpha}_t(z_3)=\alpha_t(0)$. Therefore, the conditional term is equivalent to the label-weighted Wasserstein adversarial learning. 
We plug in each source domain as weight $\blambda[t]$ and domain discriminator as $d_t$, we finally have Lemma 1.

\section{Derive the label distribution ratio Loss}\label{sec:derive_ratio_loss}
In GLS, we have $\calT(z|y)\approx\calS_t(z|y)$, $\forall t$, then we suppose the predicted target distribution as $\bar{\calT}(y)$. By simplifying the notation, we define $f(z) = \text{argmax}_{y} h(z,y)$ the most possible prediction label output, then we have: 
\begin{align*}
    \bar{\calT}(y) & = \sum_{k=1}^{\calY} \calT(f(z)=y|Y=k)\calT(Y=k) = \sum_{k=1}^{\calY} \calS_t(f(z)=y|Y=k)\calT(Y=k) \\
    & = \sum_{i=1}^{\calY} \calS_t(f(z)=y, Y=k)\alpha_t(k) =\bar{\calT}_{\alpha_t}(y)
\end{align*}
The first equality comes from the definition of target label prediction distribution, $\bar{\calT}(y) = \E_{\calT(z)} \mathbf{1}\{f(z)=y\} = \calT(f(z)=y) = \sum_{k=1}^{\calY} \calT(f(z)=y,Y=k) =\sum_{k=1}^{\calY} \calT(f(z)=y|Y=k)\calT(Y=k)$.

The second equality $\calT(f(z)=y|Y=k) = \calS_t(f(z)=y|Y=k)$ holds since $\forall t$, $\calT(z|y)\approx\calS_t(z|y)$, then for the shared hypothesis $f$, we have $\calT(f(z)=y|Y=k) = \calS_t(f(z)=y|Y=k)$.

The term $\calS_t(f(z)=y,Y=k)$ is the (expected) source prediction confusion matrix, and we denote its empirical (observed) version as $\hat{\calS}_t(f(z)=y,Y=k)$.

Based on this idea, in practice we want to find a $\hat{\alpha}_t$ to match the two predicted distribution $\bar{\calT}$ and $\bar{\calT}_{\hat{\alpha}_t}$. If we adopt the KL-divergence as the metric, we have:
\begin{align*}
    \min_{\hat{\alpha}_t} D_{\text{KL}}(\bar{\calT}\|\bar{\calT}_{\hat{\alpha}_t})
    & = \min_{\hat{\alpha}_t} \E_{y\sim\bar{\calT}} \log(\frac{\bar{\calT}(y)}{\bar{\calT}_{\hat{\alpha}_t}(y)}) 
     = \min_{\hat{\alpha}_t} - \E_{y\sim\bar{\calT}} \log(\bar{\calT}_{\hat{\alpha}_t}(y)) \\
     & = \min_{\hat{\alpha}_t} - \sum_{y} \bar{\calT}(y) \log(\sum_{k=1}^{\calY} \calS_t(f(z)=y, Y=k)\hat{\alpha}_t(k))
\end{align*}
We should notice the nature constraints of label ratio: $\{\hat{\alpha}_t(y) \geq 0, \sum_{y}\hat{\alpha}_{t}(y)\hat{\calS}_t(y) = 1 \}$.  Based on this principle, we proposed the optimization problem to estimate each label ratio. We adopt its empirical counterpart, the empirical confusion matrix $C_{\hat{\calS}_t}[y,k] = \hat{\calS}_t [f(z)=y,Y=k]$, then the optimization loss can be expressed as:
\begin{align*}
    \min_{\hat{\alpha}_t} & \quad\quad  -\sum_{y=1}^{|\calY|} \bar{\calT}(y) \log(\sum_{k=1}^{|\calY|} C_{\hat{\calS}_t}[y,k]\hat{\alpha}_t(k)) \\
    & \text{s.t.} \quad \forall y\in\calY, \hat{\alpha}_t(y)\geq 0, \quad \sum_{y}\hat{\alpha}_t(y)\hat{\calS}_t(y) = 1
\end{align*}

\section{Label Partial Multi-source unsupervised DA}\label{sec:multi_partial_loss}

The key difference between multi-conventional and partial unsupervised DA is the estimation step of $\hat{\alpha}_t$. In fact, we only add a sparse constraint for estimating each $\hat{\alpha}_t$:
\begin{equation}
\begin{split}
    \min_{\hat{\alpha}_t} & \quad\quad  -\sum_{y=1}^{|\calY|} \bar{\calT}(y) \log(\sum_{k=1}^{|\calY|} C_{\hat{\calS}_t}[y,k]\hat{\alpha}_t(k)) +C_2\|\hat{\alpha}_t\|_1 \\
    & \text{s.t.} \quad \forall y\in\calY, \hat{\alpha}_t(y)\geq 0, \quad \sum_{y}\hat{\alpha}_t(y)\hat{\calS}_t(y) = 1
\end{split}
\label{eq:solve_partial}
\end{equation}
Where $C_2$ is the hyper-parameter to control the level of target label sparsity, to estimate the target label distribution. In the paper, we denote $C_2 = 0.1$.

\section{Explicit and Implicit conditional learning}\label{sec:algorithm}
Inspired by Theorem 2, we need to learn the function $g:\calX\to\calZ$ and $h:\calZ\times\calY\to\R$ to minimize:
\begin{equation*}
    \min_{g,h} \sum_{t}\blambda[t] \hat{R}^{\hat{\alpha}_t}_{\calS_t}(h, g) + C_0 \sum_{t}\blambda[t] \E_{y\sim\hat{\calT}(y)} W_1(\hat{\calS}_t(z|Y=y)\| \hat{\calT}(z|Y=y))
\end{equation*}
This can be equivalently expressed as:
\begin{equation*}
    \begin{split}
         \min_{g,h} \sum_{t}\blambda[t]  \hat{R}^{\alpha_t}_{\calS_t}(h, g) & + \epsilon C_0  \sum_{t}\blambda[t] \E_{y\sim\hat{\calT}(y)} W_1(\hat{\calS}_t(z|Y=y)\| \hat{\calT}(z|Y=y)) \\
        & + (1-\epsilon) C_0 \sum_{t}\blambda[t] \E_{y\sim\hat{\calT}(y)} W_1(\hat{\calS}_t(z|Y=y)\| \hat{\calT}(z|Y=y)) \\
    \end{split}
\end{equation*}
Due to the explicit and implicit approximation of conditional distance, we then optimize an alternative form:
\begin{equation}
\begin{split}
     \min_{g,h} \max_{d_1,\dots,d_T} \underbrace{\sum_{t}\blambda[t]\hat{R}^{\hat{\alpha}_t}_{\calS_t}(h, g)}_{\text{Classification Loss}} & + \epsilon C_0  \underbrace{\sum_{t}\blambda[t] \E_{y\sim\hat{\calT}(y)}\|\mathbf{C}_t^y - \mathbf{C}^y\|_2}_{\text{Explicit Conditional Loss}} \\
        & + (1-\epsilon) C_0 \underbrace{\sum_{t}\blambda[t][\E_{z\sim\hat{\calS}_t(z)} \bar{\alpha}^t(z) d(z) - \E_{z\sim\hat{\calT}(z)} d(z)]}_{\text{Implicit Conditional Loss}} \\
\end{split}
\label{equation:full_loss}
\end{equation}

\begin{figure}[!t]
    \centering
    \includegraphics[scale=0.45]{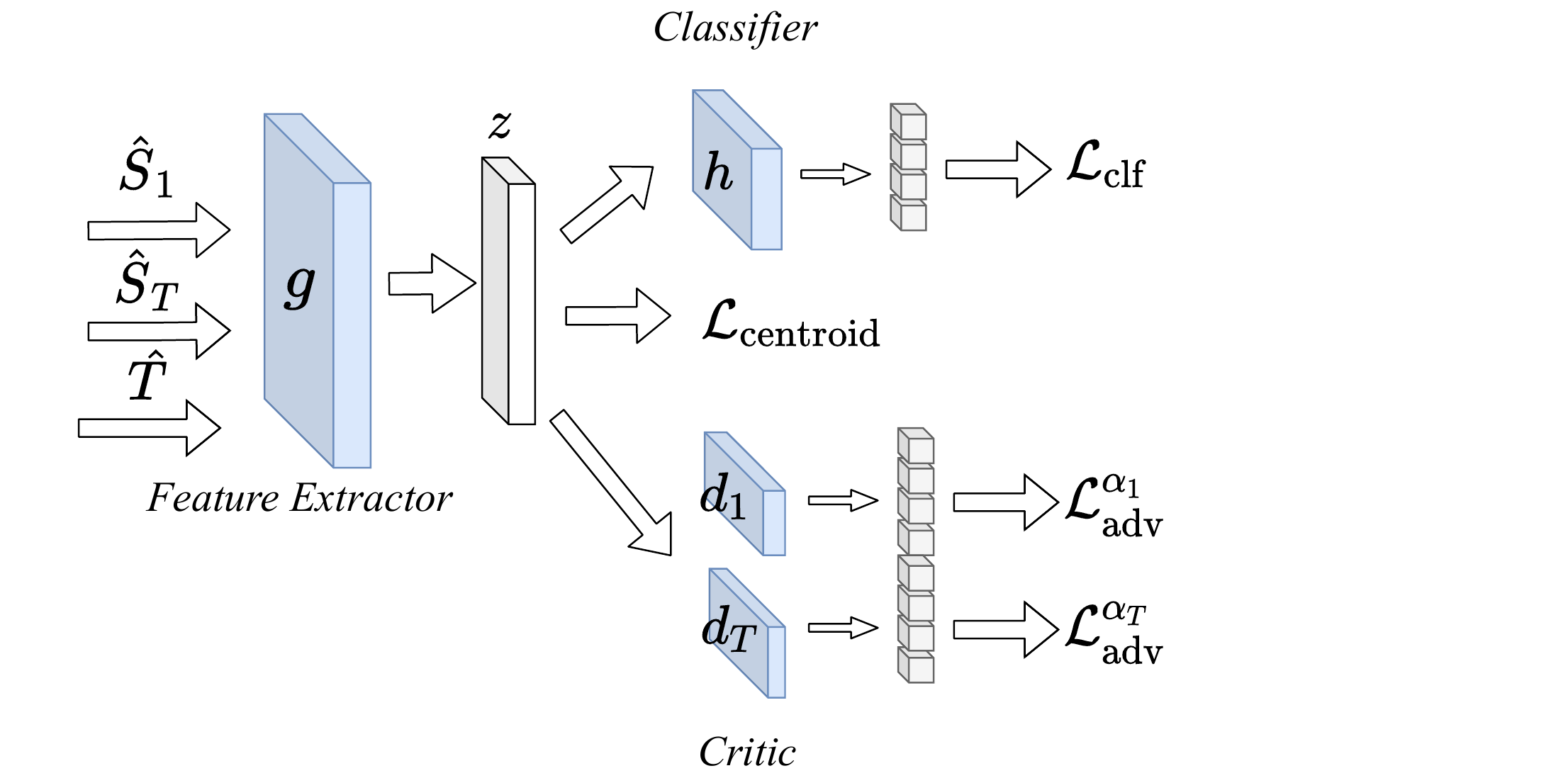}
    \caption{Network Structure of Proposed Approach. It consists three losses: the weighted Classification losses; the centroid matching for explicit conditional matching; the weighted adversarial loss for implicit conditional matching, showed in Eq. (\ref{equation:full_loss})}
    \label{fig:network}
\end{figure}

Where 
\begin{itemize}
    \item $\mathbf{C}_t^y = \sum_{(z_t, y_t) \sim \hat{\calS}_t} \mathbf{1}_{\{y_t = y\}} z_t $ the centroid of label $Y=y$ in source $\calS_t$.
    \item $\mathbf{C}^y = \sum_{(z_t, y_p) \sim \hat{\calT}} \mathbf{1}_{\{y_p = y\}} z_t $ the centroid of pseudo-label $Y=y_p$ in target $\calS_t$. (If it is the unsupervised DA scenarios).
    \item $\bar{\alpha}_{t}(z) = \mathbf{1}_{\{(z,y)\sim\calS_t\}} \hat{\alpha}_{t}(Y=y)$, namely if each pair observation $(z,y)$ from the distribution, then $\bar{\alpha}_{t}(Z=z) = \hat{\alpha}_{t}(Y=y)$. 
    \item $d_1,\cdots,d_T$ are domain discriminator (or critic function) restricted within $1$-Lipschitz function.
    \item $\epsilon\in[0,1]$ is the adjustment parameter in the trade-off of explicit and implicit learning. We fix $\epsilon=0.5$ in the experiments.
    \item $\hat{\calT}(y)$ empirical target label distribution. (In the unsupervised DA scenarios, we approximate it by predicted target label distribution $\bar{\calT}(y)$.) 
\end{itemize}

\paragraph{Gradient Penalty} In order to enforce the Lipschitz property of the statistic critic function, we adopt the gradient penalty term \citep{gulrajani2017improved}. More concretely, given two samples $z_s\sim\calS_t(z)$ and $z_t\sim\calT(z)$ we generate an interpolated sample $z_{\text{int}} = \xi z_s + (1-\xi) z_t$ with $\xi\sim\text{Unif}[0,1]$. Then we add a gradient penalty $\|\nabla d(z_{\text{int}})\|^2_2 $ as a regularization term to control the Lipschitz property w.r.t. the discriminator $d_1,\cdots, d_T$.

\section{Algorithm Descriptions}\label{sec:algo_description}
We propose a detailed pipeline of the proposed algorithm in the following, shown in Algorithm \ref{WMDARN_algo} and \ref{WMT_algo}. As for updating $\blambda$ and $\alpha_t$, we iteratively solve the convex optimization problem after each training epoch and updating them by using the moving average technique.

For solving the $\blambda$ and $\alpha_t$, we notice that frequently updating these two parameters in the mini-batch level will lead to an instability result during the training. \footnote{In the label distribution shift scenarios, the mini-batch datasets are highly labeled imbalanced. If we evaluate $\alpha_t$ over the mini-batch, it can be computationally expensive and unstable.} As a consequence, we compute the accumulated confusion matrix, weighted prediction risk, and conditional Wasserstein distance for the whole training epoch and then solve the optimization problem. We use CVXPY to optimize the two standard convex losses. \footnote{The optimization problem w.r.t. $\alpha_t$ and $\blambda$ is not large scale, then using the standard convex solver is fast and accurate.}

\paragraph{Comparison with different time and memory complexity.} We discuss the time and memory complexity of our approach.

Time complexity: In computing each batch we need to compute $T$ re-weighted loss, $T$ domain adversarial loss and $T$ explicit conditional loss. Then our computational complexity is still $\mathcal(O)(T)$ during the mini-batch training, which is comparable with recent SOTA such as MDAN and DARN.  In addition, after each training epoch we need to estimate $\alpha_t$ and $\blambda$, which can have time complexity $\mathcal{O}(T|\calY|)$ with each epoch. (If we adopt SGD to solve these two convex problems). Therefore, the our proposed algorithm is time complexity $\mathcal{O}(T|\calY|)$. The extra $\calY$ term in time complexity is due to the approach of label shift in the designed algorithm. 

Memory Complexity: Our proposed approach requires $\mathcal{O}(T)$ domain discriminator and $\mathcal{O}(T|\calY|)$ class-feature centroids. By the contrary, MDAN and DARN require $\mathcal{O}(T)$ domain discriminator and M3SDA and MDMN require $\mathcal{O}(T^2)$ domain discriminators. Since our class-feature centroids are defined in the latent space ($z$), then the memory complexity of the class-feature centroids can be much smaller than domain discriminators.

\begin{center}
\begin{algorithm}[h]
		\caption{Wasserstein Aggregation Domain Network (unsupervised scenarios, one iteration)}
		\begin{algorithmic}[1] 
		\REQUIRE Labeled source samples $\hat{\calS}_1,\dots,\hat{\calS}_T$, Target samples $\hat{\calT}$
        \ENSURE  Label distribution ratio $\hat{\alpha}_t$ and task relation simplex $\blambda$. Feature Learner $g$, Classifier $h$, Statistic critic function $d_1,\dots,d_T$, class centroid for source $\mathbf{C}_t^y$ and target $\mathbf{C}^y$ ($\forall t=[1,T],y\in\calY$).
        \STATE \(\triangleright\triangleright\triangleright\) DNN Parameter Training Stage (fixed $\alpha_t$ and $\blambda$) \(\triangleleft\triangleleft\triangleleft\)
        \FOR{mini-batch of samples $(\x_{\calS_1},\y_{\calS_1})\sim\hat{\calS}_1$, $\dots$, $(\x_{\calS_T},\y_{\calS_T})\sim\hat{\calS}_T$, $(\x_{\calT})\sim\hat{\calT}$ }
        \STATE Predict target pseudo-label $\bar{\y}_{\calT} = \text{argmax}_{y} h(g(\x_{\calT}),y)$
        \STATE Compute source confusion matrix for each batch (un-normalized)\\
        ~~~~~~ $C_{\hat{\calS}_t} = \#[\text{argmax}_{y^{\prime}} h(z,y^{\prime})=y, Y=k]$ ($t=1,\dots,T$)
        \STATE Compute the \emph{batched}~class centroid for source $C_t^y$  and target $C^y$.
        \STATE Moving Average for update source/target class centroid: (We set $\epsilon_1 = 0.7$)
        \STATE \quad\quad\quad Source class centroid update \quad $\mathbf{C}_t^y = \epsilon_1 \times \mathbf{C}_t^y + (1-\epsilon_1) \times C_t^y $  
        \STATE \quad\quad\quad Target class centroid update \quad $\mathbf{C}^y = \epsilon_1 \times \mathbf{C}^y + (1-\epsilon_1) \times C^y $ 
        \STATE Updating $g,h,d_1,\dots,d_T$ (SGD and Gradient Reversal), based on Eq.(\ref{equation:full_loss})
        \ENDFOR
        \STATE \(\triangleright\triangleright\triangleright\) Estimation $\hat{\alpha}_t$ and $\blambda$ \(\triangleleft\triangleleft\triangleleft\)
		\STATE Compute the global(normalized) source confusion matrix \\
		$C_{\hat{\calS}_t} = \hat{\calS}_t[\text{argmax}_{y^{\prime}} h(z,y^{\prime})=y, Y=k]$ ($t=1,\dots,T$)
		\STATE Solve $\alpha_t$ (denoted as $\{\alpha_t^{\prime}\}_{t=1}^T$) (Or Eq.(\ref{eq:solve_partial})) in the partial scenario).
		\STATE Update $\alpha_t$ by moving average: $\alpha_t = \epsilon_1 \times \alpha_t  + (1-\epsilon_1) \times \alpha_t^{\prime}$
		\STATE Compute the weighted loss and weighted centroid distance, then solve $\blambda$ (denoted as  $\blambda^{\prime}$) from Sec. 2.3.
		\STATE Updating $\blambda$ by moving average: $\blambda = 0.8 \times \blambda + 0.2 \times \blambda^{\prime} $
        \end{algorithmic}
        \label{WMDARN_algo}
\end{algorithm}
\end{center}

\begin{center}
\begin{algorithm}[h]\label{algo:training_uda}
		\caption{Wasserstein Aggregation Domain Network (Limited Target Data, one iteration)}
		\begin{algorithmic}[1] 
		\REQUIRE Labeled source samples $\hat{\calS}_1,\dots,\hat{\calS}_T$, Target samples $\hat{\calT}$,  Label shift ratio $\alpha_t$
        \ENSURE Task relation simplex $\blambda$. Feature Learner $g$, Classifier $h$, Statistic critic function $d_1,\dots,d_T$, class centroid for source $\mathbf{C}_t^y$ and target $\mathbf{C}^y$ ($\forall t=[1,T],y\in\calY$).
        \STATE \(\triangleright\triangleright\triangleright\) DNN Parameter Training Stage (fixed $\blambda$) \(\triangleleft\triangleleft\triangleleft\)
        \FOR{mini-batch of samples $(\x_{\calS_1},\y_{\calS_1})\sim\hat{\calS}_1$, $\dots$, $(\x_{\calS_T},\y_{\calS_T})\sim\hat{\calS}_T$, $(\x_{\calT})\sim\hat{\calT}$ }
        \STATE Compute the \emph{batched}~class centroid for source $C_t^y$  and target $C^y$.
        \STATE Moving Average for update source/target class centroid: (We set $\epsilon_1 = 0.7$)
        \STATE \quad\quad\quad Source class centroid update \quad $\mathbf{C}_t^y = \epsilon_1 \times \mathbf{C}_t^y + (1-\epsilon_1) \times C_t^y $ 
        \STATE \quad\quad\quad Target class centroid update \quad $\mathbf{C}^y = \epsilon_1 \times \mathbf{C}^y + (1-\epsilon_1) \times C^y $ 
        \STATE Updating $g,h,d_1,\dots,d_T$ (SGD and Gradient Reversal), based on Eq.(\ref{equation:full_loss}).
        \ENDFOR
        \STATE \(\triangleright\triangleright\triangleright\) Estimation $\blambda$ \(\triangleleft\triangleleft\triangleleft\)
		\STATE Solve $\blambda$ by Sec. 2.3. (denoted as  $\blambda^{\prime}$) 
		\STATE Updating $\blambda$ by moving average: $\blambda = \epsilon_1 \times \blambda + (1-\epsilon_1) \times \blambda^{\prime}$
        \end{algorithmic}
        \label{WMT_algo}
\end{algorithm}
\end{center}

\newpage
\section{Dataset Description and Experimental Details}\label{sec: Details}

\subsection{Amazon Review Dataset} 
We used the amazon review dataset \citep{blitzer2007biographies}. It contains four domains 
(Books, DVD, Electronics, and Kitchen) with positive (label "1") and negative product reviews (label "0"). The data size is 6465 (Books), 5586 (DVD), 7681 (Electronics), and 7945 (Kitchen). We follow the common data pre-processing strategies \cite{chen2012marginalized}: use the bag-of-words (BOW) features then extract the top-5000 frequent unigram and bigrams of all the reviews.

We also noticed the original data-set are label balanced $\D(y=0)=\D(y=1)$. To enhance the benefits of the proposed approach, we create a new dataset with label distribution drift. Specifically, 
in the experimental settings, we randomly drop $50\%$ data with label "0" (negative reviews) for all the source data while keeping the target identical, showing in Fig (\ref{fig:amazon_label}).

We choose the MLP model with 
\begin{itemize}
    \item feature representation function $g$: $[5000,1000]$ units
    \item Task prediction and domain discriminator function $[1000,500,100]$ units, 
\end{itemize}

We choose the dropout rate as $0.7$ in the hidden and input layers. The hyper-parameters are chosen based on cross-validation. The neural network is trained for $50$ epochs and the mini-batch size is 20 per domain. The optimizer is Adadelta with a learning rate of 0.5. 

\paragraph{Experimental Setting} We use the amazon Review dataset for two transfer learning scenarios (limited target labels and unsupervised DA). We first randomly select 2K samples for each domain. Then we create a drifted distribution of each source, making each source $\approx1500$ and target sample still 2K. 

In the unsupervised DA, we use these labeled source tasks and \emph{unlabelled} target task, which aims to predict the labels on the target domain. 

In the conventional transfer learning, we random sample only $10\%$ dataset ($\approx 200$ samples) as the target training set and the rest $90\%$ samples as the target test set. 

We select $C_0=0.01$ and $C_1 = 1$ for these two transfer scenarios. In both practical settings, we set the maximum training epoch as 50.

\begin{figure}[h]
  \centering
  \begin{subfigure}{0.45\textwidth}
  \centering
     \includegraphics[scale=0.45]{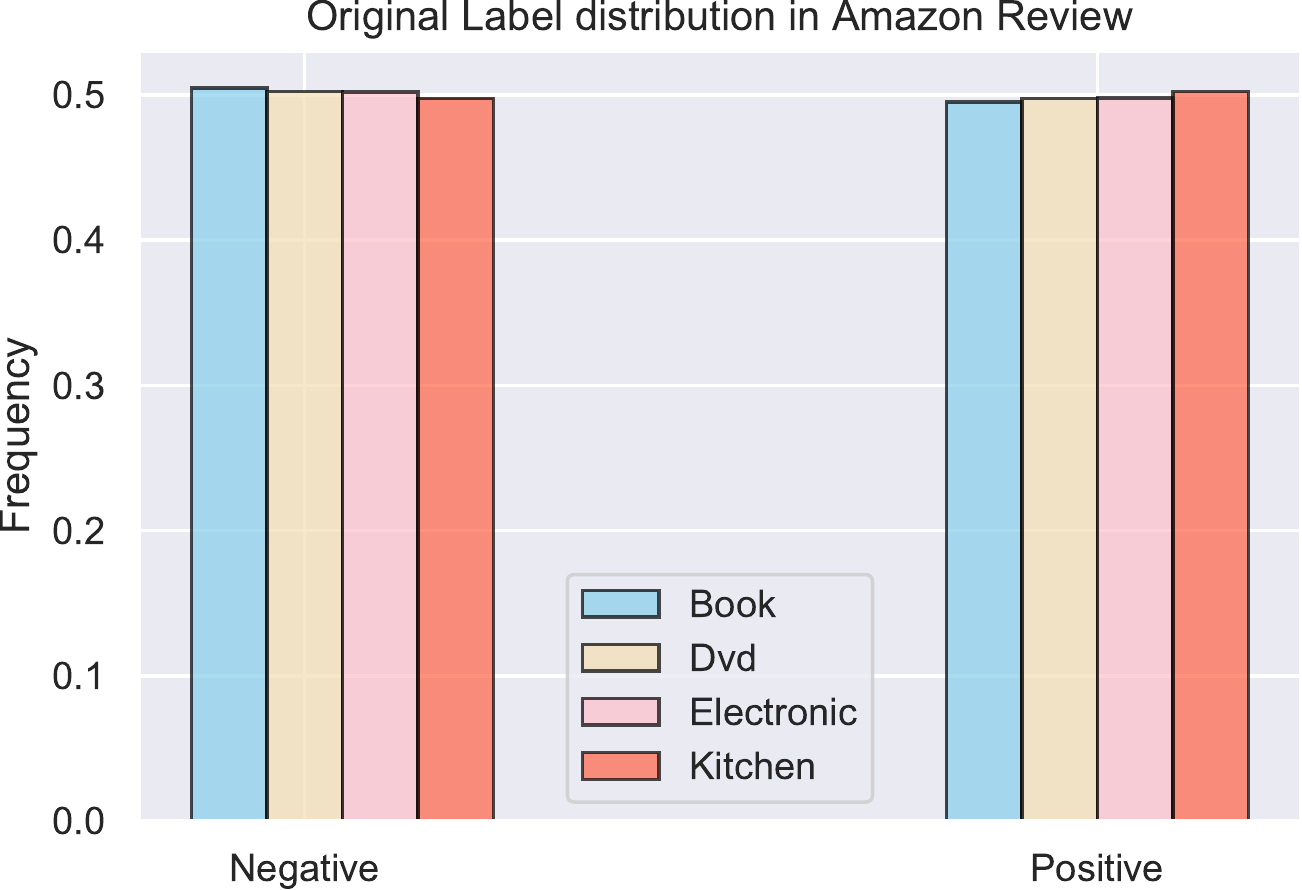}
     \caption{}
  \end{subfigure}
  \begin{subfigure}{0.45\textwidth}
  \centering
     \includegraphics[scale=0.45]{figure/amazon-label-shifted.pdf}
     \caption{}
  \end{subfigure}
  \caption{Amazon Review dataset (a) Original Label Training Distribution; (b) Label-Shifted distribution with sources tasks: Book, Dvd, Electronic, and target task Kitchen. We randomly drop $50\%$ of the negative reviews for all the source distribution while keeping the target label distribution unchanged.}
  \label{fig:amazon_label}
\end{figure}

\subsection{Digit Recognition} 
We follow the same settings of \cite{ganin2016domain} and we use four-digit recognition datasets in the experiments MNIST, USPS, SVHN, and Synth. MNIST and USPS are the standard digits recognition task. Street View House Number (SVHN) \cite{ganin2016domain} is the digit recognition dataset from house numbers in Google Street View Images. 
Synthetic Digits (Synth) \cite{ganin2016domain} is a synthetic dataset that by various transforming SVHN dataset. 

We also visualize the label distribution in these four datasets. The original datasets show an almost uniform label distribution on the MNIST as well as Synth, (showing in Fig. \ref{fig:digits_label} (a)). 
In our paper, we generate a label distribution drift on the source datasets for each multi-source transfer learning.
{\color{blue}{Concretely, we drop $50\%$ of the data on digits 5-9 of all the sources while we keep the target label distribution unchanged.}} (Fig. \ref{fig:digits_label} (b) illustrated one example with sources: Mnist, USPS, SVHN, and Target Synth. We drop the labels only on the sources.) 

MNIST and USPS images are resized to 32 $\times$ 32 and represented as 3-channel color images to match the shape of the other three datasets. Each domain has its own given training and test sets when downloaded. Their respective training sample sizes are 60000, 7219, 73257, 479400, and the respective test sample sizes are 10000, 2017, 26032, 9553. 

The model structure is shown in Fig.~\ref{network_str}. There is no dropout and the hyperparameters are chosen based on cross-validation. It is trained for 60 epochs and the mini-batch size is 128 per domain. The optimizer is Adadelta with a learning rate of 1.0. We adopted $\gamma=0.5$ for MDAN and $\gamma = 0.1$ for DARN in the baseline \citep{wen2019domain}.

\paragraph{Experimental Setting} We use the Digits dataset for two transfer learning scenarios (limited target labels and unsupervised DA). Notice the USPS data has only 7219 samples and the digits dataset is relatively simple. We first randomly select 7K samples for each domain. We create a drifted distribution of each source, making each source $\approx 5300$, and the target sample still 7K. 

In the unsupervised DA, we use these labeled source tasks and \emph{unlabelled} target task, which aims to predict the labels on the target domain. 

In the transfer learning with limited data, we random sample only $10\%$ dataset ($\approx 700$ samples) as the target training set and the rest $90\%$ samples as the target test set.

We select $C_0=0.01$ and $C_1$ as the maximum prediction loss $C_1= \max_{t} {R^{\alpha_t}(h)}$ as the hyper-parameters across these two scenarios. The maximum training epoch is $60$.

\begin{figure}[h]
    \centering
    \begin{enumerate}
     \item Feature extractor: with 3 convolution layers. 
    
    'layer1': 'conv': [3, 3, 64], 'relu': [], 'maxpool': [2, 2, 0],
    
    'layer2': 'conv': [3, 3, 128], 'relu': [], 'maxpool': [2, 2, 0],
    
    'layer3': 'conv': [3, 3, 256], 'relu': [], 'maxpool': [2, 2, 0],
    
    \item Task prediction: with 3 fully connected layers.
    
    'layer1': 'fc': [*, 512], 'act\_fn': 'relu',
    
    'layer2': 'fc': [512, 100], 'act\_fn': 'relu',
    
    'layer3': 'fc': [100, 10],
    
    \item Domain Discriminator:  with 2 fully connected layers.
    
    \emph{reverse\_gradient}()
    
    'layer1': 'fc': [*, 256], 'act\_fn': 'relu',

    'layer2': 'fc': [256, 1], 
\end{enumerate}
    \caption{Neural Network Structure in the digits recognition \citep{ganin2016domain}}
    \label{network_str}
\end{figure}

\begin{figure}[h]
  \centering
  \begin{subfigure}{0.45\textwidth}
  \centering
     \includegraphics[scale=0.45]{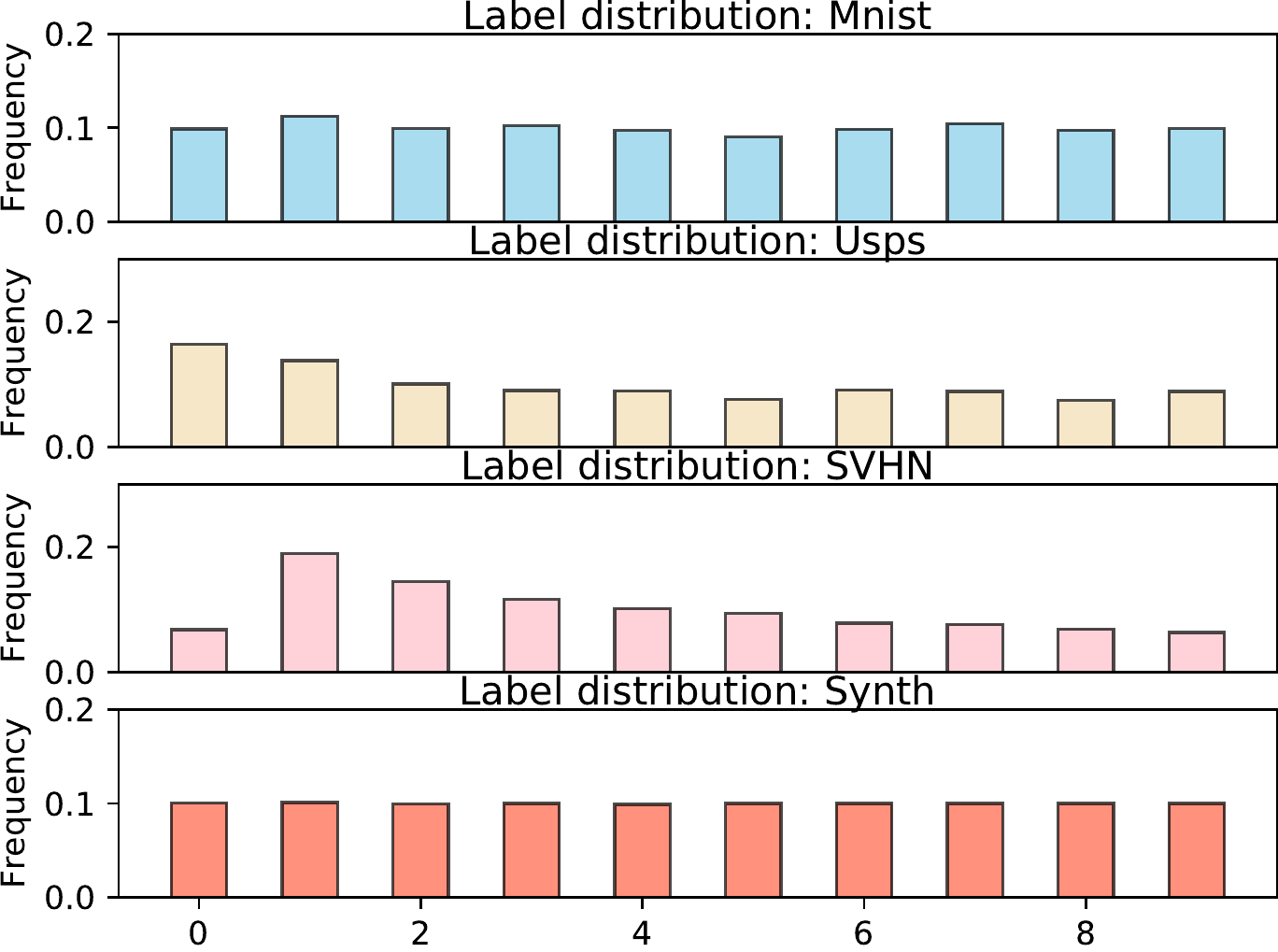}
     \caption{}
  \end{subfigure}
  \begin{subfigure}{0.45\textwidth}
  \centering
     \includegraphics[scale=0.45]{figure/digits_label_shifted.pdf}
     \caption{}
  \end{subfigure}
  \caption{One example in Digits dataset with Sources: MNIST, USPS, SVHN and Target Synth. We randomly drop $50\%$ data on digits 5-9 in all sources while keeping target label distribution unchanged.}
  \label{fig:digits_label}
\end{figure}

\subsection{Office-Home dataset}
To show the dataset in the complex scenarios, we use the challenging Office-Home dataset \citep{venkateswara2017deep}. It contains images of 65 objects such as a spoon, sink, mug, and pen from four different domains: Art (paintings, sketches, and/or artistic depictions), Clipart (clipart images), Product (images without background), and Real-World (regular images captured with a camera). One of the four datasets is chosen as an unlabelled target domain and the other three datasets are used as labeled source domains. 

The dataset size is 2427 (Art), 4365 (Clipart), 4439 (Product), 4357 (Real-World). We follow the same training/test procedure as \citep{wen2019domain}.  We additionally visualize the label distribution $\D(y)$ in four domains in Fig.\ref{fig:office-home-label}, which illustrated the inherent different label distributions. We did not re-sample the source label distribution to uniform distribution in the data pre-processing step. All the baselines are evaluated under the same setting.

We use the ResNet50 \citep{he2016deep} pretrained from the ImageNet in PyTorch as the base network for feature learning and put an MLP with the network structure shown in Fig. \ref{network_str_office}.

\paragraph{Experimental Settings} We use the original Office-Home dataset for two transfer learning scenarios (unsupervised DA and label-partial unsupervised DA). We use SGD optimizer with learning rate 0.005, momentum 0.9 and weight\_decay value 1e-3. It is trained for 100 epochs and the mini-batch size is 32 per domain.  As for the baselines, MDAN use $\gamma$ = 1.0 while DARN use $\gamma$ = 0.5. We select $C_0=0.01$ and $C_1$ as the maximum prediction loss $C_1= \max_{t} {R^{\alpha_t}(h)}$ as the hyper-parameters across these two scenarios.

In the multi-source unsupervised partial DA, we randomly select 35 classes from the target (by repeating 3 samplings), then at each sampling we run 5 times. The final result is based on these $3\times 5 = 15$ repetitions. 


\begin{figure}[h]
    \centering
    \begin{enumerate}
     \item Feature extractor: ResNet50 \citep{he2016deep},
    
    \item Task prediction: with 3 fully connected layers.
    
    'layer1': 'fc': [*, 256], 'batch\_normalization', 'act\_fn': 'Leaky\_relu',
    
    'layer2': 'fc': [256, 256], 'batch\_normalization', 'act\_fn': 'Leaky\_relu',
    
    'layer3': 'fc': [256, 65], 
    
    \item Domain Discriminator:  with 3 fully connected layers.
    
    \emph{reverse\_gradient}()
    
    'layer1': 'fc': [*, 256], 'batch\_normalization', 'act\_fn': 'Leaky\_relu',
    
    'layer2': 'fc': [256, 256], 'batch\_normalization', 'act\_fn': 'Leaky\_relu', 
    
    'layer3': 'fc': [256, 1], 'Sigmoid', 
\end{enumerate}
    \caption{Neural Network Structure in the Office-Home}
    \label{network_str_office}
\end{figure}

\section{Analysis in Unsupervised DA}\label{sec:appendx_uda_additional}

\subsection{Ablation Study: Different Dropping Rate}
To show the effectiveness of our proposed approach, we change the drop rate of the source domains, showing in Fig.(\ref{fig:appendix_amazon_drop_task}). We observe that in task Book, DVD, Electronic, and Kitchen, the results are significantly better under a large label-shift. In the initialization with almost no label shift, the state-of-the-art DARN illustrates a slightly better ($<1\%$) result.  

\begin{figure}[h]
    \centering
\begin{subfigure}{0.2\textwidth}
     \centering
     \includegraphics[scale=0.35]{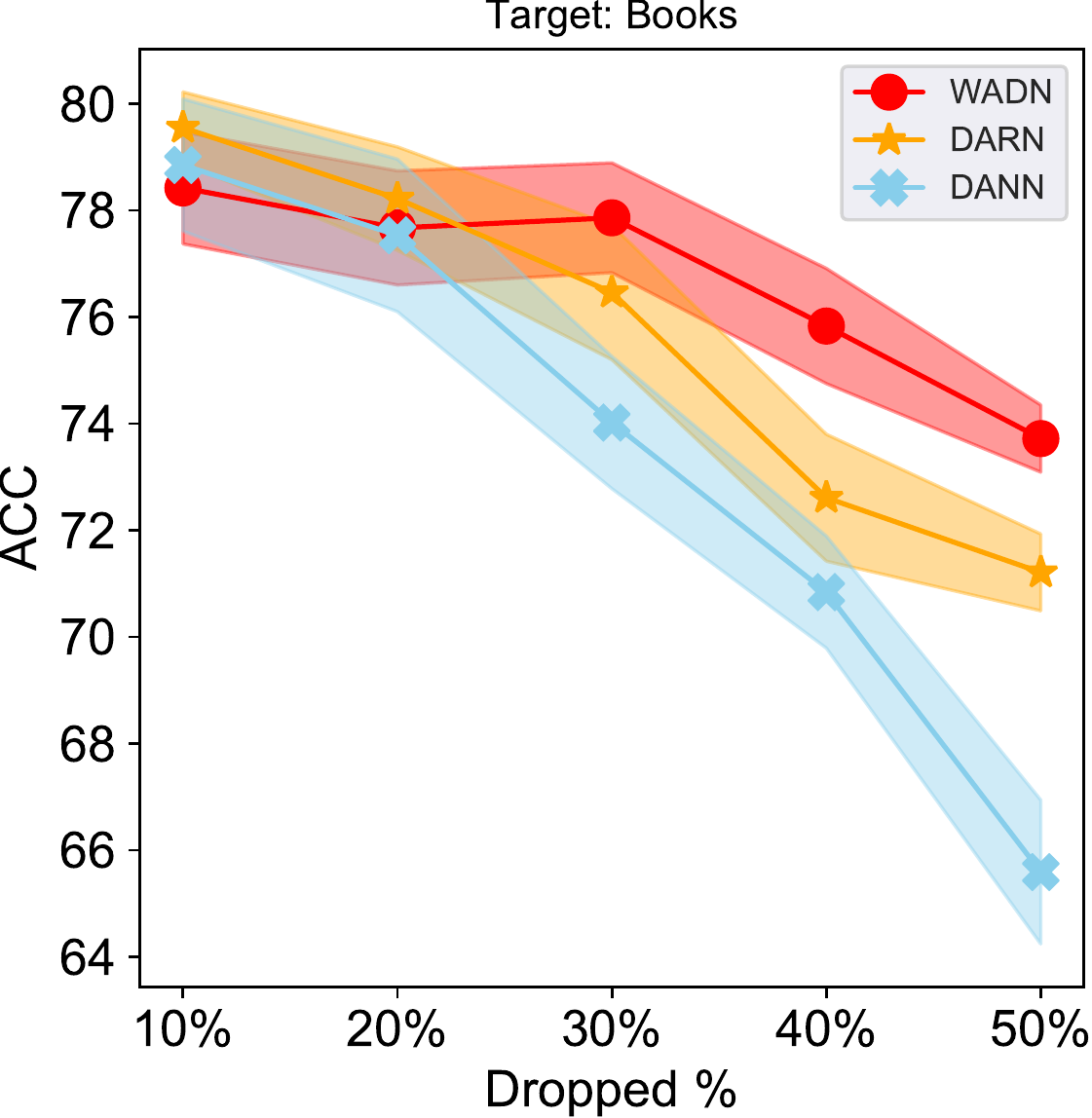}
     \caption{Target: Book}
  \end{subfigure}
  \hfill
  \begin{subfigure}{0.2\textwidth}
  \centering
     \includegraphics[scale=0.35]{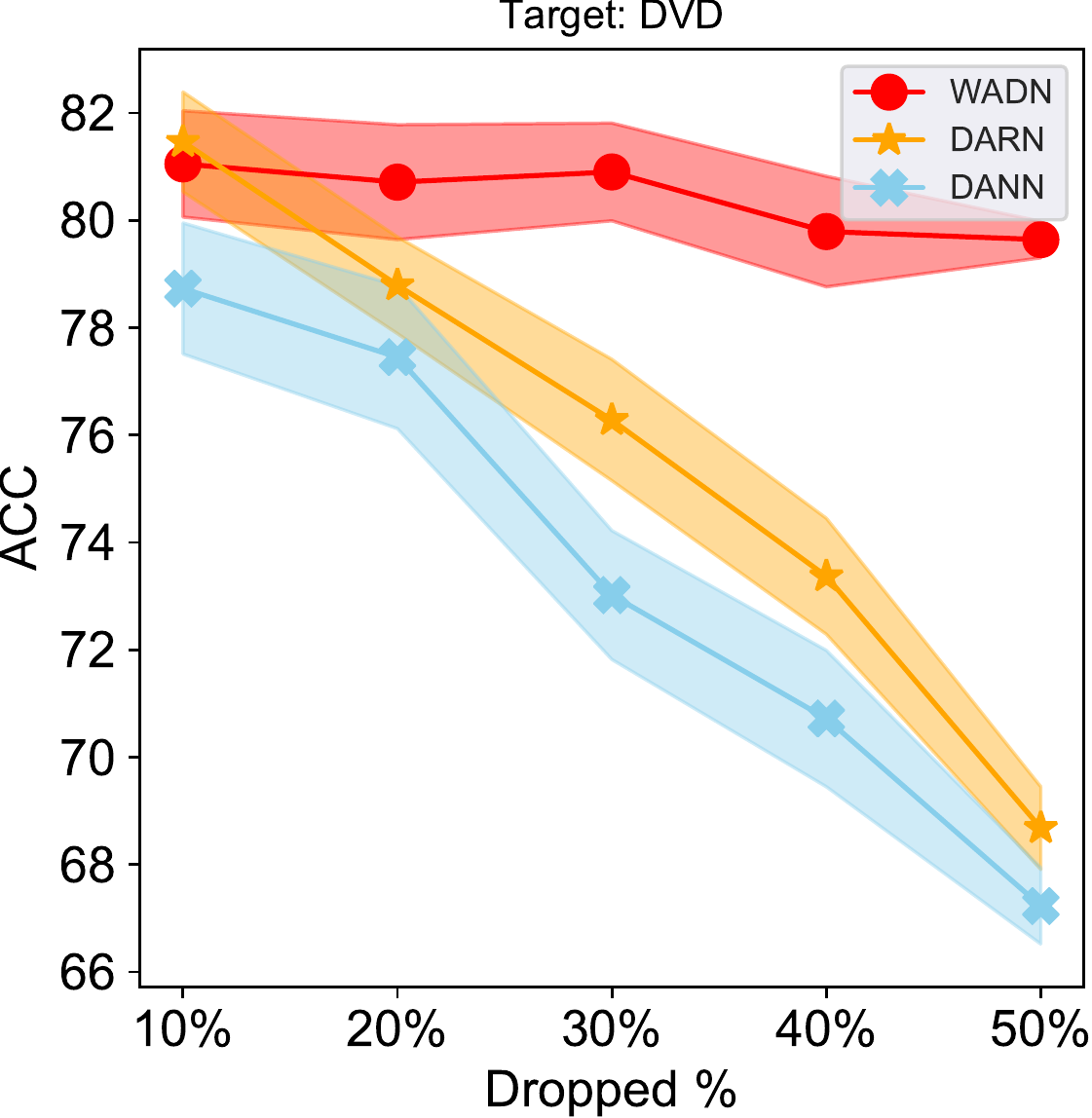}
     \caption{Target: DVD}
  \end{subfigure}
  \hfill
  \begin{subfigure}{0.2\textwidth}
     \centering
     \includegraphics[scale=0.35]{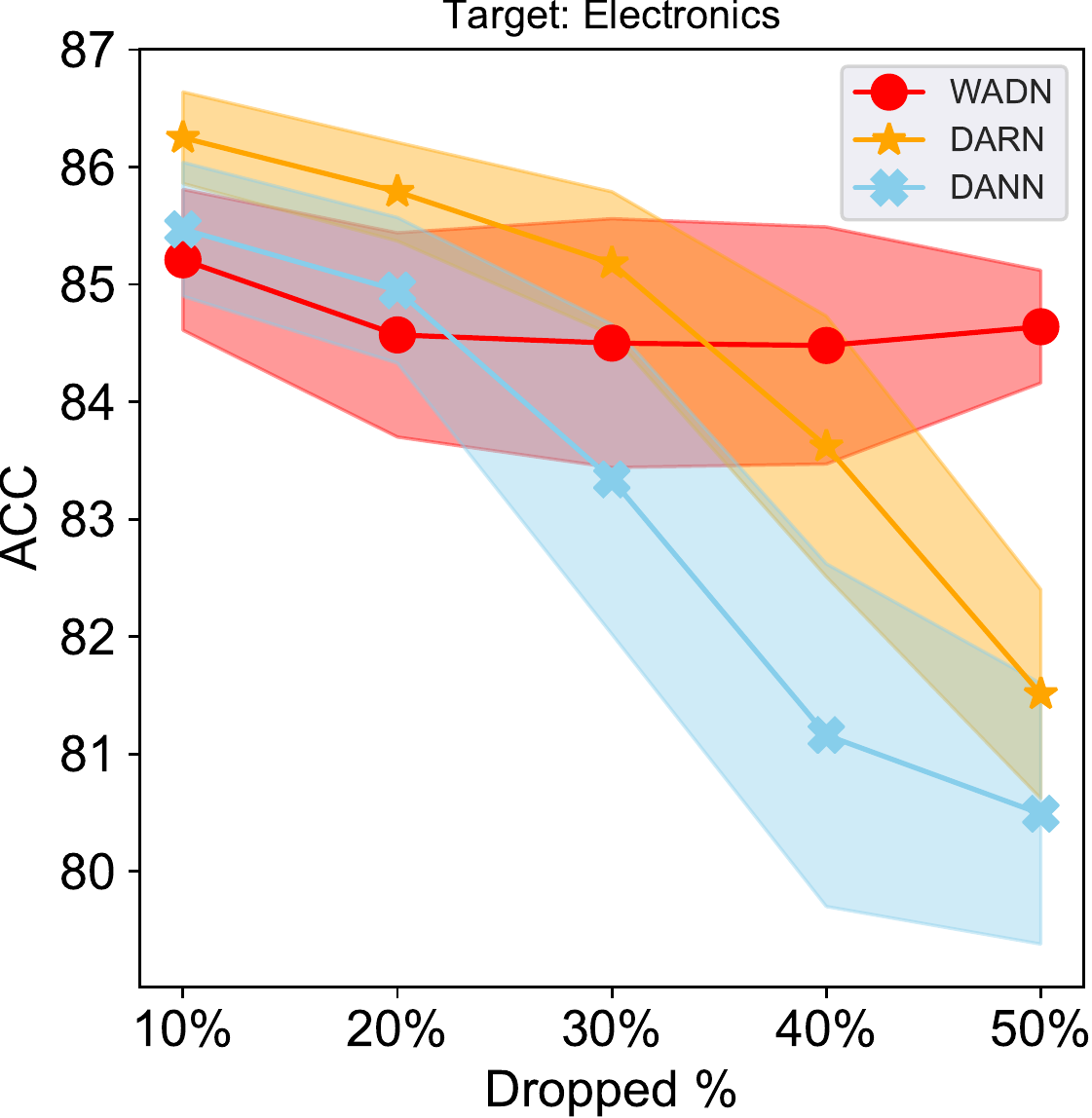}
     \caption{Target: Electronics}
  \end{subfigure}
  \hfill
  \begin{subfigure}{0.2\textwidth}
  \centering
     \includegraphics[scale=0.35]{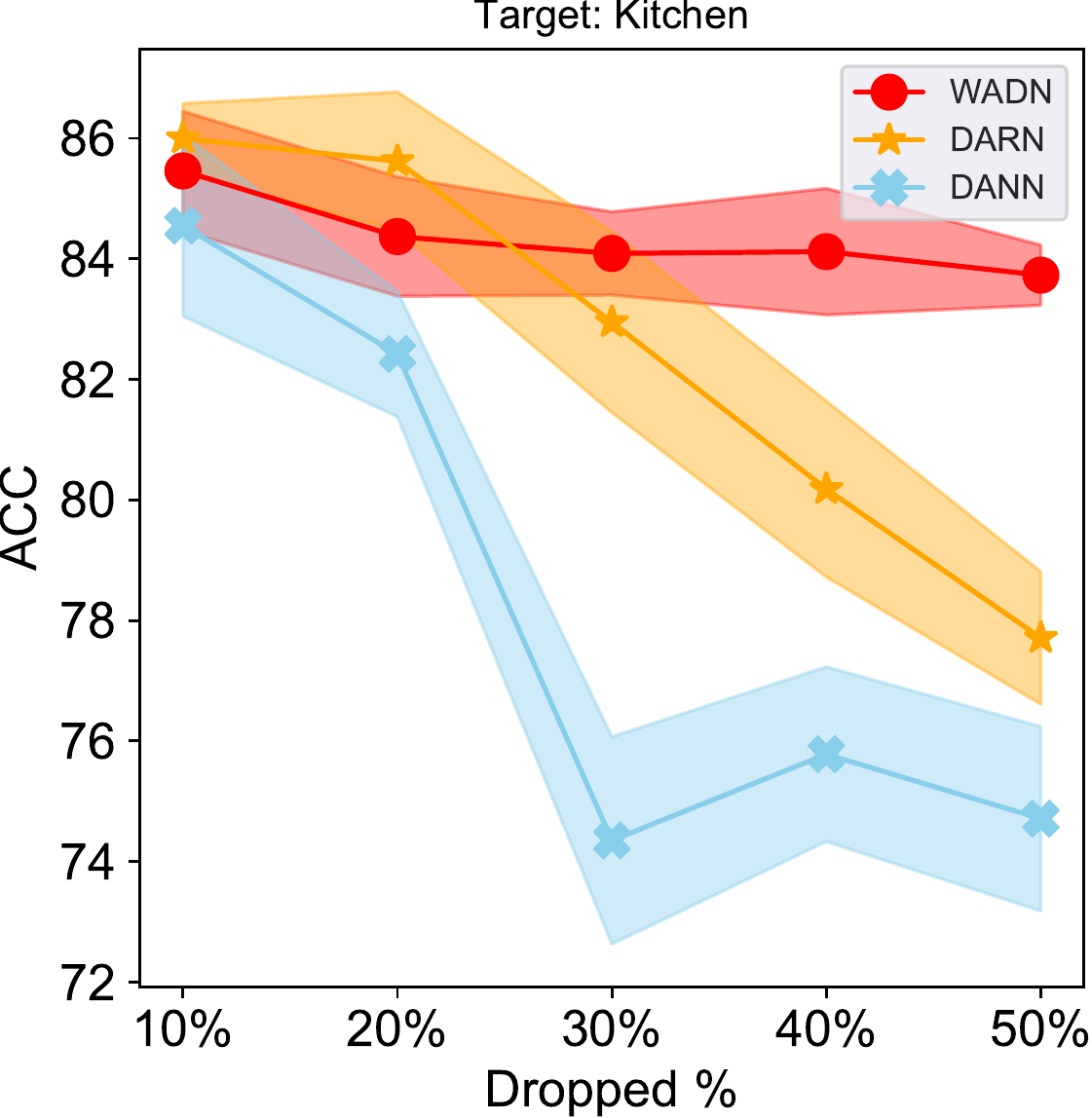}
     \caption{Target: Kitchen}
  \end{subfigure}
  \caption{Different label drift levels on Amazon Dataset. Larger dropping rate means higher label shift.}
  \label{fig:appendix_amazon_drop_task}
\end{figure}

\subsection{Additional Analysis on Amazon Dataset}
We present two additional results to illustrate the working principles of WADN, showing in Fig. (\ref{fig:appendix_amazon_alpha}). 

\begin{figure}
    \centering
\begin{subfigure}{0.4\textwidth}
     \centering
     \includegraphics[scale=0.4]{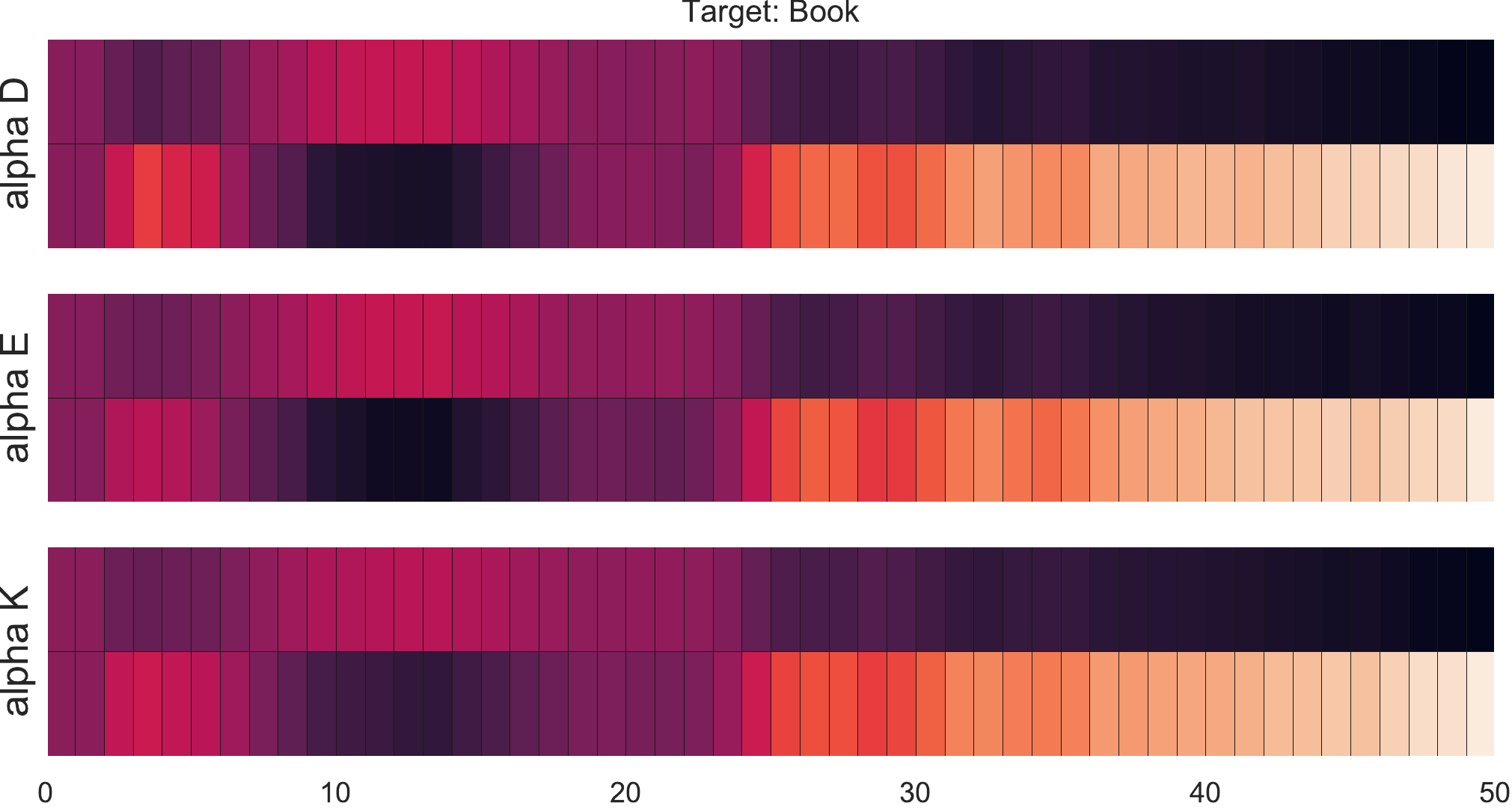}
     \caption{Target: Book}
  \end{subfigure}
  \hfill
  \begin{subfigure}{0.4\textwidth}
  \centering
     \includegraphics[scale=0.4]{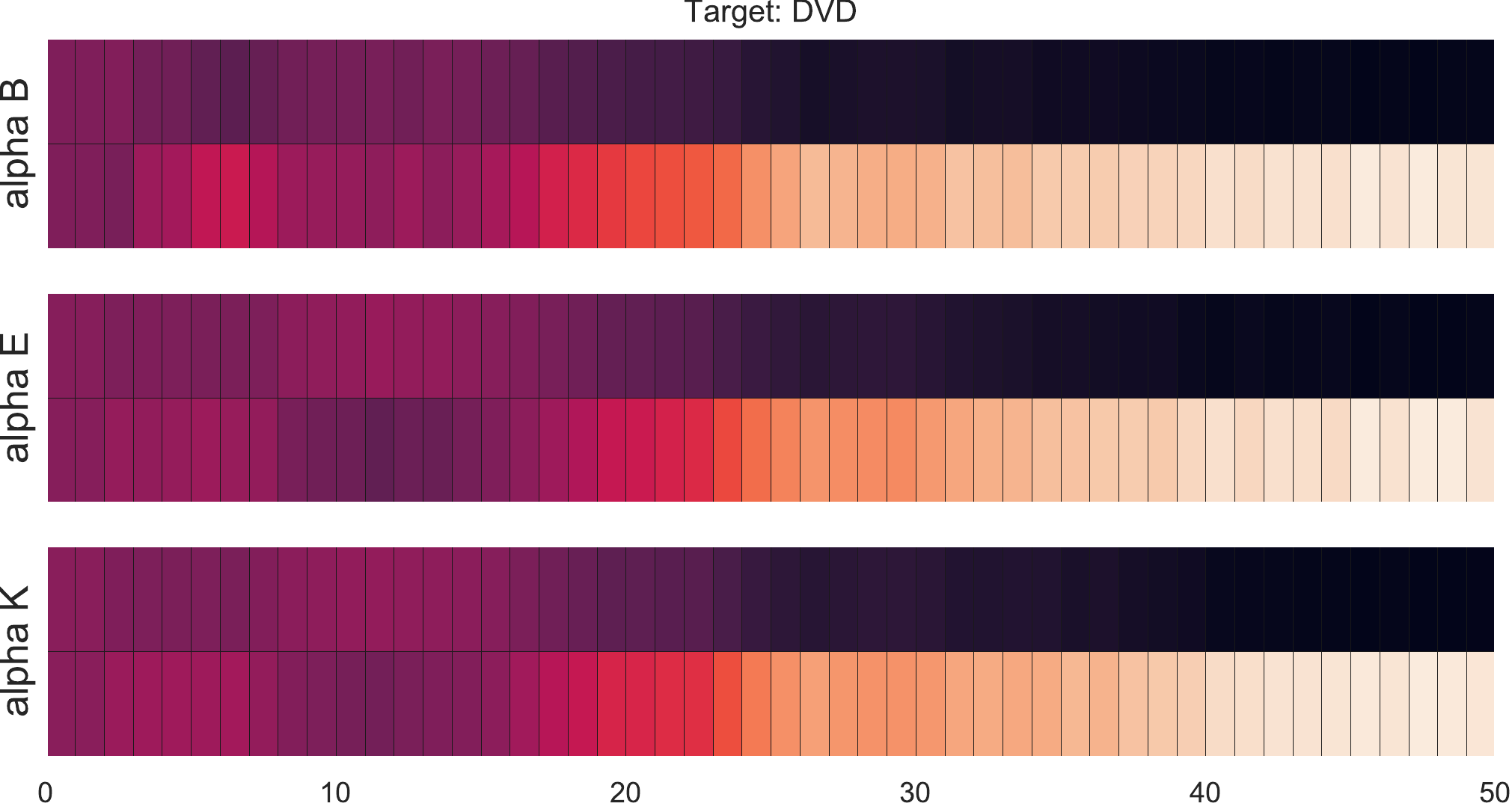}
     \caption{Target: DVD}
  \end{subfigure}
  \begin{subfigure}{0.4\textwidth}
     \centering
     \includegraphics[scale=0.4]{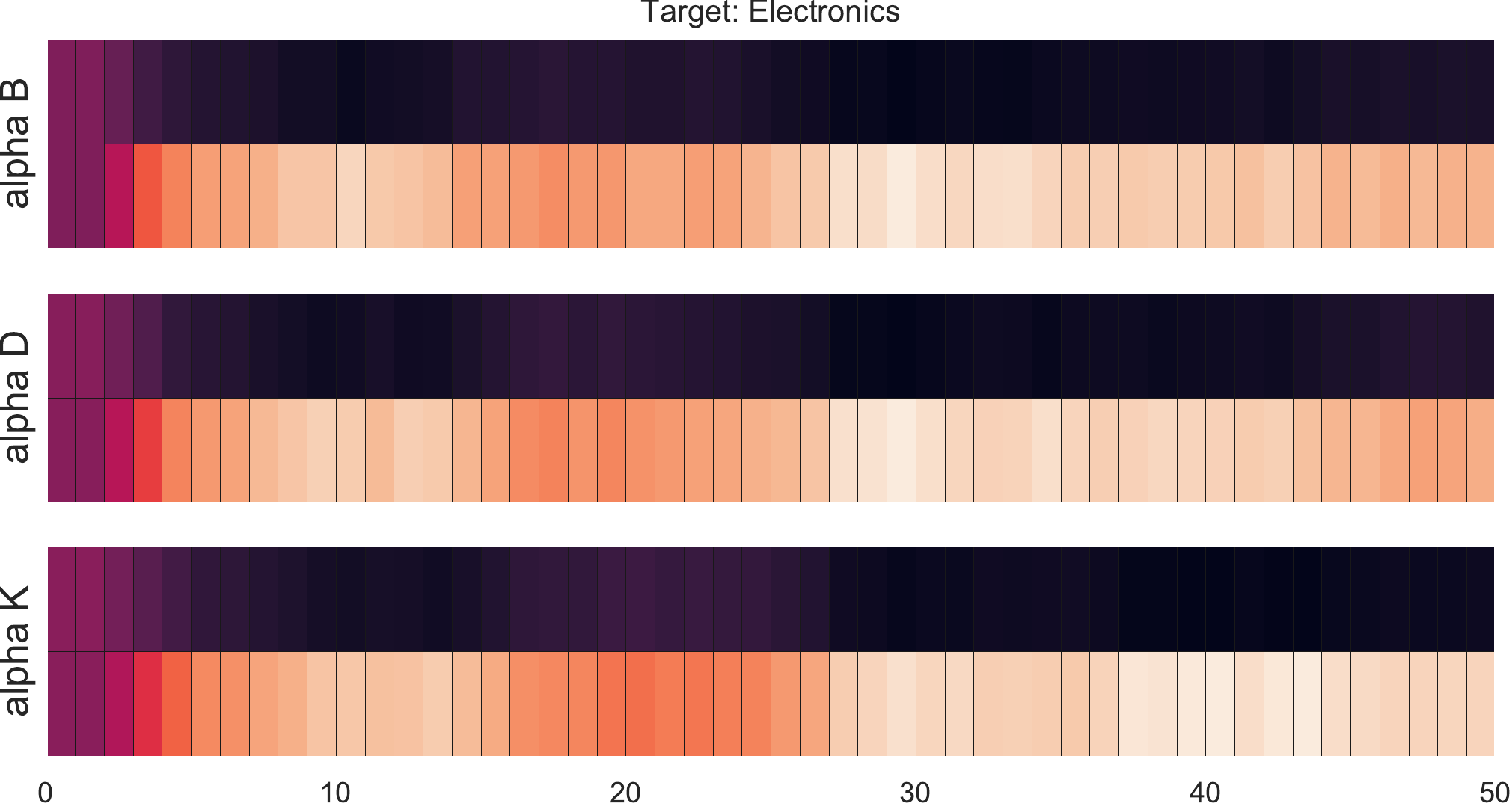}
     \caption{Target: Electronics}
  \end{subfigure}
  \hfill
  \begin{subfigure}{0.4\textwidth}
  \centering
     \includegraphics[scale=0.4]{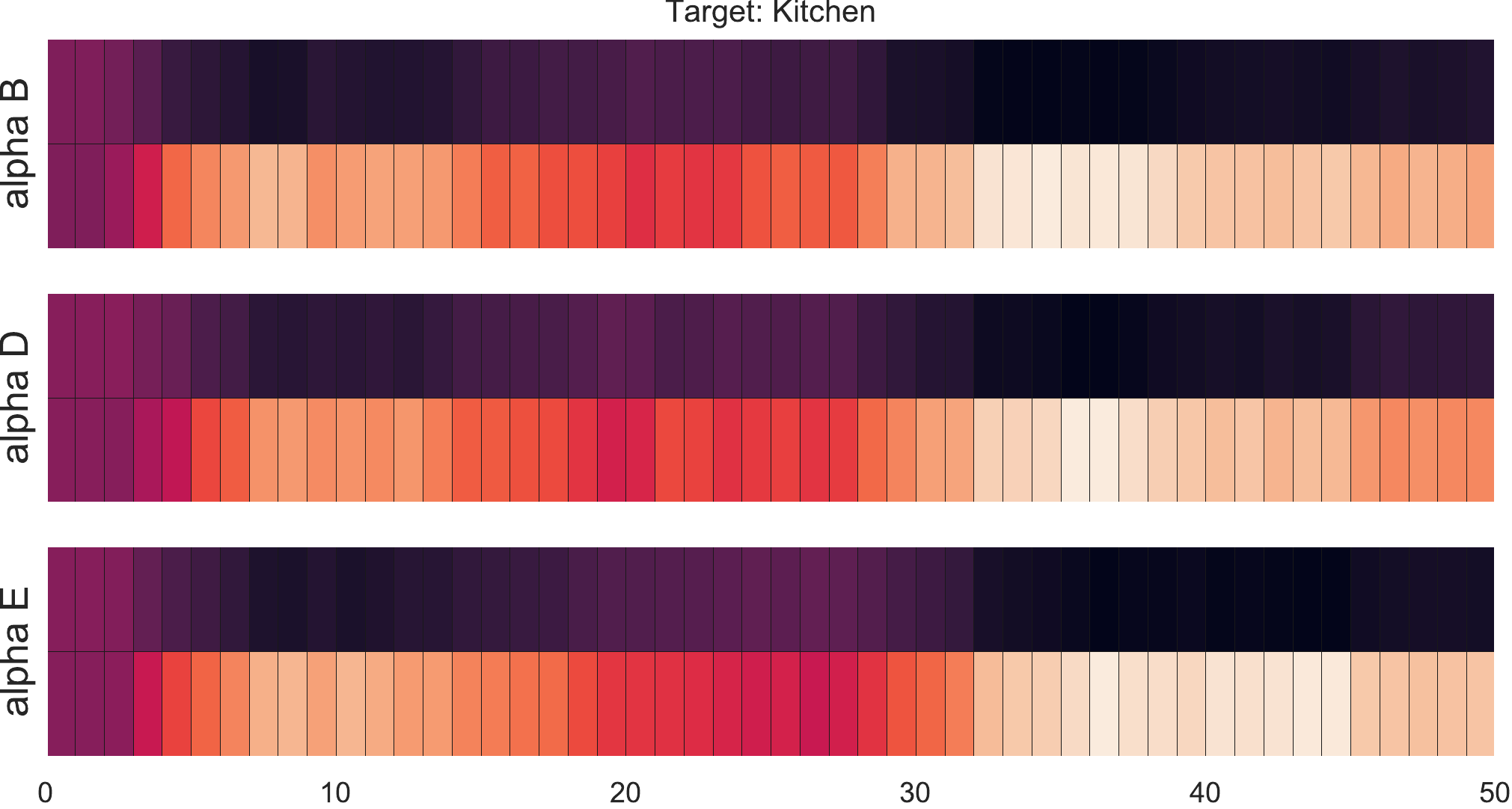}
     \caption{Target: Kitchen}
  \end{subfigure}
  \caption{Amazon Dataset. WADN approach: evolution of $\hat{\alpha}_t$ during the training. Darker indicates higher Value. Since we drop $y=0$ in the sources, then the true $\alpha_t(0) > 1$ will be assigned with higher value.}
  \label{fig:appendix_amazon_alpha}
\end{figure}

We visualize the evolution of $\blambda$ between DARN and WADN, which both used theoretical principled approach to estimate $\blambda$. We observe that in the source shifted data, DARN shows an inconsistent estimator of $\blambda$. This is different from the observation of \cite{wen2019domain}. We think it may in the conditional and label distribution shift problem, using $\hat{R}_{\calS}(h(z)) + \text{Discrepancy}(\calS(z),\calT(z))$ to update $\blambda$ is unstable. In contrast, WADN illustrates a relative consistent estimator of $\blambda$ under the source shifted data. 

In addition, WARN gradually and correctly estimates the unbalanced source data and assign higher wights  $\alpha_t$ for label $y=0$ (first row of Fig.(\ref{fig:appendix_amazon_alpha})). These principles in WADN jointly promote significantly better results.  

\subsection{Additional Analysis on Digits Dataset}
We show the evolution of $\hat{\alpha}_t$ on WADN, which verifies the correctness of our proposed principle. Since we drop digits 5-9 in the source domains, the results in Fig. (\ref{fig:appendix_digits_alpha}) illustrate a higher $\hat{\alpha}_t$ on these digits. 
\begin{figure}[h]
    \centering
\begin{subfigure}{0.4\textwidth}
     \centering
     \includegraphics[scale=0.4]{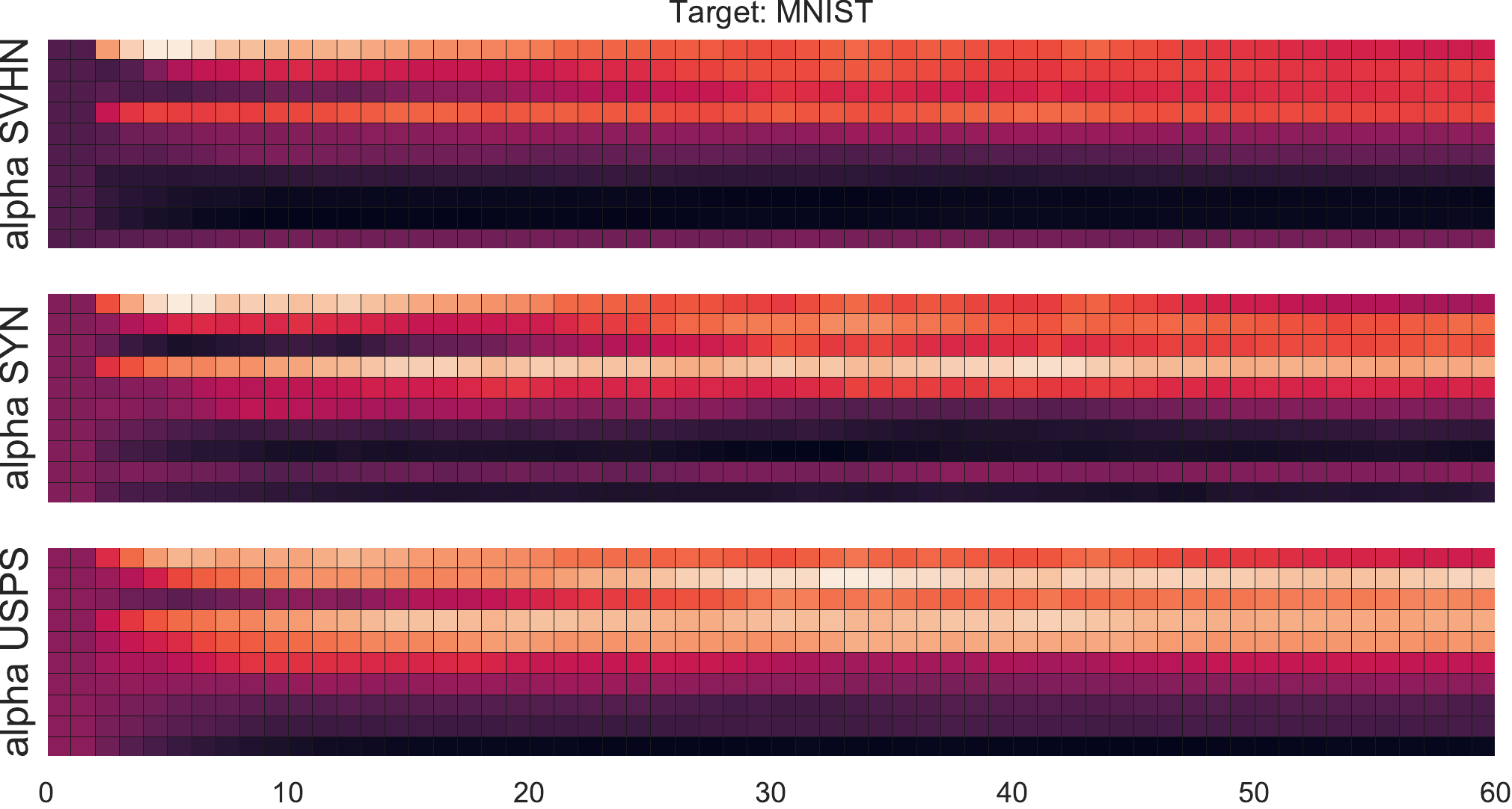}
     \caption{Target: MNIST}
  \end{subfigure}
  \hfill
  \begin{subfigure}{0.4\textwidth}
  \centering
     \includegraphics[scale=0.4]{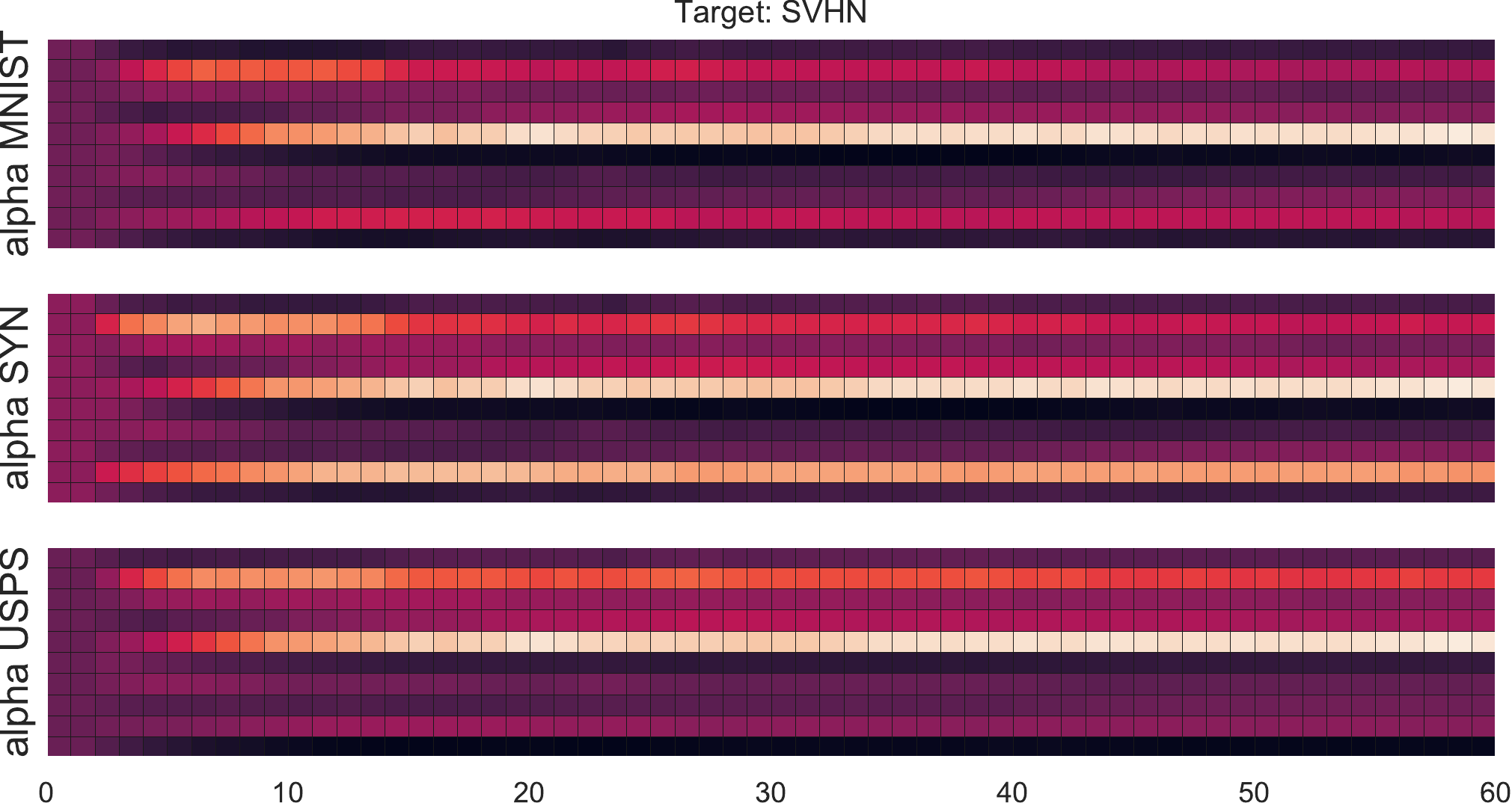}
     \caption{Target: SVHN}
  \end{subfigure}
  \begin{subfigure}{0.4\textwidth}
     \centering
     \includegraphics[scale=0.4]{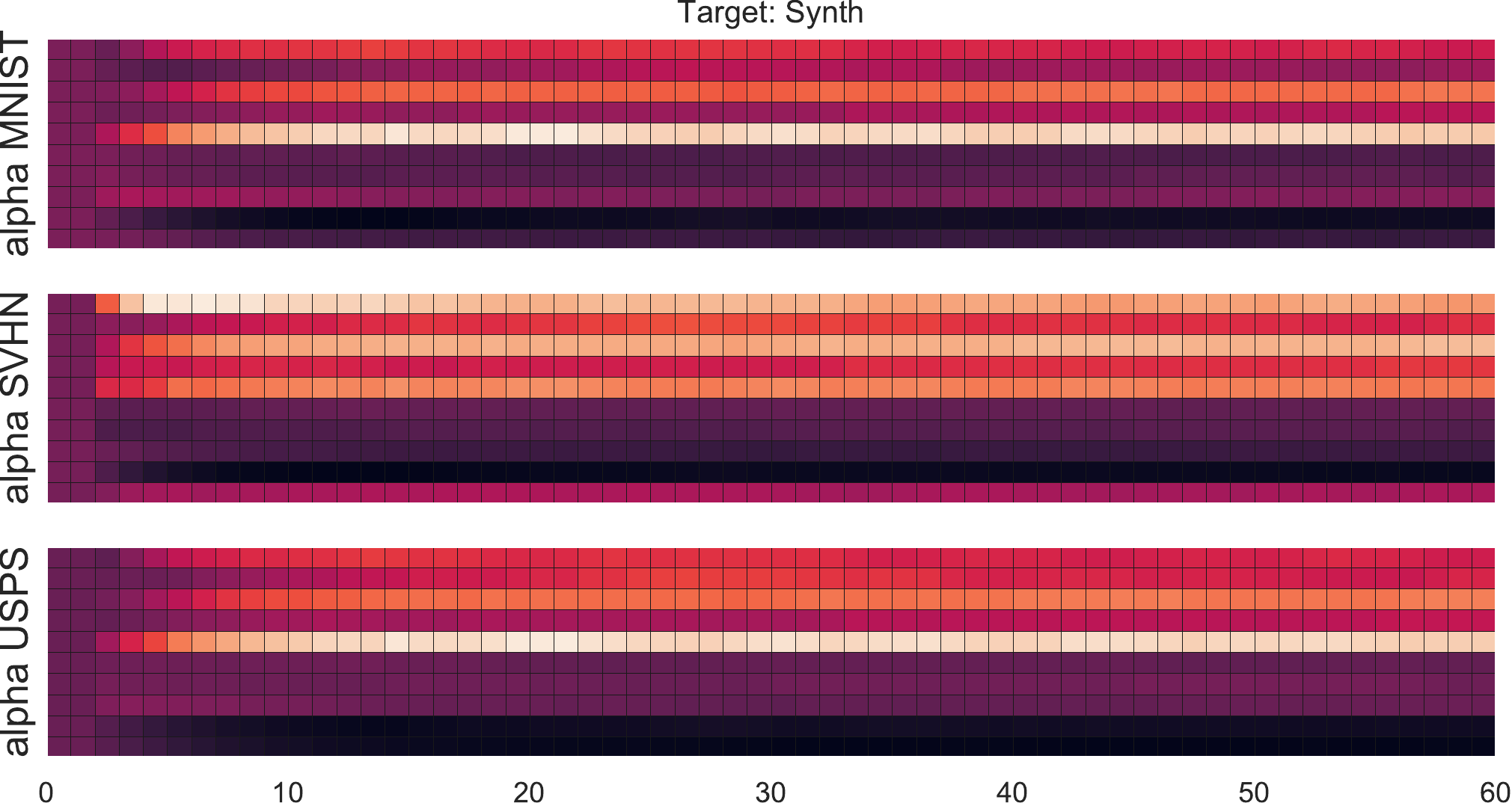}
     \caption{Target: Synth}
  \end{subfigure}
  \hfill
  \begin{subfigure}{0.4\textwidth}
  \centering
     \includegraphics[scale=0.4]{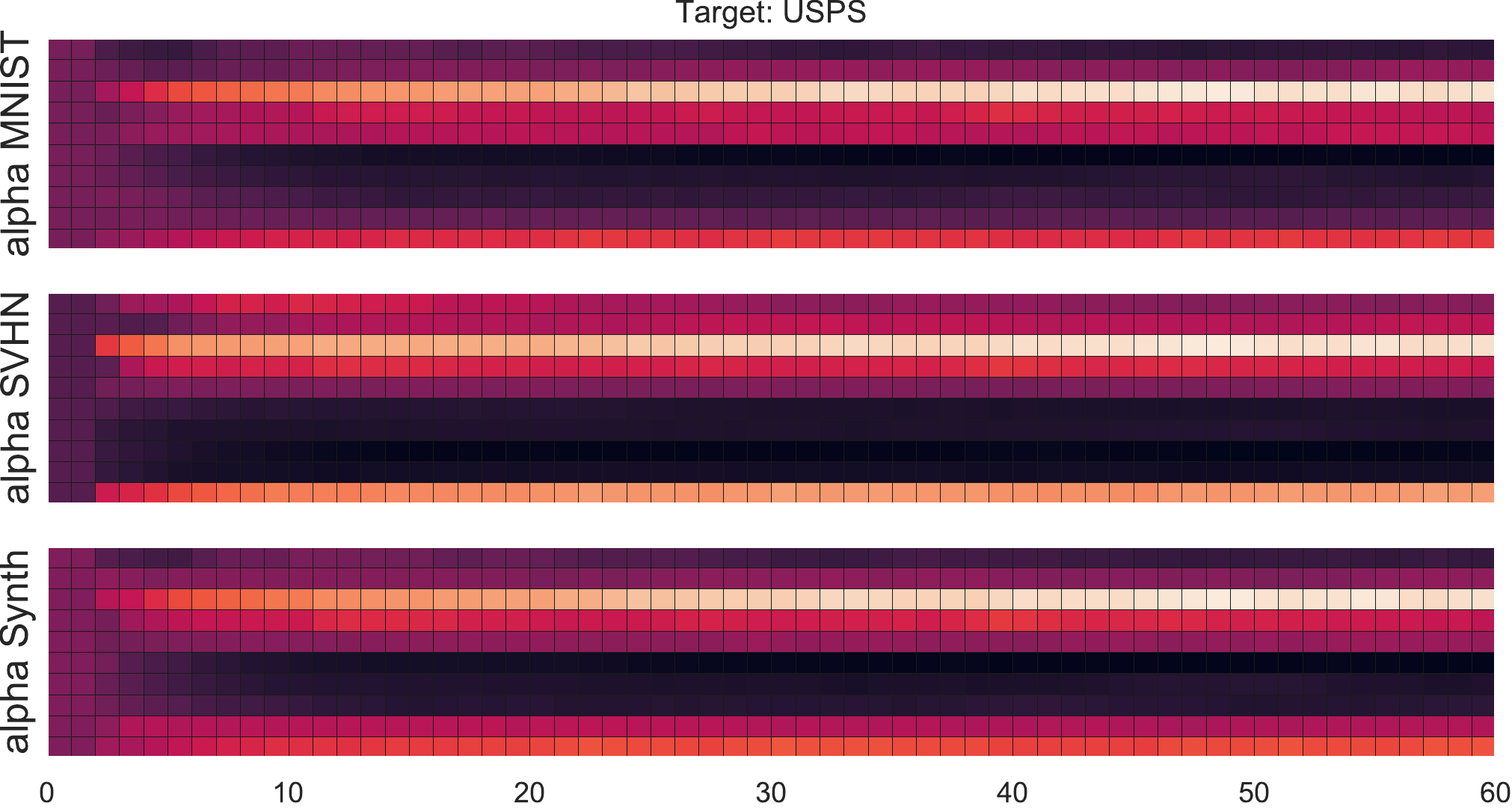}
     \caption{Target: USPS}
  \end{subfigure}
  \caption{Digits Dataset. WADN approach: evolution of $\hat{\alpha}_t$ during the training. Darker indicates higher value. Since we drop digits $5-9$ on source domain, therefore, $\alpha_t(y)$, $y\in[5,9]$ will be assigned with a relative higher value.}
  \label{fig:appendix_digits_alpha}
\end{figure}

\begin{figure}[!t]
  \centering
  \begin{subfigure}{0.3\textwidth}
  \centering
     \includegraphics[scale=0.3]{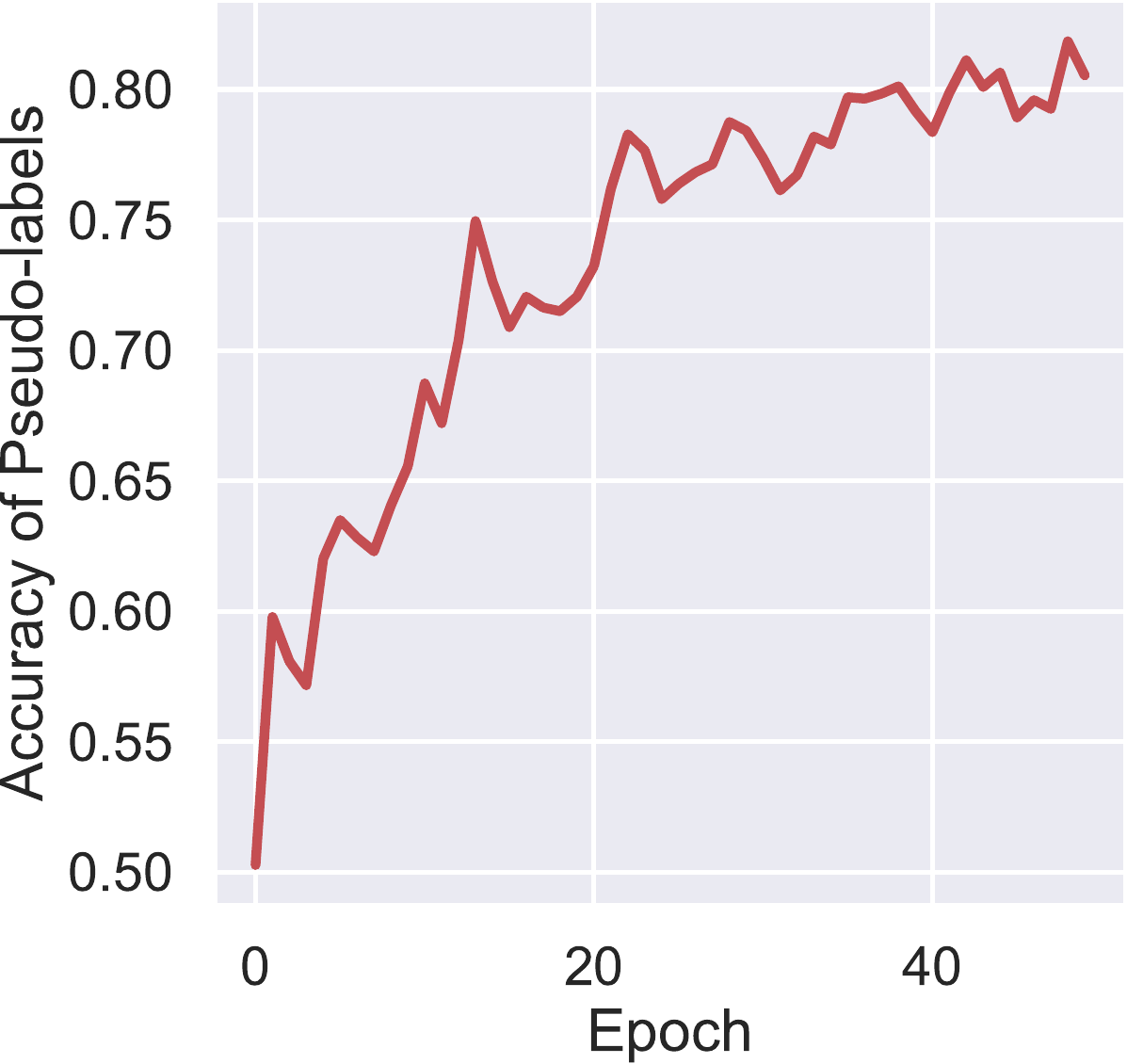}
     \caption{Amazon Review: target DVDs}
  \end{subfigure}
  \begin{subfigure}{0.3\textwidth}
  \centering
     \includegraphics[scale=0.3]{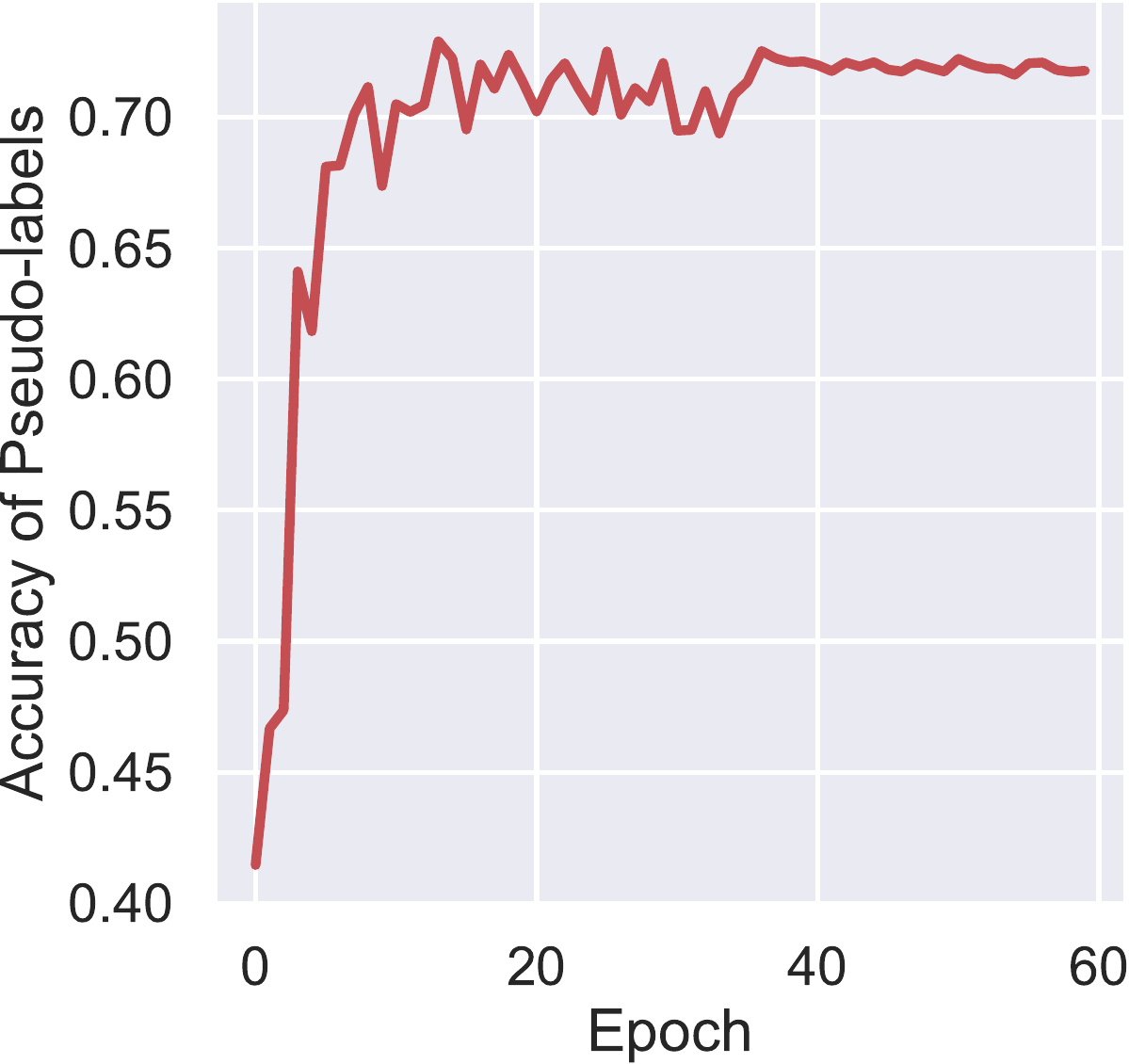}
     \caption{Digits: target SVHN}
  \end{subfigure}
  \begin{subfigure}{0.3\textwidth}
  \centering
     \includegraphics[scale=0.3]{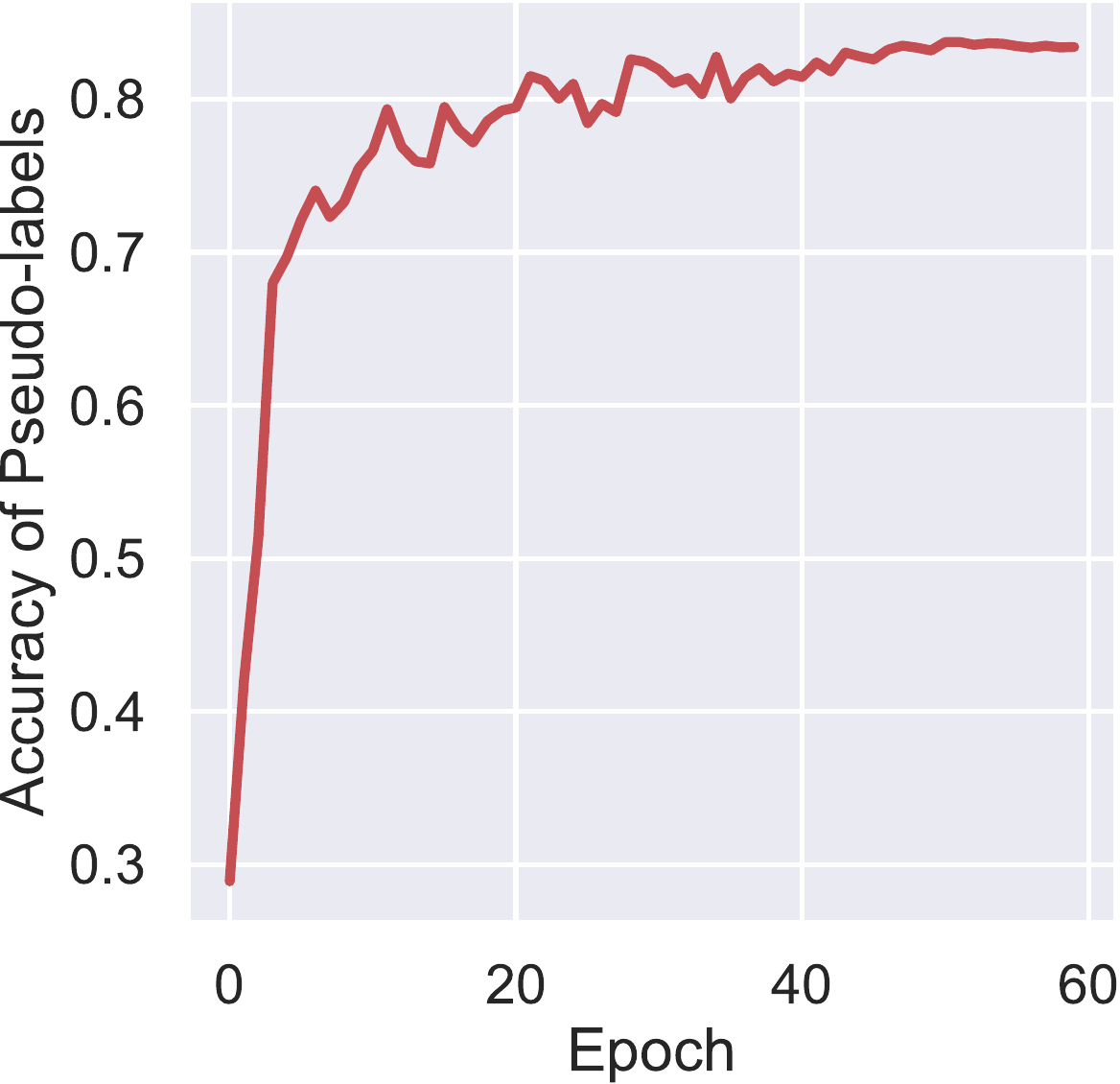}
     \caption{Digits: target Synth}
  \end{subfigure}
  \caption{Evolution of accuracy w.r.t. the predicted target pseudo-labels in different tasks in unsupervised DA.}
  \label{fig:pseudo_labels}
\end{figure}

\section{Partial multi-source Unsupervised DA}\label{sec:appendx_label_partial}
From Fig.~(\ref{fig:appendix_office_home_abl1}), WADN is consistently better than other baselines, given different selected classes.

Besides, when fewer classes are selected, the accuracy in DANN, PADA, and DARN is not drastically dropping but maintaining a relatively stable result. We think the following possible reasons:
\begin{itemize}
   \item The reported performances are based on the \textbf{average of different selected sub-classes rather than one sub-class selection.} From the statistical perspective, if we take a close look at the \textbf{variance}, the results in DANN are \emph{much more unstable} (higher std) inducing by the different samplings. Therefore, the conventional domain adversarial training is improper for handling the partial transfer since it is not reliable and negative transfer still occurs.
    \item In multi-source DA, it is equally important to detect the non-overlapping classes and find the most similar sources. Comparing the baselines that only focus on one or two principles shows the importance of unified principles in multi-source partial DA.
    \item We also observe that in the Real-World dataset, the DANN improves the performance by a relatively large value. This is due to the inherent difficultly of the learning task itself. In fact, the Real-World domain illustrates a much higher performance compared with other domains. According to the Fano lower bound, \emph{a task with smaller classes is generally easy to learn}. It is possible the vanilla approach showed improvement but still with a much higher variance. 
\end{itemize}

\begin{figure}[!t]
  \centering
  \begin{subfigure}{0.32\textwidth}
  \centering
     \includegraphics[scale=0.32]{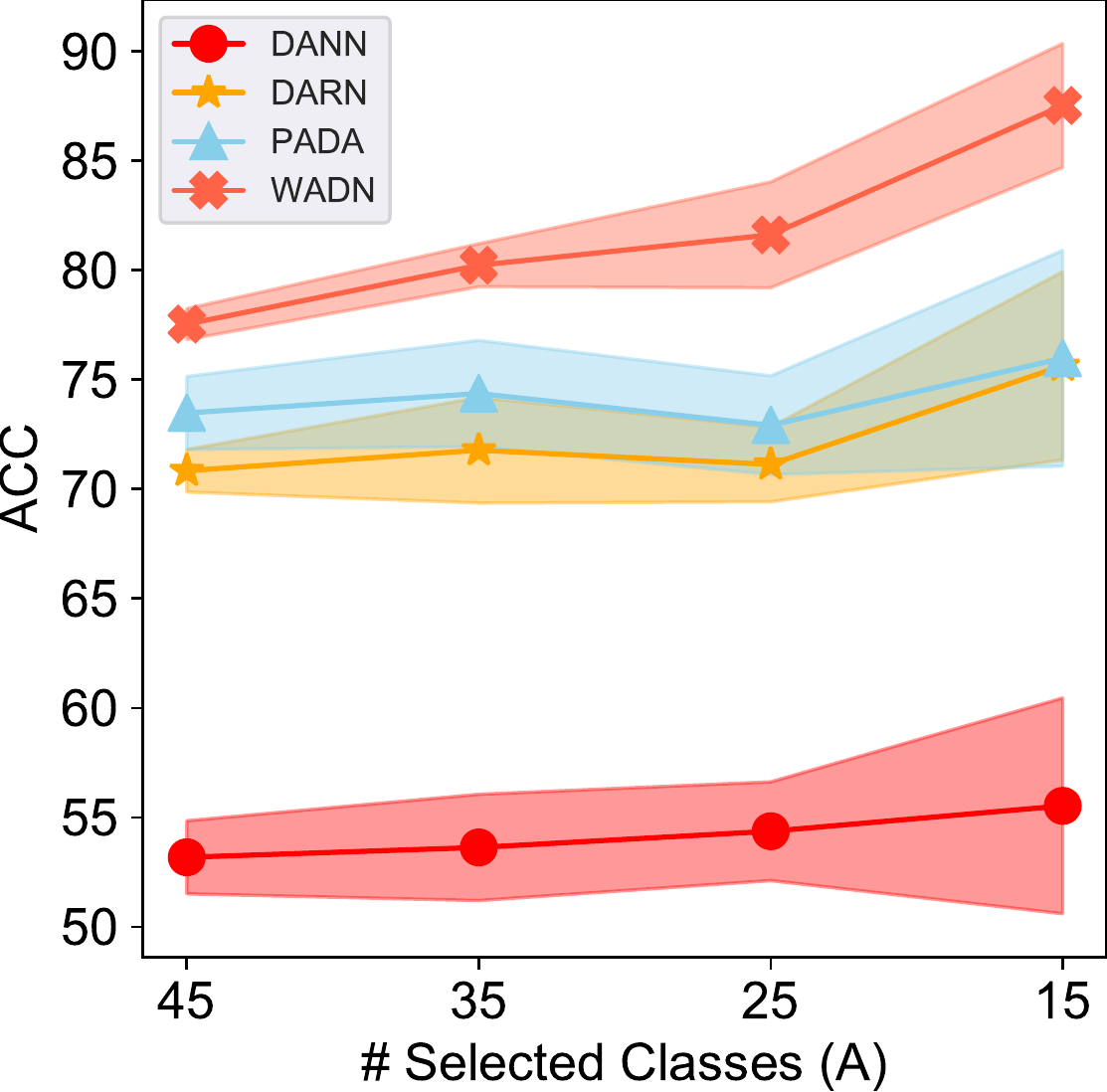}
     \caption{Target: Art}
  \end{subfigure}
  \begin{subfigure}{0.32\textwidth}
  \centering
     \includegraphics[scale=0.32]{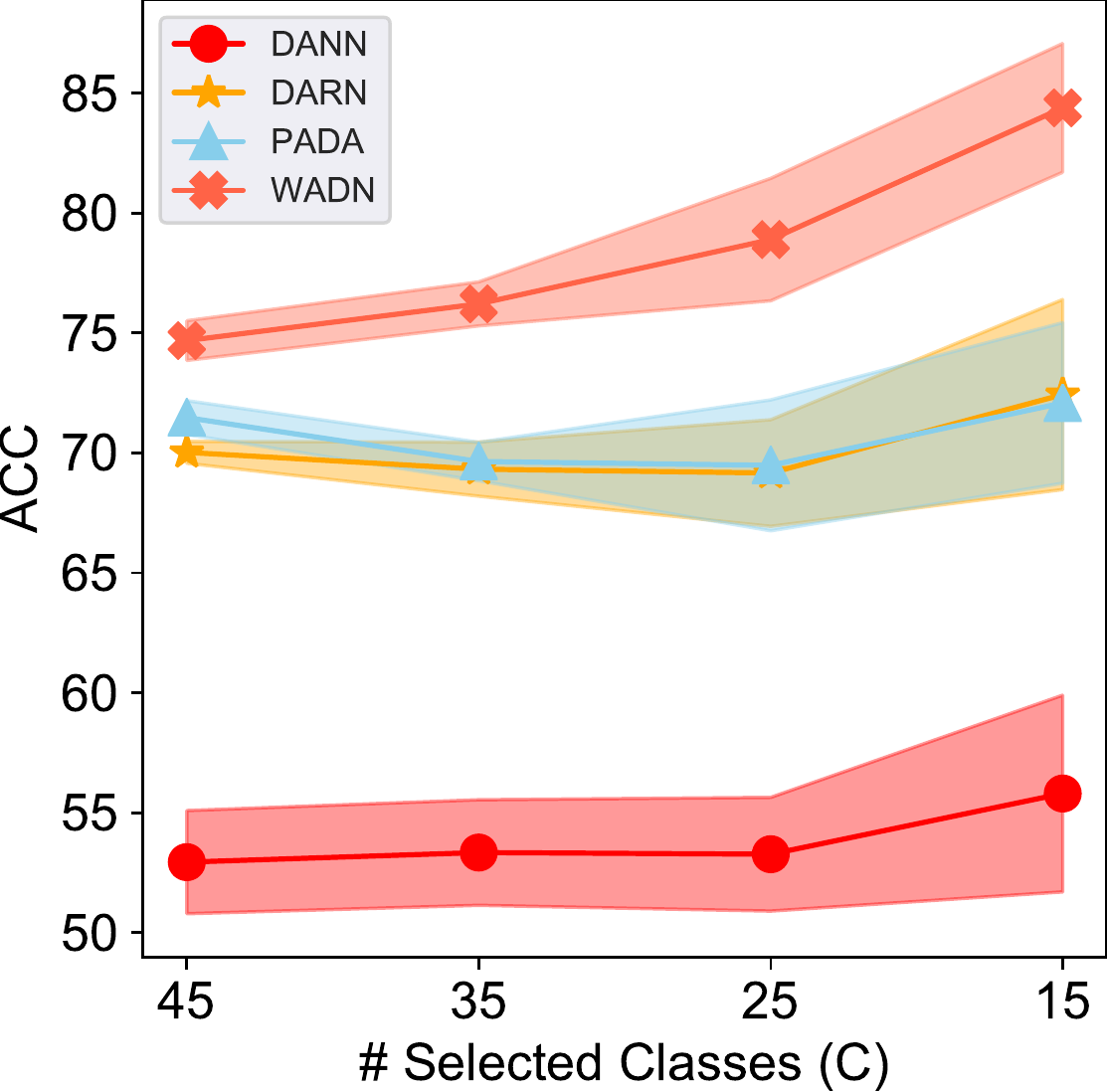}
     \caption{Target: Clipart}
  \end{subfigure}
  \begin{subfigure}{0.32\textwidth}
  \centering
     \includegraphics[scale=0.32]{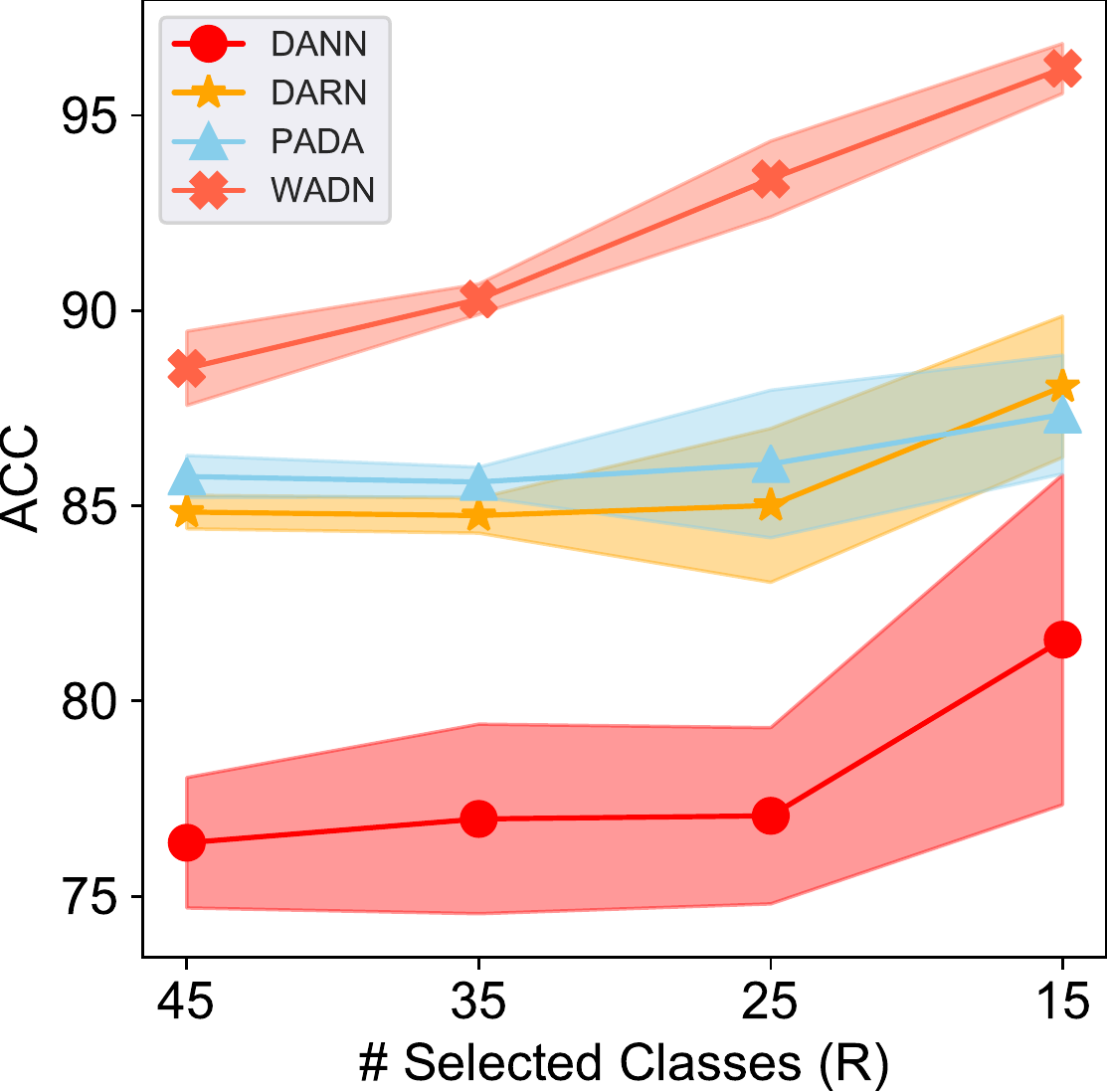}
     \caption{Target: Real-World}
  \end{subfigure}
  \caption{Multi-source Label Partial DA: Performance with different target selected classes.}
  \label{fig:appendix_office_home_abl1}
\end{figure}

Fig (\ref{fig:appendix_office_home_abl2}), (\ref{fig:appendix_office_home_abl3}) showed the estimated $\hat{\alpha}_t$ with different selected classes. The results validate the correctness of WADN in estimating the label distribution ratio.

\begin{figure}[t]
\centering
\begin{subfigure}{0.45\textwidth}
\centering
     \includegraphics[scale=0.45]{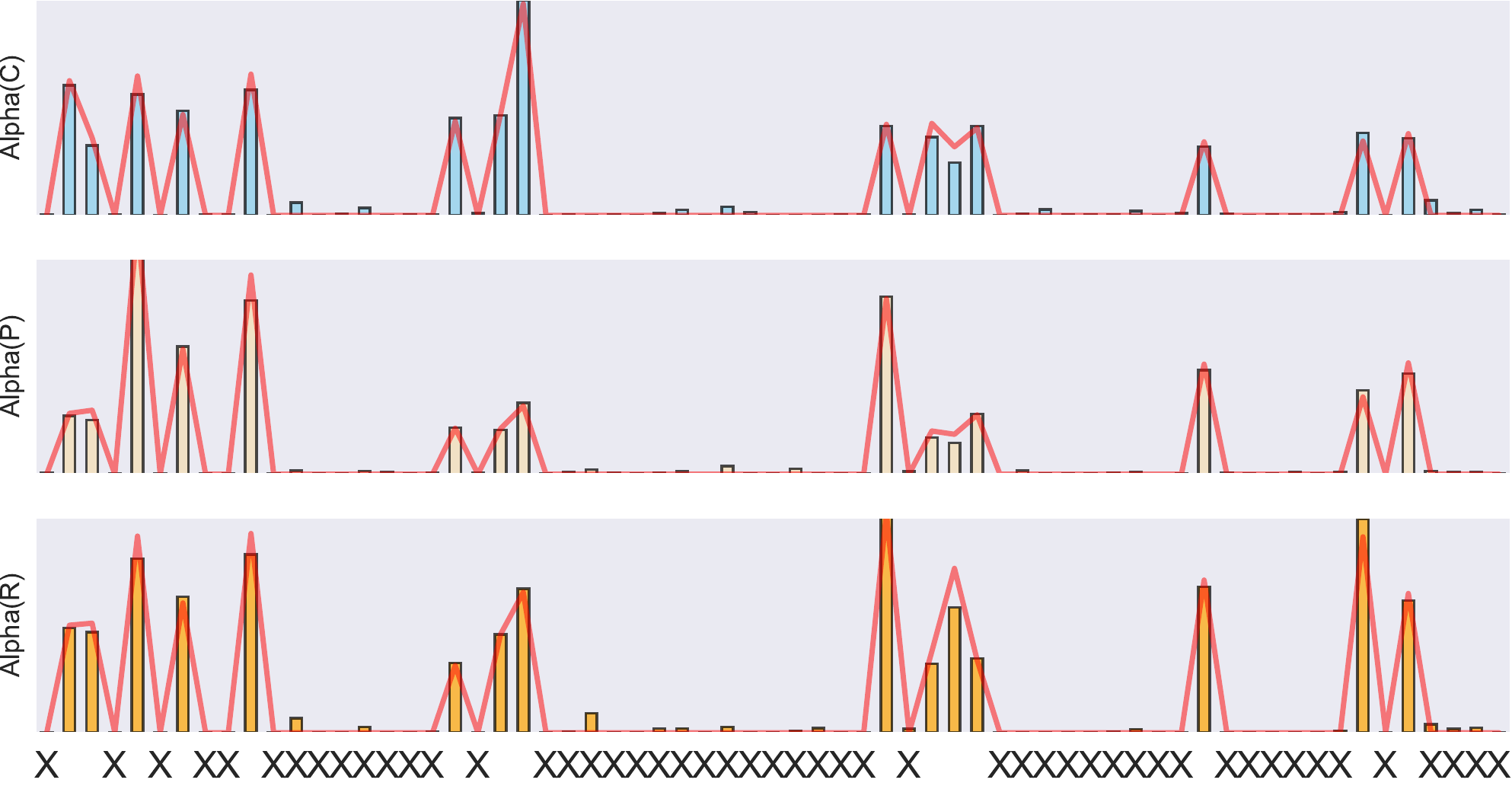}
     \caption{Target: Art}
  \end{subfigure}
 
  \begin{subfigure}{0.45\textwidth}
  \centering
     \includegraphics[scale=0.45]{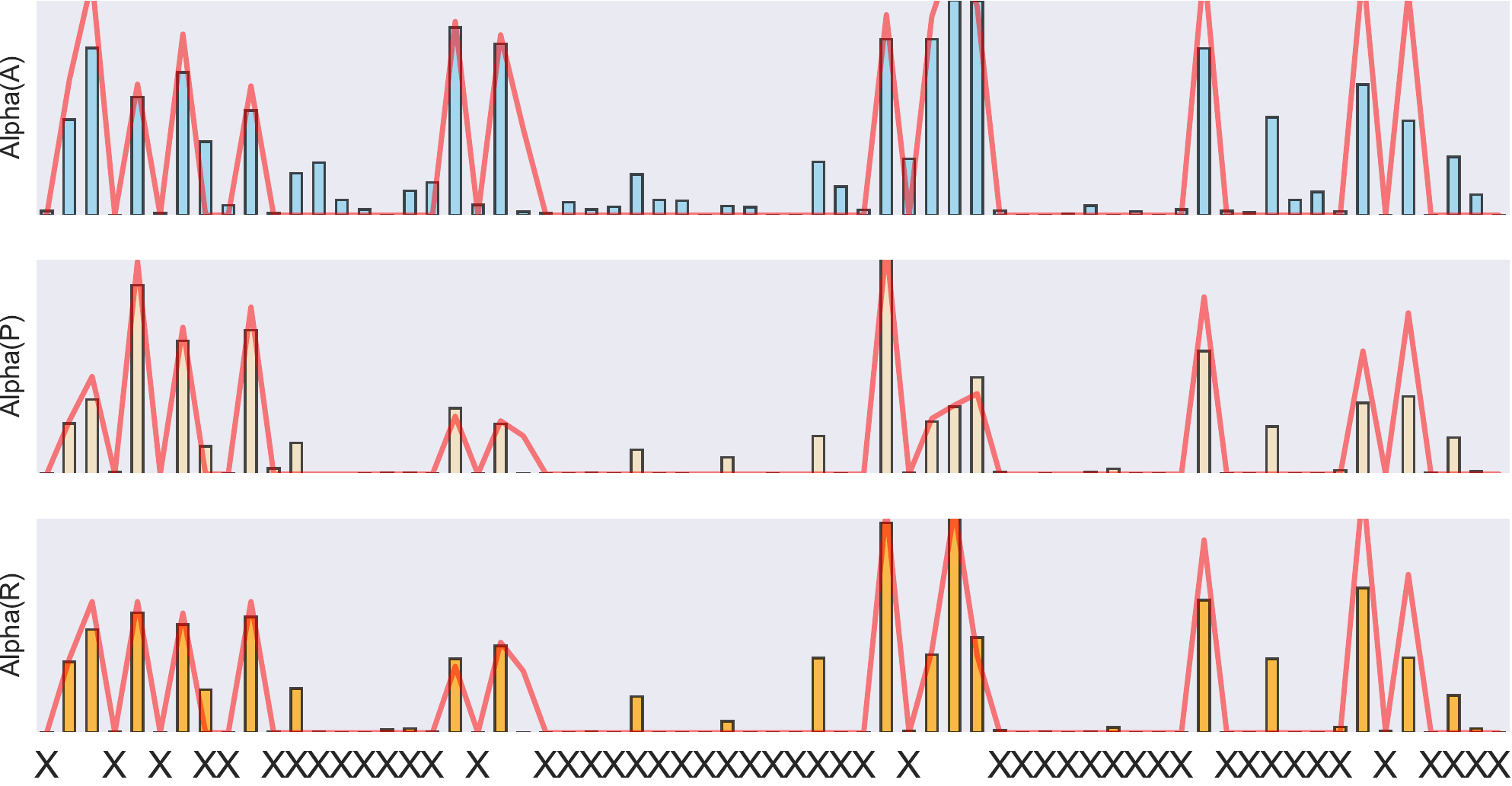}
     \caption{Target: Clipart}
  \end{subfigure}

  \begin{subfigure}{0.45\textwidth}
  \centering
     \includegraphics[scale=0.45]{figure/office_home_P50_11.pdf}
     \caption{Target: Product}
  \end{subfigure}
  
  \begin{subfigure}{0.45\textwidth}
  \centering
     \includegraphics[scale=0.45]{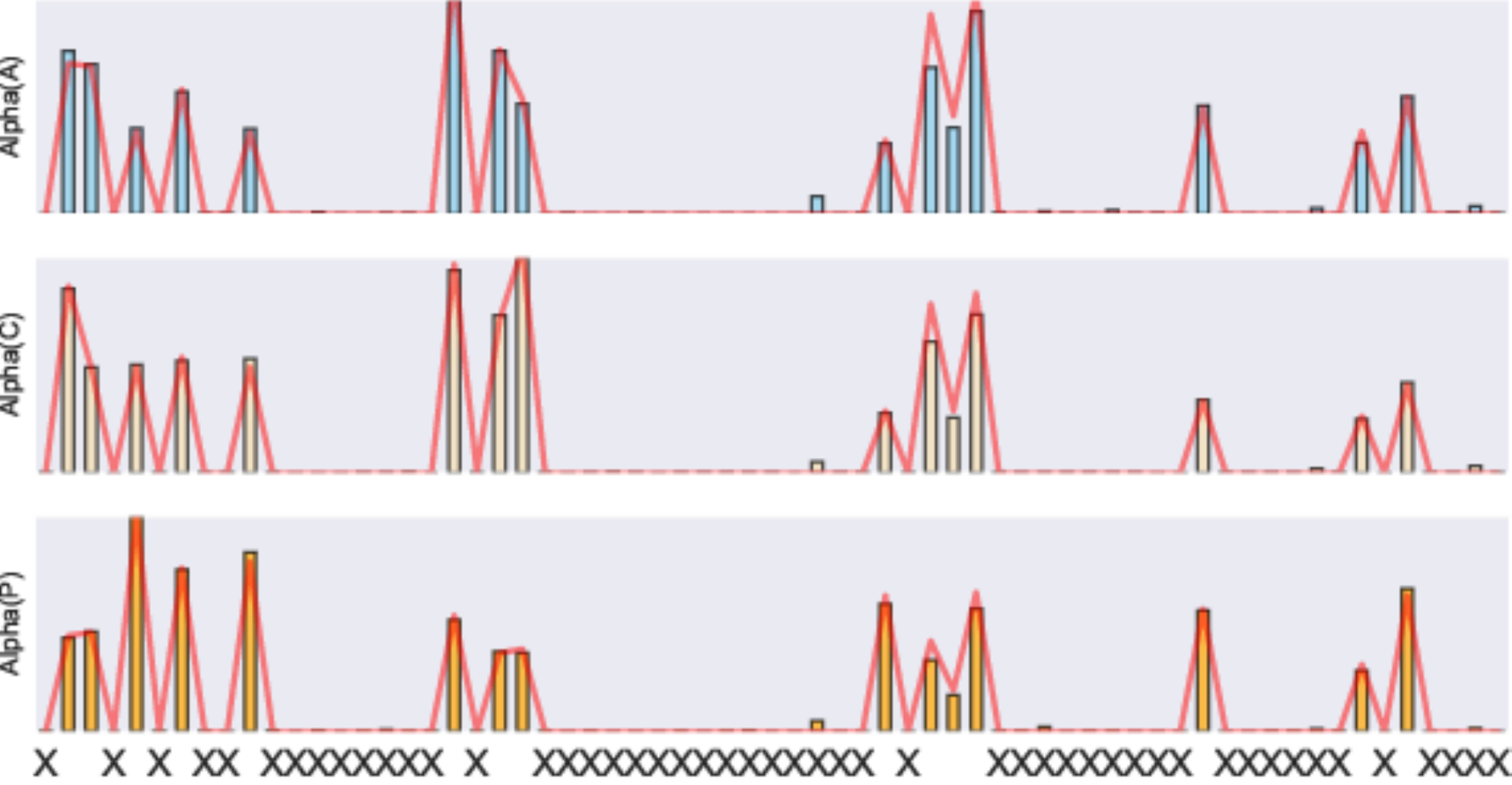}
     \caption{Target: Real-World}
  \end{subfigure}
  \caption{We select 15 classes and visualize estimated $\hat{\alpha}_t$ (the bar plot). The "X" along the x-axis represents the index of \textbf{dropped} 50 classes. The red curves are the ground-truth label distribution ratio.}
  \label{fig:appendix_office_home_abl2}
\end{figure}

\begin{figure}[t]
\centering
\begin{subfigure}{0.45\textwidth}
     \includegraphics[scale=0.45]{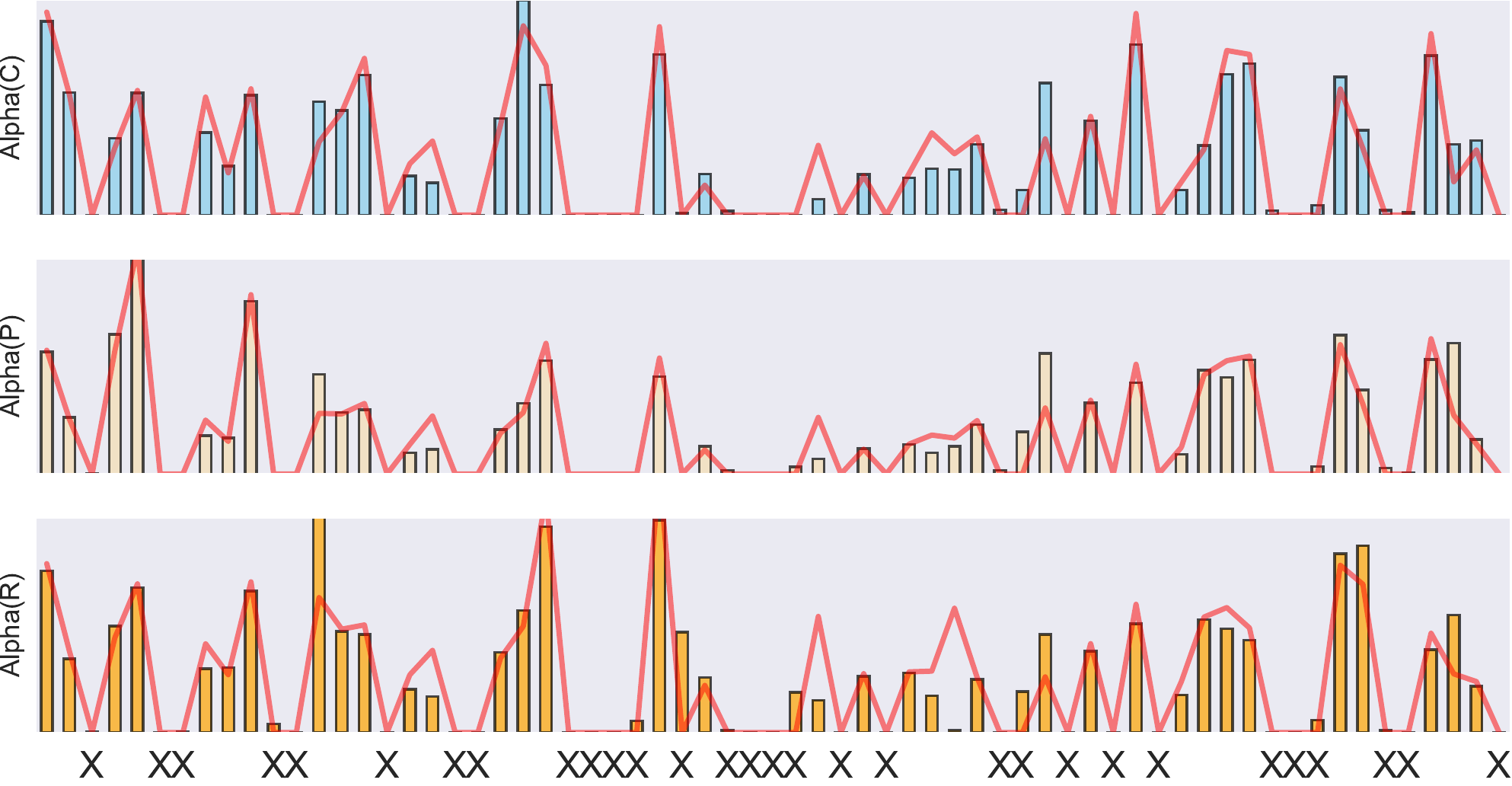}
     \caption{Target: Art}
  \end{subfigure}
  
  \begin{subfigure}{0.45\textwidth}
     \includegraphics[scale=0.45]{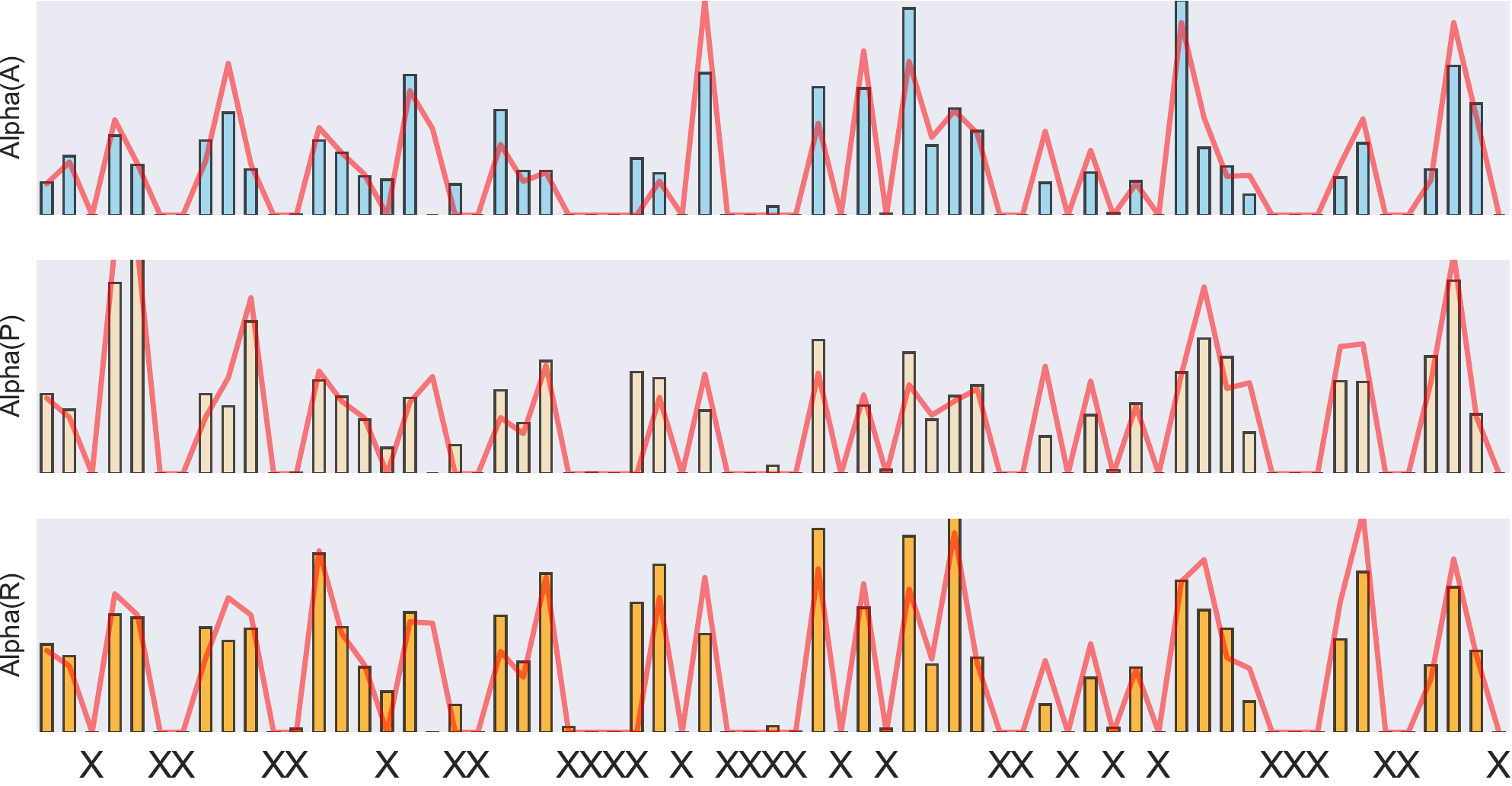}
     \caption{Target: Clipart}
  \end{subfigure}

  \begin{subfigure}{0.45\textwidth}
     \includegraphics[scale=0.45]{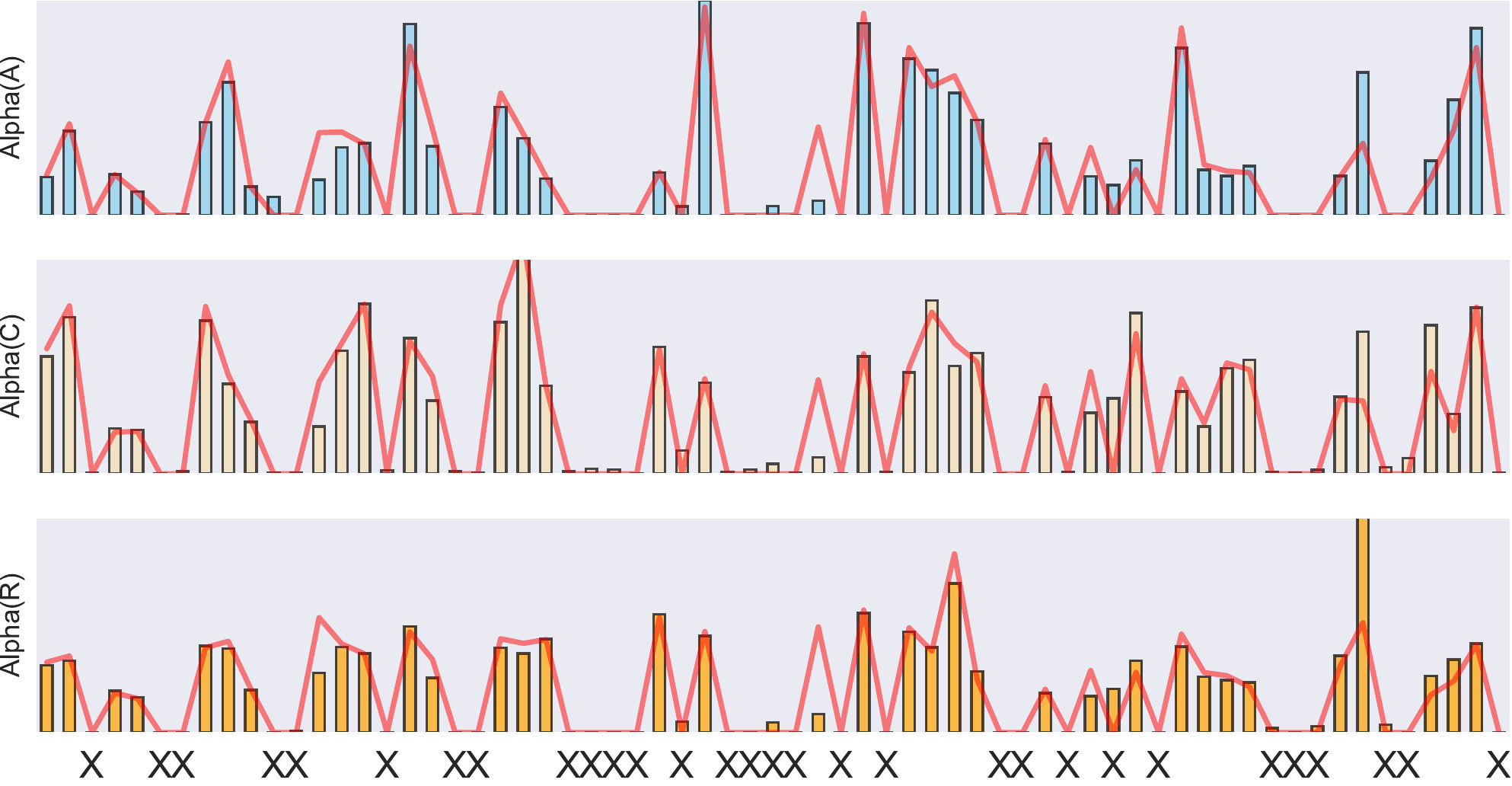}
     \caption{Target: Product}
  \end{subfigure}
  
  \begin{subfigure}{0.45\textwidth}
     \includegraphics[scale=0.45]{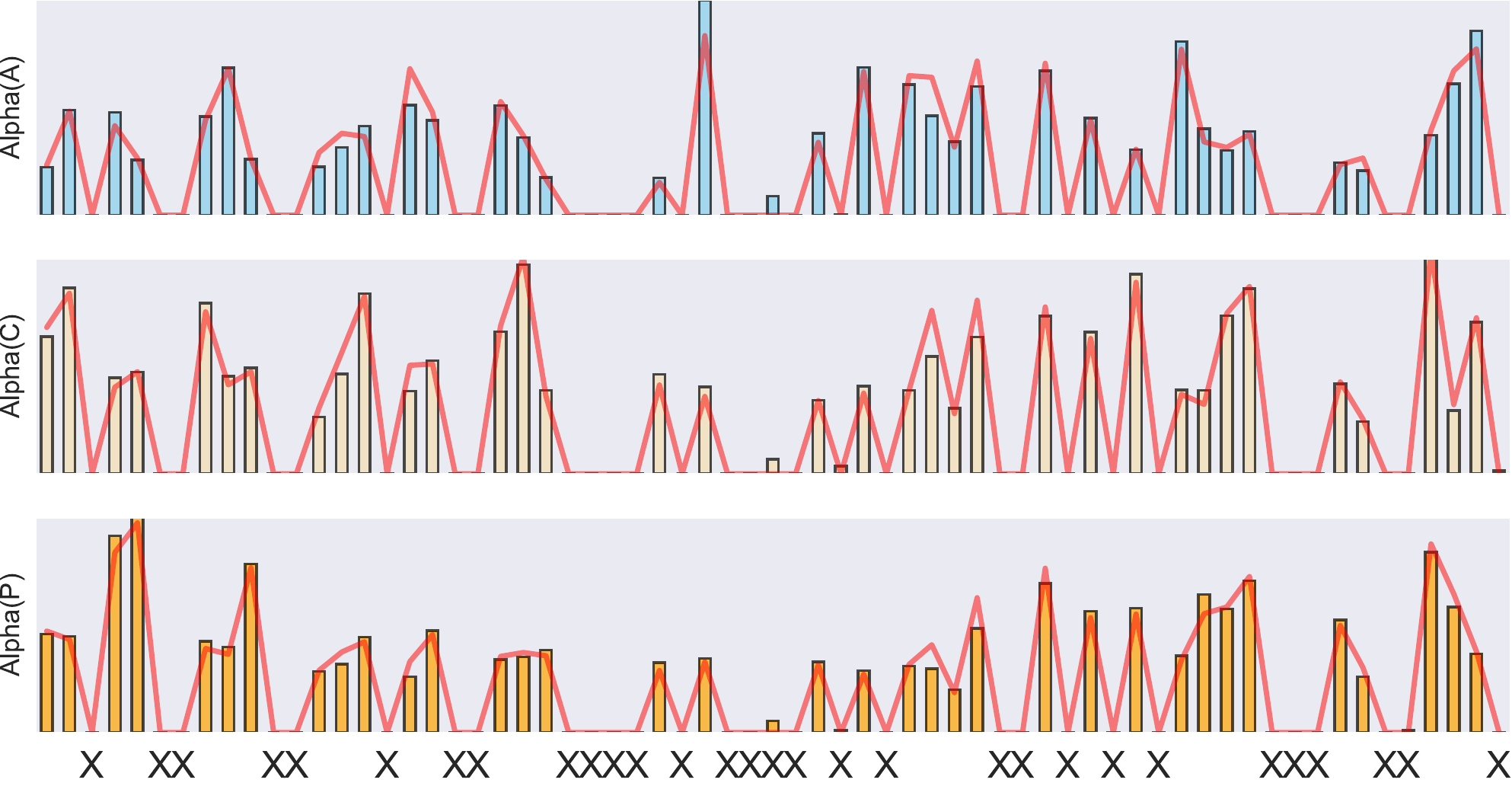}
     \caption{Target: Real-World}
  \end{subfigure}
  \caption{We select 35 classes and visualize estimated $\hat{\alpha}_t$ (the bar plot). The "X" along the x-axis represents the index of \textbf{dropped} 30 classes. The red curves are the ground-truth label distribution ratio.}
  \label{fig:appendix_office_home_abl3}
\end{figure}


\end{document}